\documentclass[11pt]{article}
\usepackage[a4paper,left=2.2cm,right=2.2cm,top=2.8cm,bottom=2.6cm,headheight=14pt,headsep=0.75cm,footskip=1.1cm]{geometry}
\usepackage[T1]{fontenc}
\usepackage{amsmath,amsthm}
\usepackage{XCharter}
\usepackage[xcharter,vvarbb]{newtxmath}
\usepackage{microtype}
\usepackage{booktabs}
\usepackage{multirow}
\usepackage{graphicx}
\usepackage{wrapfig}
\usepackage{placeins}
\usepackage{xcolor}
\usepackage[normalem]{ulem}
\definecolor{cornellred}{HTML}{B31B1B}
\newcommand{\methodchange}[1]{\textcolor{cornellred}{#1}}
\definecolor{reviewgreen}{HTML}{1B7837}   
\definecolor{reviewred}{HTML}{B3261E}
\definecolor{reviewpurple}{HTML}{7B2CBF}
\newif\ifshowreviewedits
\showrevieweditstrue

\usepackage{fontawesome5}
\usepackage{caption}
\DeclareCaptionLabelSeparator{bar}{\ \textbar\ }
\usepackage[square,authoryear]{natbib}
\usepackage{algorithm}
\usepackage{algorithmic}
\usepackage{parskip}
\usepackage{titlesec}
\titleformat{\section}{\Large\bfseries\color{cornellred}}{\thesection.}{0.5em}{}
\titleformat{\subsection}{\large\bfseries\color{cornellred}}{\thesubsection.}{0.5em}{}
\titleformat{\subsubsection}{\normalsize\bfseries\color{cornellred}}{\thesubsubsection.}{0.5em}{}
\titlespacing*{\section}{0pt}{1.4em}{0.6em}
\titlespacing*{\subsection}{0pt}{1.0em}{0.4em}
\titlespacing*{\subsubsection}{0pt}{0.8em}{0.3em}
\usepackage{fancyhdr}
\newcommand{\papertitle}{Does On-Policy Distillation Need to Complete Every Rollout?}
\newcommand{\paperdate}{\number\year-\ifnum\month<10 0\fi\number\month-\ifnum\day<10 0\fi\number\day}
\renewcommand{\headrulewidth}{0.5pt}
\renewcommand{\footrulewidth}{0.5pt}
\fancypagestyle{firstpage}{%
  \fancyhf{}%
  \fancyhead[L]{\small\bfseries\papertitle}%
  \fancyfoot[L]{\footnotesize\textsuperscript{*}Equal contribution.}%
  \fancyfoot[R]{\small\thepage}%
  \renewcommand{\headrulewidth}{0.5pt}%
  \renewcommand{\footrulewidth}{0.5pt}}

\usepackage{hyperref}
\hypersetup{colorlinks=true,linkcolor=cornellred,citecolor=cornellred,urlcolor=cornellred,pdftitle={\papertitle}}

\renewenvironment{abstract}{\noindent\textbf{Abstract}\ \textbar\ \ignorespaces}{\par\medskip}
\newcommand{\E}{\mathbb{E}}
\newcommand{\Prob}{\mathbb{P}}
\newcommand{\ind}{\mathbf{1}}
\newcommand{\method}{STOPD}
\newcommand{\methodfull}{\textbf{S}tochastic \textbf{T}runcated \textbf{O}n-\textbf{P}olicy \textbf{D}istillation}
\begin{document}
\thispagestyle{firstpage}
\begin{flushleft}
{\LARGE\bfseries\color{cornellred}\papertitle\par}
\vspace{0.9em}
{\bfseries
Zhengxin Zhang\textsuperscript{1,*},
Maosen Tang\textsuperscript{1,*},
Ning Wang\textsuperscript{1,*},
Changhao Li\textsuperscript{2,*},
Jingren Zhou\textsuperscript{3},
Bo Dai\textsuperscript{2},
Sainyam Galhotra\textsuperscript{1},
Alex Townsend\textsuperscript{1},
Claire Cardie\textsuperscript{1}\par}

\vspace{0.25em}
{\small
\textsuperscript{1}Cornell University \quad
\textsuperscript{2}Georgia Tech \quad
\textsuperscript{3}Alibaba Group\par}

\vspace{0.25em}
{\small\faEnvelope[regular]\
\href{mailto:zz865@cornell.edu,mt872@cornell.edu,nw366@cornell.edu}{\texttt{\{zz865, mt872, nw366\}@cornell.edu}}\par}
\end{flushleft}
\vspace{-0.4em}
\noindent\rule{\textwidth}{0.5pt}
\vspace{-0.6em}
\begin{center}
\includegraphics[width=\textwidth]{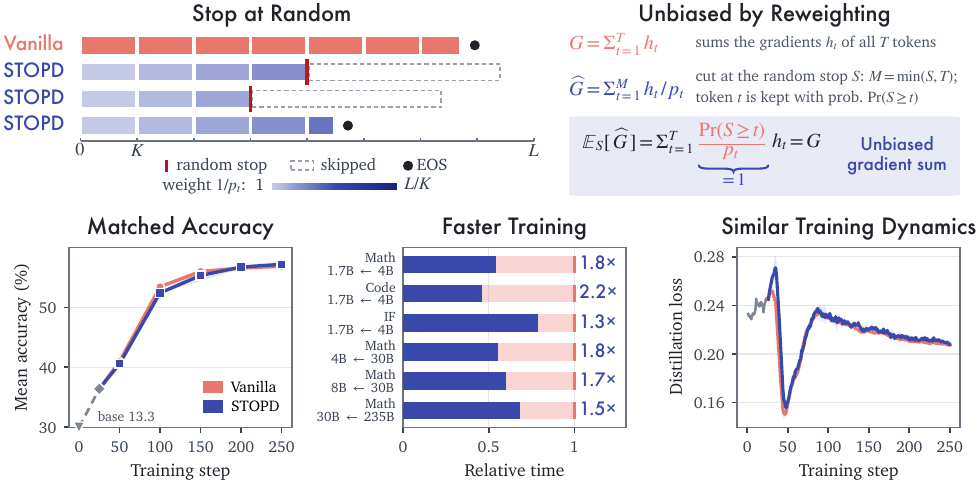}
\captionof{figure}{Overview of \method{}: random truncation with blockwise inverse-probability weighting (top) matches Vanilla OPD's accuracy and distillation loss for a 30B student in mathematics (bottom left and right) while reducing total training time across tasks and model scales (bottom center).}
\label{fig:overview}
\end{center}

\begin{abstract}
On-policy distillation (OPD) transfers teacher knowledge to a student through feedback on student-generated rollouts, but generating complete rollouts and computing per-token teacher feedback {are} computationally expensive. Existing truncation methods save computation by stopping rollouts early. However, they do not compensate for truncated token gradients, which biases the gradient estimates and changes Vanilla OPD's original training objective. Furthermore, their truncation decisions rely on assumptions about supervision quality or criteria that may not generalize across settings. We introduce \methodfull{} (\method), which accelerates OPD by randomly truncating student rollouts and correcting the resulting sampling bias with blockwise inverse-probability weighting (IPW). At each training step, we sample a stopping position for each rollout before generation and let the student generate until that position or natural termination. We obtain teacher feedback on every generated token and use blockwise IPW to compensate for the lower probability of retaining later tokens while controlling the weights. We prove that this weighting preserves Vanilla OPD's gradient sum in expectation and derive one-step update-error bounds for SGD, Adam, and Muon. We evaluate \method{} across multiple teacher--student model pairs, task domains, and optimizers. \method{} reduces generated tokens by $25$--$48\%$ while maintaining accuracy comparable to Vanilla OPD, achieving $1.3$--$2.2\times$ training speedups.

\end{abstract}

\section{Introduction}
\label{sec:intro}

On-policy distillation (OPD) effectively transfers capabilities from a teacher to a student through feedback on student-generated rollouts \citep{gu2024minillm,agarwal2024onpolicydistillationlanguagemodels,lu2025onpolicydistillation}. 
The teacher provides guidance on the student’s own rollouts.
However, each Vanilla OPD update generates complete student rollouts, scores every token with the teacher, and trains the student on all generated tokens \citep{wu2026lightningopd}. As rollouts grow during training, both autoregressive generation and the student's forward and backward passes become more expensive. Scaling up the teacher and student further increases the cost of scoring, generation, and updates. These growing costs make generation and student updates major computational bottlenecks, which motivates the following question: \emph{does OPD need to complete every rollout to obtain an effective student update?}

Existing approaches reduce different parts of this computation. Sparse supervision selects which generated positions contribute to the loss \citep{xu2026tip,wang2026teachability,sheng20261tokensenoughgradient}. These methods reduce student-update time but still require complete rollout generation, which remains the dominant cost.
Other approaches improve efficiency through early rollout truncation \citep{ziheng2026moreearlystoppingrollout,zhang2026rolloutsnecessaryonpolicydistillation,zhang-etal-2026-fast,yang2026pruneopdefficientreliableonpolicy,zhao2026prefixguidedonpolicydistillationmining,xin2026escapingklagreementtrap}. For example, Early Stopping Rollout (ESR) stops generation at a fixed cutoff \citep{ziheng2026moreearlystoppingrollout}, while Progressive On-Policy Distillation (POPD) extends the cutoff based on the assumption that feedback on later tokens is less reliable early in training \citep{zhang2026rolloutsnecessaryonpolicydistillation}. However, they leave truncated token gradients uncompensated, which introduces bias
into the gradient estimates and changes the original training objective of Vanilla OPD. Furthermore, they truncate rollouts based on assumptions about supervision quality or criteria that may not generalize across different settings.


{We introduce \methodfull{} (\method), which accelerates OPD by randomly truncating student rollouts and correcting the resulting sampling bias with blockwise inverse-probability weighting (IPW) \citep{McLeish_2011}. We first divide the maximum training length into fixed-width token blocks. We then uniformly sample a position for each rollout at each training step before generation and let the student generate until the end of the block containing it or natural termination. Every generated token receives teacher feedback, and its loss is weighted by the inverse of its inclusion probability. }


We show that \method{}
preserves the full-generation gradient and token-count sums in expectation.
We then analyze its normalized gradient, and show how batch
averaging reduces the stochastic discrepancy. Finally, we bound how this
gradient discrepancy affects the one-step parameter update under common optimizers such as SGD, Adam, and Muon.
We further analyze how \method{} saves tokens and time. \method{} saves more tokens in expectation as rollout length increases. These token savings can translate into greater training-time savings when \method{} also improves generation throughput by reducing long-context computation or freeing cache space for more concurrent rollouts.

We evaluate \method{} in multiple domains (e.g., mathematics, code, and instruction following) with students from 1.7B to 30B and teachers from 4B to 235B. \method{} generates $25$--$48\%$
fewer tokens and achieves $1.3$--$2.2\times$ training speedups while closely matching Vanilla OPD in accuracy. The gains also persist when replacing the main AdamW optimizer with SGD or Muon. \method{} remains robust across block widths, maintaining similar accuracy and computational savings. Increasing the batch size from $128$ to $512$ further reduces truncation-induced training fluctuations without reducing the speedup. Together, these results show that OPD need not complete every rollout to learn effectively. STOPD preserves the expected gradient sums of complete rollouts using only partial generation, achieving comparable accuracy at lower training cost. This principle could extend to other autoregressive training settings, including offline and online knowledge distillation and supervised fine-tuning.

\section{Related Work}
\label{sec:related}

\noindent\textbf{Sparse supervision.}
Sparse supervision selects or reweights token positions to determine where and how strongly teacher feedback contributes to distillation \citep{xu2026tip,wang2026teachability,huang2025selectkd,sheng20261tokensenoughgradient}. Token Importance in On-Policy Distillation (TIP) \citep{xu2026tip} uses student uncertainty and teacher-student disagreement to guide token selection. Teachability-Aware On-Policy Distillation (TA-OPD) \citep{wang2026teachability} selects positions with large teacher-student disagreement and high teacher probability mass on the student's top-$K$ tokens. Selective Token-Weighted Knowledge Distillation (SelecTKD) \citep{huang2025selectkd} weights token losses according to whether the teacher accepts the student's proposals. The information-efficiency ratio (IER) \citep{sheng20261tokensenoughgradient} further accounts for the reliability of sampled-token gradient estimates when selecting supervision positions. These methods select or reweight the tokens used in the student update, but still require complete rollout generation, which remains a major computational cost.

\noindent\textbf{Rollout truncation.}
Rollout truncation methods improve training efficiency by stopping student rollouts early \citep{ziheng2026moreearlystoppingrollout,zhang2026rolloutsnecessaryonpolicydistillation,zhang-etal-2026-fast,yang2026pruneopdefficientreliableonpolicy,zhao2026prefixguidedonpolicydistillationmining,xin2026escapingklagreementtrap}. Early Stopping Rollout (ESR) \citep{ziheng2026moreearlystoppingrollout} trains only on a fixed prefix. Progressive On-Policy Distillation (POPD) \citep{zhang2026rolloutsnecessaryonpolicydistillation} extends the generation cutoff during training. FastOPD \citep{zhang-etal-2026-fast} increases the prefix length according to a predefined schedule. Prune-OPD \citep{yang2026pruneopdefficientreliableonpolicy} uses teacher--student top-$k$ overlap to down-weight unreliable suffixes and adjust subsequent rollout budgets. Prefix-Guided On-Policy Distillation (PG-OPD) \citep{zhao2026prefixguidedonpolicydistillationmining} uses prefix-level overlap to select rollouts for longer continuation. KL Agreement Trap Termination (KAT) \citep{xin2026escapingklagreementtrap} stops generation when teacher--student KL remains low across consecutive windows. These truncation methods leave omitted token gradients uncompensated, which biases the gradient estimates and changes Vanilla OPD's training objective. Their truncation decisions also rely on assumptions about supervision quality or criteria that may not generalize across settings.

Rather than computing on complete rollouts or imposing a fixed cutoff, we apply random truncation to student rollouts and correct unequal token coverage with inverse-probability weighting \citep{McLeish_2011,pmlr-v97-beatson19a}. We use blockwise truncation to control extreme token weights and characterize its trade-off between generation cost and estimator variance. We preserve the expected gradient sum while reducing generation, teacher scoring, and student-update computation.

\section{\methodfull{}}
\label{sec:method}

In this section, we present the training procedure in Section~\ref{sec:training}, study gradient preservation and update error in Section~\ref{sec:gradient-error}, and analyze token and time savings in Section~\ref{sec:savings}.

\subsection{Training Procedure}
\label{sec:training}

\begin{algorithm}[!t]
\caption{\method{}, with changes from Vanilla OPD in \methodchange{red}.}
\label{alg:stopd}
\begin{algorithmic}[1]
\REQUIRE Student $\pi_\theta$, teacher $q$, max length $L$, \methodchange{block width $K$}, batch size $B$, \methodchange{warmup steps $W$}
\STATE Run $W$ Vanilla OPD updates and retain the student and optimizer states.
\STATE \methodchange{Set $C\leftarrow N_W$.}
\FOR{each subsequent training step}
  \STATE Sample a batch of $B$ prompts.
  \FOR{each prompt $x_i$}
    \STATE \methodchange{Draw $A_i$ independently and compute $S_i$ using Eq.~\eqref{eq:stop}.}
    \STATE Generate until natural termination or \methodchange{$S_i$}.
    \STATE Obtain local teacher feedback on every generated token.
    \STATE \methodchange{Assign $w_t=1/p_t$ using Eq.~\eqref{eq:survival}.}
  \ENDFOR
  \STATE \methodchange{Compute $\widehat{\mathcal L}$ using the current denominator $C$.}
  \STATE Apply the baseline optimizer to $\nabla_\theta\widehat{\mathcal L}$.
  \STATE \methodchange{Set $C\leftarrow\sum_i\sum_{t=1}^{M_i}w_t$ for the next step.}
\ENDFOR
\end{algorithmic}
\end{algorithm}

At training step $k$, consider a batch of $B$ prompts $x_1,\ldots,x_B$ and the student $\pi_{\theta_k}$. Let $L$ be the max length and $1\leq T_i\leq L$ the length that rollout $i$ would have under full generation, including a naturally generated end-of-sequence token when reached. Before generating each student rollout, we randomly sample a stopping position. The student then generates until that position or natural termination. We then apply block-wise inverse-probability weighting (IPW) to the observed token losses to compensate for omitted tokens. 

Naive IPW \citep{McLeish_2011} samples the stopping position uniformly from $\{1,\ldots,L\}$. However, long OPD rollouts can extend close to the max length, where tokens have very small probabilities of being retained. A token at position $L$, for example, is observed with probability $1/L$ and therefore receives weight $L$. When $L$ is large, observing such a token can produce an unusually large gradient contribution. To limit these extreme weights, we introduce blockwise truncation. We group rollout positions into consecutive blocks and round each sampled stopping position up to the end of its block. Let $K\geq1$ denote the block width. When $K$ divides $L$, the last block is retained with probability $K/L$, reducing the largest weight from $L$ to $L/K$. Larger blocks thus trade additional generation for less extreme weighting. Before generating each rollout, we sample the anchor $A_i$ to select the stopping block and set the cutoff to that block's end:
\begin{equation}
A_i\sim\mathrm{Uniform}\{0,\ldots,L-1\},
\qquad
S_i=\min\!\left\{K\left(\left\lfloor A_i/K\right\rfloor+1\right),L\right\}.
\label{eq:stop}
\end{equation}
The student generates until natural termination or the sampled cutoff $S_i$, retaining $M_i=\min(T_i,S_i)$ tokens. We obtain teacher feedback on all retained tokens and weight each token's loss by the inverse probability that generation reaches its position. For $1\leq t\leq T_i$, this probability is
\begin{equation}
p_t=\Prob(S_i\geq t)
=\frac{L-K\lfloor(t-1)/K\rfloor}{L}>0.
\label{eq:survival}
\end{equation}
We can compute each token's inclusion probability and blockwise inverse-probability weight $w_t=1/p_t$ without knowing the complete rollout length. Later positions are retained less often and receive correspondingly larger weights. Let $h_{i,t}$ denote the gradient contributed by token $t$ in rollout $i$. The full gradient sum and our blockwise IPW estimate are
\begin{align}
G_k=\sum_{i=1}^B\sum_{t=1}^{T_i}h_{i,t}, \qquad \widehat{G}_k=\sum_{i=1}^B\sum_{t=1}^{M_i}w_t h_{i,t}.
\label{eq:ipw}
\end{align}

Vanilla computes its training loss by averaging token losses over the complete batch of $N_k=\sum_{i=1}^BT_i$ tokens, giving the gradient $g_k^{\mathrm V}=G_k/N_k$. Since random truncation leaves $N_k$ unknown, we normalize our weighted token losses using an estimate $C_k$ from the preceding batch:
\begin{equation}
\widehat{\mathcal L}_k(\theta)
=\frac1{C_k}\sum_{i=1}^B\sum_{t=1}^{M_i}w_t\ell_{i,t}(\theta),
\qquad
\widehat g_k=\widehat G_k/C_k.
\label{eq:loss}
\end{equation}
Here, $\ell_{i,t}(\theta)$ is the distillation loss for token $t$ in rollout $i$, with $h_{i,t}=\nabla_\theta\ell_{i,t}(\theta_k)$. The weights and $C_k$ are held constant during differentiation. Initial rollouts of base models are usually short relative to the max length, so random truncation offers limited savings. We therefore use $W\geq1$ Vanilla OPD warmup steps to avoid truncation noise during this phase and initialize $C_{W+1}$ with the last warmup batch's exact token count $N_W$. After switching to random truncation, each batch estimates its complete token count using the same inclusion weights:
\begin{equation}
\widehat N_k=\sum_{i=1}^B\sum_{t=1}^{M_i}w_t,
\qquad
C_k=
\begin{cases}
N_W,&k=W+1,\\
\widehat N_{k-1},&k>W+1.
\end{cases}
\label{eq:lagged}
\end{equation}

Algorithm~\ref{alg:stopd} summarizes the training procedure of our method. Next, we analyze gradient preservation and update error in Section~\ref{sec:gradient-error} and how our method saves tokens and time in Section~\ref{sec:savings}.

\subsection{Gradient Preservation and Update Error}
\label{sec:gradient-error}

We first establish gradient preservation by showing that
\method{} preserves the full-generation gradient and
token-count sums in expectation. We then analyze the normalized gradient,
separating the variation introduced by random truncation from the remaining
scale difference, and show how batch averaging reduces the stochastic
discrepancy. Finally, we
bound how this gradient discrepancy affects the one-step parameter update
under SGD, Adam, and Muon.

We compare \method{} with full-generation Vanilla OPD at a fixed
training step. We condition on the complete rollouts and their local
teacher feedback, leaving only the sampled stopping positions random, and
write $\E_S$ for expectation over this randomness. The comparison uses
three conditions. (1) Given the same prompt and previously generated
tokens, the next-token distribution does not depend on the remaining
generation budget. (2) The local loss at a retained position depends only
on the prompt, the generated tokens through that position, and its local
teacher feedback, so removing the later tokens leaves the retained token's loss unchanged. (3) Conditional on the training history, stopping positions are mutually
independent across rollouts and are sampled independently of the complete
rollout and teacher-feedback randomness. Each retained position is weighted
by its exact inverse inclusion probability $1/p_t$. Appendix~\ref{app:preservation} states these
conditions formally.

\paragraph{Preservation.}
Let
\(
I_{i,t}=\mathbf 1\{S_i\geq t\}
\)
indicate whether position $t$ is retained. Since
$\E_S[I_{i,t}]=p_t$, the weighted contribution of each position satisfies
\begin{equation}
\E_S\left[
\frac{I_{i,t}}{p_t}h_{i,t}
\right]
=
h_{i,t},
\qquad
\E_S\left[
\frac{I_{i,t}}{p_t}
\right]
=
1.
\label{eq:tokenwise-preservation}
\end{equation}
Thus blockwise inverse-probability weighting restores both the gradient contribution
and the token-count contribution of every position in expectation.
Equivalently, the estimators in Section~\ref{sec:training} can be written as
\begin{equation}
\widehat G_k
=
\sum_{i=1}^{B}
\sum_{t=1}^{T_i}
\frac{I_{i,t}}{p_t}h_{i,t},
\qquad
\widehat N_k
=
\sum_{i=1}^{B}
\sum_{t=1}^{T_i}
\frac{I_{i,t}}{p_t},
\end{equation}
and summing Eq.~\eqref{eq:tokenwise-preservation} gives
\begin{equation}
\E_S[\widehat G_k]
=
G_k,
\qquad
\E_S[\widehat N_k]
=
N_k.
\label{eq:sum-preservation}
\end{equation}

Equation~\eqref{eq:sum-preservation} establishes preservation at the level
directly affected by truncation. The training gradient also
contains the normalization in Eq.~\eqref{eq:loss}. Vanilla uses
$g_k^{\mathrm V}=G_k/N_k$, whereas \method{} uses
$\widehat g_k=\widehat G_k/C_k$, with $C_k$ fixed before the current
stopping positions are sampled. Therefore
\begin{equation}
\E_S[\widehat g_k]
=
\frac{G_k}{C_k}
=
\frac{N_k}{C_k}g_k^{\mathrm V}.
\label{eq:normalized-preservation}
\end{equation}
Blockwise IPW therefore removes the bias introduced by truncating the gradient sum.

\paragraph{Error after normalization.}
Define
\(
\rho_k=(N_k-C_k)/C_k.
\)
Then we have
\begin{equation}
\underbrace{
\widehat g_k-g_k^{\mathrm V}
}_{\text{gradient discrepancy}}
=
\underbrace{
\frac{\widehat G_k-G_k}{C_k}
}_{\text{variation from random truncation}}
+
\underbrace{
\rho_k g_k^{\mathrm V}
}_{\text{scale difference}}.
\label{eq:gradient-error-decomposition}
\end{equation}

To quantify the variation from random truncation, let
\(
V_k=\E_S[\|\widehat G_k-G_k\|^2].
\)
Since Eq.~\eqref{eq:sum-preservation} gives
$\E_S[\widehat G_k-G_k]=0$, taking the expected squared norm of
Eq.~\eqref{eq:gradient-error-decomposition} gives
\begin{equation}
\E_S\!\left[
\|\widehat g_k-g_k^{\mathrm V}\|^2
\right]
=
\frac{V_k}{C_k^2}
+
\rho_k^2\|g_k^{\mathrm V}\|^2.
\label{eq:normalized-error}
\end{equation}
Appendices~\ref{app:preservation} and~\ref{app:variance-block-width} derive
this variation using
\begin{equation}
v(T)
=
\sum_{t=1}^{T}
\left(p_t^{-1}-1\right)
\{2(T-t)+1\}.
\label{eq:variance-profile}
\end{equation}
If $\|h_{i,t}\|\leq H$ for every token, then
\begin{equation}
V_k
\leq
H^2\sum_{i=1}^{B}v(T_i),
\qquad
\E_S\!\left[
\|\widehat g_k-g_k^{\mathrm V}\|^2
\right]
\leq
\frac{
H^2\sum_{i=1}^{B}v(T_i)
}{
C_k^2
}
+
\rho_k^2\|g_k^{\mathrm V}\|^2.
\label{eq:gradient-error-bound}
\end{equation}

Batch averaging reduces the variation from random truncation. Since
$C_k\geq B$ and $v(T_i)\leq v(L)$,
\begin{equation}
\frac{V_k}{C_k^2}
\leq
\frac{H^2v(L)}{B}.
\label{eq:cutoff-batch-bound}
\end{equation}

\paragraph{Large-batch agreement.}
We next analyze how the scale difference behaves as the batch size grows.
For $k>W+1$, Eq.~\eqref{eq:lagged} gives
$C_k=\widehat N_{k-1}$. By Eq.~\eqref{eq:sum-preservation},
$\widehat N_{k-1}$ estimates the preceding complete token count.
{Assume that} complete rollouts are conditionally independent and identically
distributed within each batch {and that} token gradients {are} bounded by $H$. Appendices~\ref{app:lagged-normalization} and~\ref{app:batch-size}
state the formal conditions.

Let $\mu_j$ denote the conditional mean
complete rollout length at step $j$. Then
\begin{equation}
\frac{N_k}{B}
=
\mu_k+O_{\mathbb P}(B^{-1/2}),
\qquad
\frac{C_k}{B}
=
\mu_{k-1}+O_{\mathbb P}(B^{-1/2}),
\label{eq:count-concentration}
\end{equation}
and therefore
\begin{equation}
\frac{N_k}{C_k}
=
\frac{\mu_k}{\mu_{k-1}}
+
O_{\mathbb P}(B^{-1/2}).
\label{eq:normalization-ratio}
\end{equation}

Combining Eqs.~\eqref{eq:cutoff-batch-bound} and~\eqref{eq:normalization-ratio} with
Eq.~\eqref{eq:gradient-error-decomposition} gives
\begin{equation}
\widehat g_k-g_k^{\mathrm V}
=
\left(
\frac{\mu_k}{\mu_{k-1}}-1
\right)
g_k^{\mathrm V}
+
O_{\mathbb P}(B^{-1/2}).
\label{eq:large-batch-agreement}
\end{equation}
Thus batch averaging reduces the stochastic discrepancy at rate
$B^{-1/2}$. When the adjacent
conditional mean lengths agree,
$\widehat g_k-g_k^{\mathrm V}=O_{\mathbb P}(B^{-1/2})$.

\paragraph{Optimizer updates.}
We next translate the gradient discrepancy into a bound on the one-step
update error. We compare \method{} and Vanilla under the same optimizer,
from the same parameters and previous optimizer state, as formalized in
Appendix~\ref{app:optimizers}.

Let $\Phi(x)$ denote the parameter increment produced from gradient input
$x$, with the current parameters, previous optimizer state, and
hyperparameters fixed. Let $\mathcal Q$ be the
global norm-clipping map in Eq.~\eqref{eq:app-global-clip}, and write
\begin{equation}
\Delta\theta^{\mathrm S}
=
\Phi(\mathcal Q(\widehat g_k)),
\qquad
\Delta\theta^{\mathrm V}
=
\Phi(\mathcal Q(g_k^{\mathrm V})).
\label{eq:optimizer-updates}
\end{equation}

Let $\Lambda_{\mathrm{opt}}$ denote the optimizer-specific coefficient
evaluated at the Vanilla reference $\mathcal Q(g_k^{\mathrm V})$, using
the largest coefficient across disjoint parameter blocks. Since
$\mathcal Q$ is nonexpansive,
Eq.~\eqref{eq:app-clipped-update} gives
\begin{equation}
\|\Delta\theta^{\mathrm S}-\Delta\theta^{\mathrm V}\|
\leq
\Lambda_{\mathrm{opt}}
\|\widehat g_k-g_k^{\mathrm V}\|.
\label{eq:update-error}
\end{equation}
Combining Eq.~\eqref{eq:update-error} with
Eq.~\eqref{eq:normalized-error} gives
\begin{equation}
\E_S\!\left[
\|\Delta\theta^{\mathrm S}-\Delta\theta^{\mathrm V}\|^2
\right]
\leq
\Lambda_{\mathrm{opt}}^2
\left\{
\frac{V_k}{C_k^2}
+
\rho_k^2\|g_k^{\mathrm V}\|^2
\right\}.
\label{eq:update-error-ms}
\end{equation}

For SGD, Eq.~\eqref{eq:app-sgd-pointwise} gives
$\Lambda_{\mathrm{SGD}}=\eta$. Proposition~\ref{prop:app-adam} gives the Adam coefficient
$\Lambda_{\mathrm A}$, and Proposition~\ref{prop:app-muon} gives the
Muon coefficient $\Lambda_{\mathrm M}$.

Together, blockwise IPW preserves the complete gradient and token-count sums in
expectation, while normalization separates the gradient discrepancy into
variation from random truncation and the scale difference. Batch averaging
reduces the stochastic discrepancy, and changes in conditional mean
rollout length between adjacent steps determine the remaining scale
difference. The optimizer-specific coefficient then bounds the
corresponding one-step update error.

\subsection{Token and Time Savings}
\label{sec:savings}

We first quantify token savings, then analyze how they translate into
reduced generation time.

\noindent\textbf{Token savings.}
For a complete rollout of length $T$, random truncation skips the token at position $t$ with probability $1-p_t$. Therefore, the expected number of tokens saved is $D(T)=\sum_{t=1}^{T}(1-p_t)$.
For $K\ll T$, $D(T)\approx T^2/(2L)$, so the expected fraction of tokens saved is approximately $T/(2L)$. This corresponds to about $25\%$ expected token savings for $T=L/2$ and nearly $50\%$ for $T=L$.

\noindent\textbf{Time savings.}
We further show why random truncation can yield greater speedups than token savings alone would predict. During decoding, attention reads keys and values for the prompt and previously generated tokens from the key--value (KV) cache. Later tokens attend to longer cached contexts, and random truncation is more likely to skip these positions. Truncation therefore avoids forward-pass computation for skipped tokens, reduces KV-cache reads, and limits cache growth. To relate these savings to generation time, we model decoding time as a fixed cost per round plus a cost proportional to the cached context processed. We measure this cumulative context by the KV area $\mathcal A_i=\sum_{t=1}^{M_i}(T_i^{\mathrm p}+t-1)$, where $T_i^{\mathrm p}$ is the prompt length of rollout $i$. Since position $t$ is generated with probability $p_t$, the expected KV area is
\begin{equation}
\E_S\mathcal{A}_i=\sum_{t=1}^{T_i}p_t\,(T_i^{\mathrm p}+t-1)
\approx T_i^{\mathrm p}T_i\Bigl(1-\frac{T_i}{2L}\Bigr)+\frac{T_i^2}{2}-\frac{T_i^3}{3L}.
\label{eq:method-kv-area}
\end{equation}
The approximation neglects discrete and block-rounding corrections for $K\ll T_i$; Appendix~\ref{app:time-savings} derives the expression and the time model. When rollouts are much longer than their prompts, the prompt term contributes little to the KV area. A rollout reaching half the max length then saves about $25\%$ of its tokens and $33\%$ of its KV area in expectation. At the max length, these savings increase to about $50\%$ and $67\%$, respectively. Thus, random truncation reduces the cumulative cached context by a larger fraction than the token count.

The per-rollout reductions above accumulate across the batch. Let
\[
\mathcal A=\sum_{i=1}^{B}\mathcal A_i,
\qquad
\mathcal A^{\mathrm V}=\sum_{i=1}^{B}\mathcal A_i^{\mathrm V},
\]
where $\mathcal A_i^{\mathrm V}$ is the full-generation KV area of rollout
$i$, and let $N=\sum_iT_i$ and $N^{\mathrm S}=\sum_iM_i$ denote the
generated-token counts of Vanilla and \method{}, respectively. The
batch-level KV-area ratio decomposes as
\begin{equation}
\frac{\mathcal A^{\mathrm V}}{\E_S[\mathcal A]}
=
\underbrace{
\frac{N}{\E_S[N^{\mathrm S}]}
}_{\text{token-count ratio}}
\quad
\underbrace{
\frac{\mathcal A^{\mathrm V}/N}
{\E_S[\mathcal A]/\E_S[N^{\mathrm S}]}
}_{\text{mean-context-per-token ratio}}.
\label{eq:method-kv-area-ratio}
\end{equation}
The first factor captures the reduction in generated tokens, while the
second captures the reduction in cached context processed per generated
token. When the prompt contribution is neglected, both factors are at least one.

These reductions translate into generation-time savings according to the decoding costs and scheduling. Consider the setting where prefill time is negligible compared with decoding time. When attention over long contexts dominates decoding, the model predicts that time savings track the reduction in KV area. The same relation holds when queued rollouts replace completed ones and maintain roughly constant cache occupancy at a similar level in both methods: shorter contexts allow more rollouts to generate concurrently, improving throughput. Larger training batches can help sustain this scheduling regime. In our 1.7B mathematics experiments, generating $34\%$ fewer tokens yields a $2.0\times$ generation speedup, exceeding the $1.5\times$ ratio implied by token savings alone (Table~\ref{tab:time-decomposition}).

If all rollouts start together, a long surviving rollout can keep the number of decoding rounds nearly unchanged. Truncation can still accelerate generation by reducing the context processed in those rounds, but the gain is smaller when fixed costs per round dominate. Random truncation also reduces teacher scoring and student loss computation, further contributing to total training-time savings. We evaluate token and time savings in Section~\ref{sec:main-results}.

\section{Experimental Setup}
\label{sec:experimental-setup}

We use the same experimental settings for Vanilla OPD and \method{}.

\noindent\textbf{Models.}
The main setting distills Qwen3-4B-Instruct-2507 into Qwen3-1.7B-Base \citep{qwen2025qwen3}. For scaling, we distill Qwen3-30B-A3B-Instruct-2507 into Qwen3-4B-Base and Qwen3-8B-Base{, and Qwen3-235B-A22B-Instruct-2507 into Qwen3-30B-A3B-Base,} in mathematics. All models run in non-thinking mode.

\noindent\textbf{Data.}
We use Nemotron-Cascade-RL-Math \citep{wang2025nemotroncascade} for mathematics, DeepCoder-Preview \citep{deepcoder2025dataset} for code, and Nemotron-RL instruction following \citep{nvidia2025nemotronrlif} for instruction following (IF).

\noindent\textbf{Training.}
Our main runs use AdamW \citep{loshchilov2019adamw} with learning rate $10^{-6}$, batch size 256, and gradient clipping at 1, for 250 steps with three seeds. Rollouts use temperature 1 and max lengths of 16{,}384, 8{,}192, and 4{,}096 for mathematics, code, and instruction following, respectively. \method{} uses block width $K=128$ after 25 Vanilla warmup steps. Each run uses two GPUs: H100 for all 1.7B experiments, H200 for the 4B student, and GB200 for the 8B student. Each run with the 30B student uses four B300 GPUs.


\noindent\textbf{Evaluation.}
We report the average score over $k$ sampled responses per problem (avg@$k$). For mathematics, we report AIME24, AIME25, and AMC23 at avg@8, and MATH500 \citep{hendrycks2021math} and Minerva \citep{lewkowycz2022minerva} at avg@3, allowing up to 32{,}768 tokens. For code we report HumanEval+ and MBPP+ \citep{liu2023evalplus} at avg@4, two LiveCodeBench windows (2024-08 to 2025-02 and 2025-02 to 2025-05; \citealp{jain2024livecodebench}) at avg@8, and 200 held-out training-distribution problems at avg@4, allowing up to 16{,}384 tokens. For instruction following, we report IFEval prompt-level strict accuracy \citep{zhou2023ifeval} at avg@4, allowing up to 8{,}192 tokens. All evaluations sample at temperature 0.7 and top-$p=0.95$. Main results are averaged over three training seeds.

\noindent\textbf{Timing.}
We measure wall-clock time per training step, including generation, teacher scoring, the student update, and other overhead.

\section{Experiments}
\label{sec:experiments}

We evaluate whether \method{} reduces training time while retaining the accuracy of Vanilla OPD. Section~\ref{sec:main-results} compares accuracy, training speed, and training dynamics over 250 steps in three domains, Section~\ref{sec:scaling} turns to larger students, Section~\ref{sec:optimizers} repeats mathematics with SGD and Muon, Section~\ref{sec:block-width} varies the block width, and Section~\ref{sec:batch-size-math} analyzes the batch size. Table~\ref{tab:speed} summarizes training speed for the main setting.

\subsection{Main Results}
\label{sec:main-results}

\begin{figure*}[t]
\centering
\includegraphics[width=\textwidth]{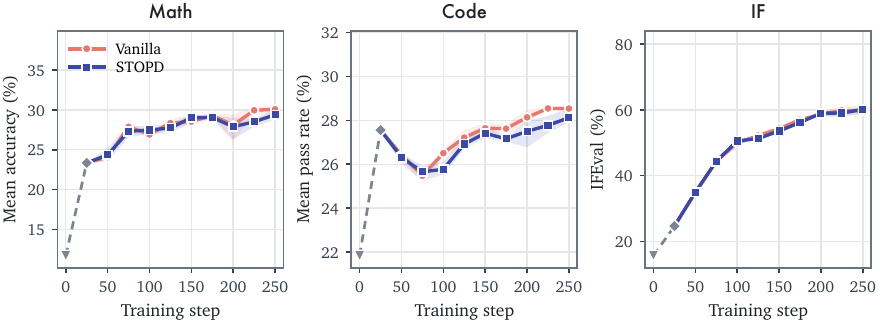}
\caption{Accuracy of the 1.7B student against training step with the base model and the shared warmup in gray.}
\label{fig:main-time}
\end{figure*}


\begin{table}[t]
\centering
\caption{Cost of a training step in the main setting.}
\label{tab:speed}\label{tab:time-decomposition}
\small
\setlength{\tabcolsep}{6pt}
\begin{tabular}{llrrrrrr}
\toprule
\multicolumn{4}{l}{\textit{Qwen3-1.7B-Base $\leftarrow$ Qwen3-4B-Instruct-2507}} & \multicolumn{4}{c}{\textbf{Time per step (s)}}\\
\cmidrule(lr){5-8}
\textbf{Domain} & \textbf{Method} & \textbf{Tokens} & \textbf{k tok/s} & \textbf{Generation} & \textbf{Update} & \textbf{Other} & \textbf{Total}\\
\midrule
\multirow{3}{*}{Math} & Vanilla & 9{,}516 & 3.5 & 697 & 189 & 56 & 941\\
 & \method{} & 6{,}255 & 4.5 & 358 & 119 & 36 & 513\\
 & \textcolor{cornellred}{\textit{Speedup}} & \textcolor{cornellred}{1.5$\times$} & \textcolor{cornellred}{1.3$\times$} & \textcolor{cornellred}{2.0$\times$} & \textcolor{cornellred}{1.6$\times$} & \textcolor{cornellred}{1.5$\times$} & \textcolor{cornellred}{\textbf{1.8$\times$}}\\
\midrule
\multirow{3}{*}{Code} & Vanilla & 6{,}587 & 4.4 & 383 & 142 & 41 & 566\\
 & \method{} & 3{,}444 & 5.6 & 158 & 78 & 24 & 260\\
 & \textcolor{cornellred}{\textit{Speedup}} & \textcolor{cornellred}{1.9$\times$} & \textcolor{cornellred}{1.3$\times$} & \textcolor{cornellred}{2.4$\times$} & \textcolor{cornellred}{1.8$\times$} & \textcolor{cornellred}{1.7$\times$} & \textcolor{cornellred}{\textbf{2.2$\times$}}\\
\midrule
\multirow{3}{*}{IF} & Vanilla & 710 & 3.8 & 54 & 36 & 14 & 104\\
 & \method{} & 533 & 3.8 & 39 & 29 & 12 & 80\\
 & \textcolor{cornellred}{\textit{Speedup}} & \textcolor{cornellred}{1.3$\times$} & \textcolor{cornellred}{1.0$\times$} & \textcolor{cornellred}{1.4$\times$} & \textcolor{cornellred}{1.2$\times$} & \textcolor{cornellred}{1.1$\times$} & \textcolor{cornellred}{\textbf{1.3$\times$}}\\
\bottomrule
\end{tabular}

\end{table}


Figure~\ref{fig:main-time} shows that
\method{} and Vanilla both increase the accuracy across all three domains. Accuracy every 25 steps is shown in Table~\ref{tab:results} in Appendix~\ref{app:accuracy}. For example, student accuracy in mathematics increases from $11.7$ to $30.1$ with Vanilla and $29.4$ with \method{}; in code, from $21.8$ to $28.5$ and $28.1$; and in instruction following, from $15.6$ to $59.8$ and $60.1$, respectively. Besides following similar increasing trends, \method{} and Vanilla OPD
remain closely matched in accuracy across the three domains. The relative \method{}--Vanilla difference is
$-0.90\%\pm2.36\%$ in mathematics, $-1.28\%\pm1.09\%$ in code, and
$-0.31\%\pm1.15\%$ in instruction following. Across all comparisons, the
mean relative difference is $-0.83\%\pm1.62\%$, and the largest absolute difference is $5.00\%$. 

We then compare the training time of the two methods in Table~\ref{tab:speed}. Overall, \method{} speeds up training by $1.8\times$ in mathematics, $2.2\times$ in code, and $1.3\times$ in instruction following. \method{} generates $34\%$, $48\%$, and $25\%$ fewer tokens in
mathematics, code, and instruction following, corresponding to token-count
ratios of $1.5\times$, $1.9\times$, and $1.3\times$. Generation is
$2.0\times$, $2.4\times$, and $1.4\times$ faster, respectively, with
throughput increasing from $3.5$k to $4.5$k tokens/s in mathematics and
from $4.4$k to $5.6$k tokens/s in code. The larger generation gains in
mathematics and code reflect their longer rollouts: truncation reduces
not only the number of generated tokens but also the cached context
processed. Skipped later positions would otherwise attend to larger KV
caches, so removing them reduces KV area per generated token. In
instruction following, rollouts are much shorter, so the additional
KV-area reduction is smaller, leaving the generation speedup close to
the token-count ratio. Student updates are also $1.2$--$1.8\times$ faster, as the student's forward and backward passes process fewer tokens, and the remaining time (Other column in Table~\ref{tab:speed}), mostly the forward pass that recomputes old log-probabilities, falls by $1.1$--$1.7\times$ for the same reason.

Figure~\ref{fig:dynamics} shows that
\method{} and Vanilla follow similar overall training dynamics across all
three domains: the entropy, distillation-loss, gradient-norm, and response-length curves
have similar shapes over training (since \method{} truncates its rollouts, its response length is the estimated complete length per rollout, $\widehat N_k/B$ from Eq.~\eqref{eq:lagged}), while \method{} exhibits larger local
fluctuations. These fluctuations come from the variation introduced by random truncation.
Blockwise inverse-probability weighting preserves the expected contribution
of truncated tokens, but positions retained with lower probability receive
larger weights, so the weighted gradient varies across batches depending on
which positions are retained. We provide more discussion of these fluctuations in Section~\ref{sec:batch-size-math}. The estimated complete rollout lengths under \method{} remain close to Vanilla.



%

\begin{figure*}[t]
\centering
\includegraphics[width=\textwidth]{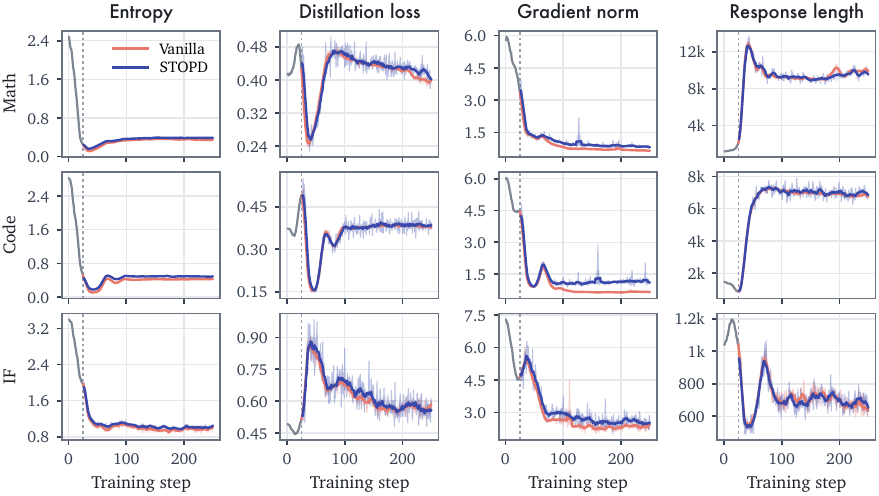}
\caption{Training dynamics of the 1.7B student with the shared warmup in gray.}
\label{fig:dynamics}
\end{figure*}

\subsection{Scaling}
\label{sec:scaling}

\begin{figure*}[t]
\centering
\includegraphics[width=\textwidth]{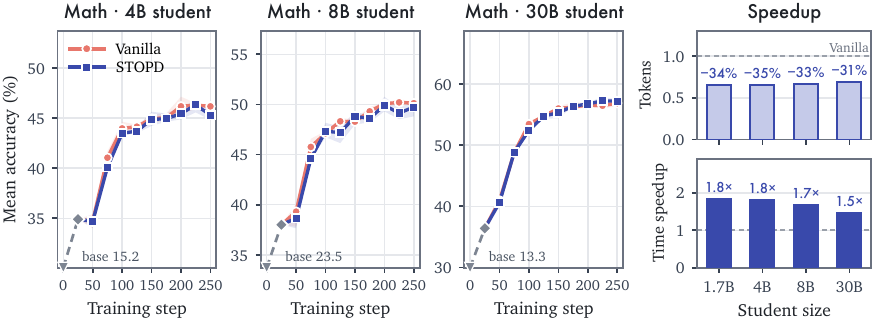}
\caption{Scaling in mathematics: accuracy for 4B and 8B students distilled from
Qwen3-30B-A3B-Instruct-2507 (first two panels) and a 30B student distilled from Qwen3-235B-A22B-Instruct-2507 (third panel), and token and training-time
speedups across student sizes, including the 1.7B main setting of Section~\ref{sec:main-results} (right).}
\label{fig:scaling}
\end{figure*}

\begin{figure}[h]
\centering
\includegraphics[width=\textwidth]{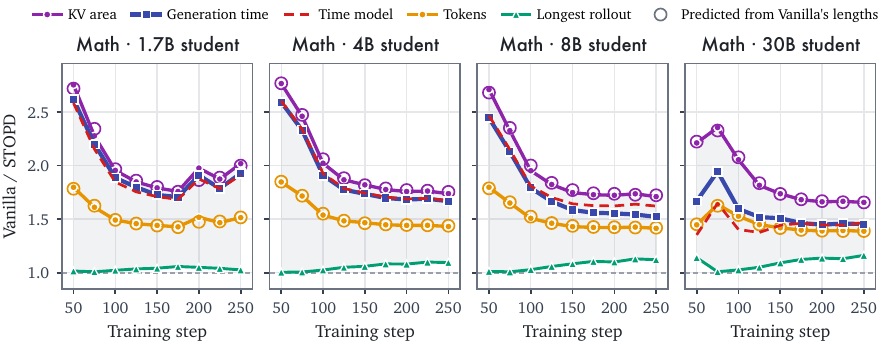}
\caption{Vanilla-to-\method{} ratios of KV area, generation time, token count, and longest rollout for each student. Open markers show predicted ratios, the dashed line shows our time model, and the gray band spans the range between the longest-rollout and KV-area ratios.}
\label{fig:time-model}
\end{figure}

\begin{figure*}[t]
\centering
\includegraphics[width=\textwidth]{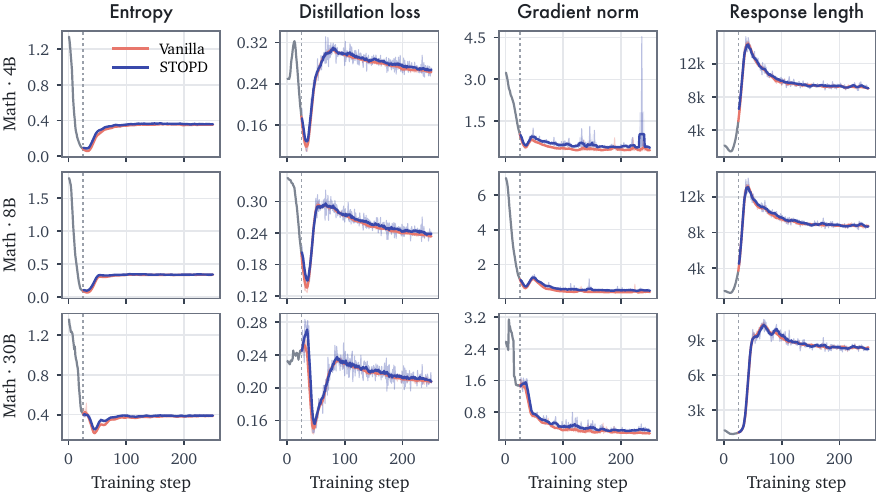}
\caption{Training dynamics of the 4B and 8B students and the 30B student in mathematics with the shared warmup in gray.}
\label{fig:scaling-dynamics}
\end{figure*}

{We then evaluate \method{} at larger model scales. Figure~\ref{fig:scaling} shows that both methods achieve higher accuracy when scaling to larger teachers and students, with Qwen3-30B-A3B-Instruct-2507 distilled into Qwen3-4B-Base and Qwen3-8B-Base, and Qwen3-235B-A22B-Instruct-2507 distilled into Qwen3-30B-A3B-Base. Accuracy every 25 steps is shown in Table~\ref{tab:results} in Appendix~\ref{app:accuracy}, and additional time decomposition in Table~\ref{tab:speed-scaling} in Appendix~\ref{app:scaling-cost}. For the 4B and 8B students, accuracy increases from $15.2$ and $23.5$ to $46.2$ and $50.1$ with Vanilla, and to $45.3$ and $49.7$ with \method{}, respectively. For the 30b student, accuracy increases from $13.3$ to $56.9$ with Vanilla and to $57.2$ with \method{} (Figures~\ref{fig:scaling} and~\ref{fig:scaling-dynamics}). 

The relative \method{}--Vanilla differences are $-0.96\%\pm0.90\%$ and $-1.09\%\pm1.26\%$ for the 4B and 8B students with the 30B teacher, and $-0.13\%\pm1.00\%$ for the 30B student with the 235B teacher. The corresponding largest absolute differences are $2.43\%$, $2.62\%$, and $1.87\%$, respectively, compared with $-0.90\%\pm2.36\%$ and $5.00\%$ in the main 4B-teacher/1.7B-student mathematics setting (Section~\ref{sec:main-results}). \method{} matches Vanilla more closely at larger scales. Figure~\ref{fig:scaling-dynamics} further shows that the training dynamics of the larger students follow those of Vanilla.

In terms of speed, \method{} also retains its gains at larger model scales. It accelerates the full training step by $1.8\times$, $1.7\times$, and $1.5\times$ for the 4B, 8B, and 30B students and generates $35\%$, $33\%$, and $31\%$ fewer tokens, respectively. We analyze these speedups with the time model of Section~\ref{sec:savings} in Figure~\ref{fig:time-model}. For each student, the figure shows the Vanilla-to-\method{} ratios of the KV area, the generation time, the tokens, and the longest rollout over training. The open markers are the token-count and KV-area ratios predicted by Eq.~\eqref{eq:method-kv-area-ratio}. The dashed line is the generation speedup predicted by our time model.

We first show that the measured token-count and KV-area ratios match these predictions under our time model. The KV-area ratios are $2.03\times$, $2.02\times$, $1.96\times$, and $1.86\times$ for the 1.7B, 4B, 8B, and 30B students and exceed the token-count ratios of $1.52\times$, $1.54\times$, $1.50\times$, and $1.45\times$. This is because \method{} skips later tokens that attend to longer cached contexts. Generation takes as many decoding rounds as the longest rollout has tokens. Each decoding round has a fixed cost plus a cost proportional to the cached context processed, and the KV area measures this cumulative context. When fixed costs per decoding round dominate, generation time follows the number of decoding rounds, and the speedup equals the longest-rollout ratio of $1.03$--$1.09\times$. When attention over long contexts dominates, the speedup equals the KV-area ratio. The generation speedup therefore lies between these two ratios (Eq.~\eqref{eq:app-time-model}), the gray area in Figure~\ref{fig:time-model}. Its position follows the two scheduling cases of Section~\ref{sec:savings}. 

For the 1.7B, 4B, and 8B students, a batch whose rollouts all start together would need more than twice the KV-cache capacity of the rollout engines. Queued rollouts therefore replace completed ones and time savings track the reduction in KV area (Eq.~\eqref{eq:app-full-cache}). Our time model predicts generation speedups of $1.92\times$, $1.93\times$, and $1.82\times$. Their generation is $1.95\times$, $1.92\times$, and $1.77\times$ faster, respectively. For the 30B student, a full batch fills about half of each rollout engine's KV cache. All rollouts therefore start together and truncation accelerates generation by reducing the context processed in each decoding round (Eq.~\eqref{eq:app-fixed-batch}). Our time model predicts a generation speedup of $1.45\times$. Its generation speedup of $1.55\times$ exceeds the token-count ratio of $1.45\times$. A larger training batch would exceed the KV-cache capacity, so queued rollouts would replace completed ones and the generation speedup would approach the KV-area ratio of $1.86\times$. Together, these results show that the time model of Section~\ref{sec:savings} captures the generation speedup across model scales.}

\subsection{Optimizer}
\label{sec:optimizers}

\begin{figure*}[t]
\centering
\includegraphics[width=\textwidth]{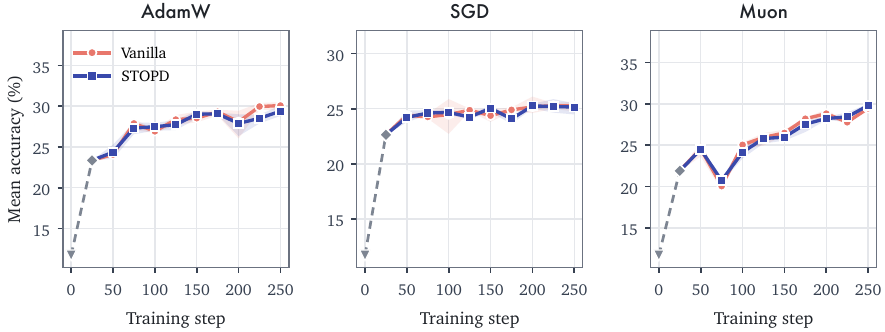}
\caption{Accuracy against training step in mathematics under AdamW in the main setting of Section~\ref{sec:main-results} and under SGD and Muon, with the base model and each optimizer's warmup in gray.}
\label{fig:optim}
\end{figure*}
Along with the main AdamW setting (Section~\ref{sec:experimental-setup}), we further evaluate whether \method{} generalizes across common optimizers by testing it with SGD and Muon (Figure~\ref{fig:optim}). SGD uses learning rate $5\times10^{-4}$. Muon \citep{jordan2024muon} uses learning rate $5\times10^{-6}$ and momentum $0.95$ on the hidden weight matrices, with AdamW on the remaining parameters. Accuracy every 25 steps is shown in Table~\ref{tab:results-optim} in Appendix~\ref{app:accuracy}, and additional time decomposition in Table~\ref{tab:time-optim} in Appendix~\ref{app:scaling-cost}.
Under SGD, accuracy increases from $11.7$ to $25.2$ with Vanilla
and $25.1$ with \method{}, with the largest absolute relative difference
of $3.2\%$. Under Muon, it increases from $11.7$ to $29.5$ and
$29.8$, respectively, with the largest absolute relative difference of
$3.6\%$. The computational gains also hold under both optimizers. Under SGD,
\method{} generates $41\%$ fewer tokens and accelerates the full training
step by $2.2\times$. Under Muon, it generates $37\%$ fewer tokens and
achieves a $1.9\times$ speedup. The training dynamics also remain similar to Vanilla under both SGD and
Muon (Figure~\ref{fig:optim-dynamics}).


\begin{figure*}[t]
\centering
\includegraphics[width=\textwidth]{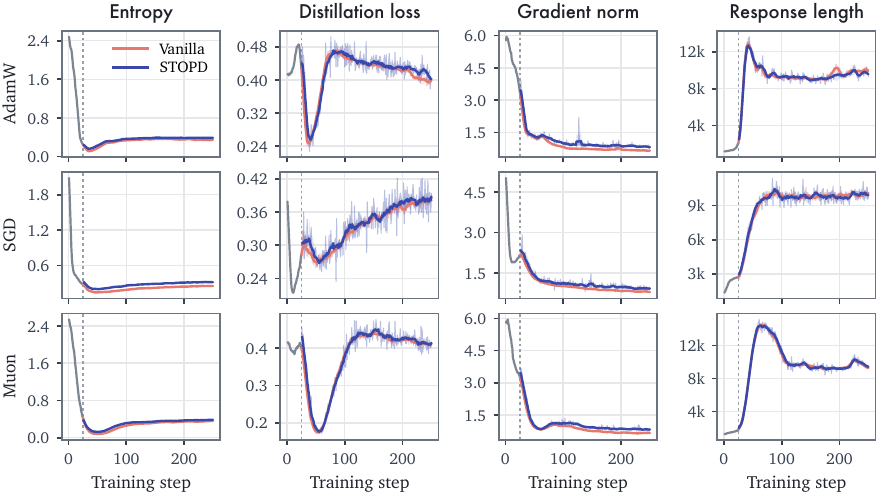}
\caption{Training dynamics in mathematics (1.7B $\leftarrow$ 4B) under AdamW and under SGD and Muon, with each optimizer's warmup in gray.}
\label{fig:optim-dynamics}
\end{figure*}


\subsection{Block Width}
\label{sec:block-width}

\begin{figure*}[t]
\centering
\includegraphics[width=\textwidth]{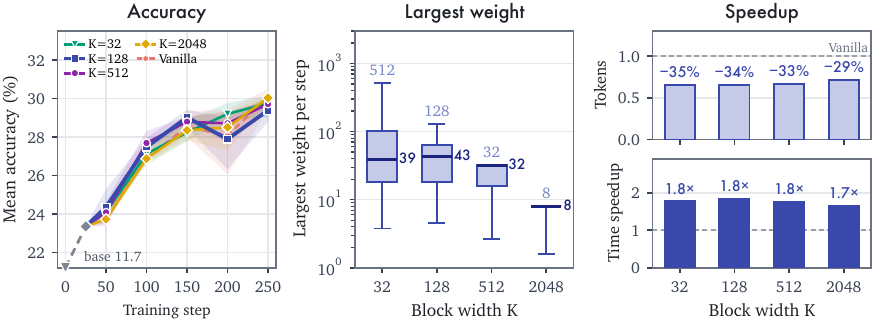}
\caption{Block width in mathematics (1.7B $\leftarrow$ 4B): accuracy against training step with the base model and the shared warmup in gray (left), the largest weight in each training step as a box plot with its bound $L/K$ above and its median beside each box (middle), and tokens relative to Vanilla and the training-time speedup} (right).
\label{fig:blockwidth}
\end{figure*}

We next study how the block width $K$ affects the trade-off between
weight magnitude and computational savings. We evaluate
$K\in\{32,128,512,2048\}$ over three
seeds in Figure~\ref{fig:blockwidth}. Accuracy every 25 steps is shown in Table~\ref{tab:results-blockwidth} in Appendix~\ref{app:accuracy}, and additional time decomposition in Table~\ref{tab:time-blockwidth} in Appendix~\ref{app:scaling-cost}. For
$K=32$, $128$, and $512$, \method{} generates $34.6\%$, $34.3\%$, and
$33.1\%$ fewer tokens, respectively, and remains $1.8\times$ faster than
Vanilla. At $K=2048$, the saving decreases to $28.6\%$ and the speedup to
$1.7\times$. Accuracy is consistent across block widths. Across all pairwise comparisons
among $K$, the mean absolute relative difference is
$1.29\%\pm0.97\%$, and the largest absolute relative difference is $4.55\%$. 

We then show the largest weight per step across $K$ and use $K=128$ as the default. This is because it balances weight control and
speedup. Reducing $K$ to $32$ raises the largest weight
from $128$ to $512$, while the median largest weight remains similar
($43$ versus $39$). Increasing $K$ to $512$ or $2048$ lowers the weight
bound further, but also reduces token savings. Thus, $K=128$ avoids a
large worst-case weight without sacrificing the savings of smaller blocks.

\subsection{Batch Size}
\label{sec:batch-size-math}

\begin{figure*}[t]
\centering
\includegraphics[width=\textwidth]{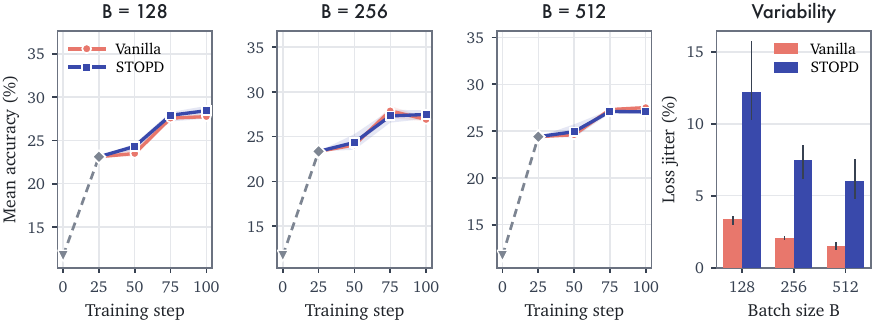}
\caption{Accuracy against training step in mathematics (1.7B $\leftarrow$ 4B) at $B=128$, $256$, and $512$ (left), and loss jitter (right).}
\label{fig:batchsize-math}
\end{figure*}

We vary the training batch size in mathematics with the Qwen3-4B teacher and Qwen3-1.7B student, training at $B=128$ and $B=512$ for 100 steps alongside the main setting of $B=256$ (Figure~\ref{fig:batchsize-math}). Accuracy every 25 steps is shown in Table~\ref{tab:results-batch-math} in Appendix~\ref{app:accuracy}, and additional time decomposition in Table~\ref{tab:time-batch-math} in Appendix~\ref{app:scaling-cost}.  Accuracy remains closely matched to Vanilla at every batch size, and the match tightens as the batch size grows. 
The relative \method{}--Vanilla differences are $2.22\%\pm1.16\%$, $0.44\%\pm2.30\%$, and $-0.44\%\pm1.54\%$, and the largest absolute difference falls from $3.40\%$ to $2.23\%$ and $1.82\%$. 
Over the first 100 steps, the speedup holds across all batch sizes, with \method{} accelerating the full training step by $2.0\times$ at batch size $128$ and $2.1\times$ at batch sizes $256$ and $512$ (Table~\ref{tab:time-batch-math} in Appendix~\ref{app:scaling-cost}). This is because larger batches can help sustain high generation concurrency and KV-cache occupancy, potentially enabling greater speedups.
Increasing the batch size also reduces the fluctuations introduced by random truncation. We measure them by the loss jitter, the median over training steps of the relative deviation $|\mathcal L_k-(\mathcal L_{k-1}+\mathcal L_{k+1})/2|/|\mathcal L_k|$ of the distillation loss $\mathcal L_k$ from the average of its two neighbors, averaged over seeds. The loss jitter of \method{} decreases from $12.2\%$ at $B=128$ to $7.5\%$ at $B=256$ and $6.0\%$ at $B=512$, close to the $8.6\%$ and $6.1\%$ that the $B^{-1/2}$ rate predicts from $B=128$, while Vanilla's decreases from $3.4\%$ to $2.1\%$ and $1.5\%$. Because truncations are sampled independently across rollouts, larger batches average over more realizations, reducing truncation-induced gradient fluctuations at a rate of $O_{\mathbb{P}}(B^{-1/2})$.

\section{Discussion}
\label{sec:discussion}

The principle behind \method{} could accelerate a broader range of autoregressive training workflows, including offline and online knowledge distillation, supervised fine-tuning, on-policy self-distillation, and multi-teacher on-policy distillation. 

\noindent\textbf{Knowledge distillation (KD).}
In offline KD~\citep{kim-rush-2016-sequence}, we store the truncated sequences and their inclusion probabilities during dataset construction, then use them for subsequent student training with blockwise IPW. In online KD~\citep{agarwal2024onpolicydistillationlanguagemodels}, we apply blockwise IPW within each training batch using the current teacher. Both settings could reduce generation and student-update computation to speed up training.

\noindent\textbf{Supervised fine-tuning.}
The same approach applies to SFT by randomly truncating reference responses and applying blockwise IPW to the retained next-token cross-entropy losses.

\noindent\textbf{On-policy self-distillation.}
OPSD uses the same model for both student and teacher, with the teacher additionally conditioned on a reference solution \citep{zhao2026opsd}. We can retain the complete reference solution, randomly truncate the student rollout, and apply blockwise IPW to the distillation losses computed on its retained tokens. This could also reduce computation cost.

\noindent\textbf{Multi-teacher on-policy distillation.}
MOPD assigns each student rollout to a teacher according to its task domain \citep{ma2026mopd}. We can keep this assignment, independently truncate each student rollout, and have its assigned teacher score only the retained prefix. Applying blockwise IPW preserves the expected local gradient sum within each domain and could reduce the computation needed to integrate the specialist teachers' capabilities into a shared student.

\section{Conclusion}
\label{sec:conclusion}

\method{} accelerates OPD by combining random rollout truncation with blockwise IPW. It preserves Vanilla OPD's gradient sum in expectation while controlling extreme token weights. We derive one-step update-error bounds for SGD, Adam, and Muon and analyze how token savings translate into training-time gains. Random truncation reduces both generated tokens and long-context computation, with additional throughput gains when long-context processing is costly or cache capacity limits concurrency. Across mathematics, code, and instruction following with 1.7B--30B students, \method{} generates $25$--$48\%$ fewer tokens and achieves $1.3$--$2.2\times$ training speedups while maintaining accuracy comparable to Vanilla OPD. These results show that OPD does not really need to complete every rollout to learn effectively.

\bibliography{references}

@article{McLeish_2011, title={A general method for debiasing a {Monte Carlo} estimator}, volume={17}, ISSN={0929-9629}, url={http://dx.doi.org/10.1515/mcma.2011.013}, DOI={10.1515/mcma.2011.013}, number={4}, journal={Monte Carlo Methods and Applications}, publisher={Walter de Gruyter GmbH}, author={McLeish, Don}, year={2011}, month=Jan }

@inproceedings{gu2024minillm,
      title={{MiniLLM}: Knowledge Distillation of Large Language Models},
      author={Yuxian Gu and Li Dong and Furu Wei and Minlie Huang},
      booktitle={The Twelfth International Conference on Learning Representations},
      year={2024},
      url={https://openreview.net/forum?id=5h0qf7IBZZ},
}

@inproceedings{agarwal2024onpolicydistillationlanguagemodels,
      title={On-Policy Distillation of Language Models: Learning from Self-Generated Mistakes},
      author={Rishabh Agarwal and Nino Vieillard and Yongchao Zhou and Piotr Stanczyk and Sabela {Ramos Garea} and Matthieu Geist and Olivier Bachem},
      booktitle={The Twelfth International Conference on Learning Representations},
      pages={21246--21263},
      year={2024},
      url={https://proceedings.iclr.cc/paper_files/paper/2024/hash/5be69a584901a26c521c2b51e40a4c20-Abstract-Conference.html},
}

@article{lu2025onpolicydistillation,
      author={Kevin Lu and {Thinking Machines Lab}},
      title={{On-Policy Distillation}},
      journal={Thinking Machines Lab: Connectionism},
      year={2025},
      doi={10.64434/tml.20251026},
      url={https://thinkingmachines.ai/blog/on-policy-distillation/},
}

@misc{ziheng2026moreearlystoppingrollout,
      title={Less is More: Early Stopping Rollout for On-Policy Distillation},
      author={Zhou, Ziheng and Jiaqi Li and Huacong Tang and Ying Nian Wu and Demetri Terzopoulos},
      year={2026},
      eprint={2605.27028},
      archivePrefix={arXiv},
      primaryClass={cs.LG},
      url={https://arxiv.org/abs/2605.27028},
}

@misc{sheng20261tokensenoughgradient,
      title={1\% of Tokens Can Be Enough: On Gradient Estimation in On-Policy Distillation},
      author={Huanxin Sheng and Zhiling Ye and Haonan Wang and Jian Wang and Jinjie Gu and Jian Kang},
      year={2026},
      eprint={2609.24432},
      archivePrefix={arXiv},
      primaryClass={cs.LG},
      url={https://arxiv.org/abs/2609.24432},
}

@misc{tallec2017unbiasingtruncatedbackpropagationtime,
      title={Unbiasing Truncated Backpropagation Through Time},
      author={Corentin Tallec and Yann Ollivier},
      year={2017},
      eprint={1705.08209},
      archivePrefix={arXiv},
      primaryClass={cs.NE},
      url={https://arxiv.org/abs/1705.08209},
}

@InProceedings{pmlr-v97-beatson19a,
  title = 	 {Efficient optimization of loops and limits with randomized telescoping sums},
  author =       {Beatson, Alex and Adams, Ryan P},
  booktitle = 	 {Proceedings of the 36th International Conference on Machine Learning},
  pages = 	 {534--543},
  year = 	 {2019},
  editor = 	 {Chaudhuri, Kamalika and Salakhutdinov, Ruslan},
  volume = 	 {97},
  series = 	 {Proceedings of Machine Learning Research},
  month = 	 {09--15 Jun},
  publisher =    {PMLR},
  url = 	 {https://proceedings.mlr.press/v97/beatson19a.html}
}

@book{boucheron2013concentration,
  author = {Boucheron, St{\'e}phane and Lugosi, G{\'a}bor and Massart, Pascal},
  title = {Concentration Inequalities: A Nonasymptotic Theory of Independence},
  publisher = {Oxford University Press},
  year = {2013},
  doi = {10.1093/acprof:oso/9780199535255.001.0001}
}

@inproceedings{kingma2015adam,
  author = {Kingma, Diederik P. and Ba, Jimmy},
  title = {{Adam}: A Method for Stochastic Optimization},
  booktitle = {International Conference on Learning Representations},
  year = {2015},
  url = {https://arxiv.org/abs/1412.6980}
}

@misc{jordan2024muon,
  author = {Keller Jordan and Yuchen Jin and Vlado Boza and Jiacheng You and Franz Cesista and Laker Newhouse and Jeremy Bernstein},
  title = {{Muon}: An optimizer for hidden layers in neural networks},
  year = {2024},
  url = {https://kellerjordan.github.io/posts/muon/},
  note = {Reference implementation, revision \href{https://github.com/KellerJordan/Muon/blob/64560829f70a216f8263ea8ae0fefeac8b451d36/muon.py}{\texttt{6456082}}}
}

@article{andersson2016operator,
  author = {Andersson, Fredrik and Carlsson, Marcus and Perfekt, Karl-Mikael},
  title = {Operator-{Lipschitz} estimates for the singular value functional calculus},
  journal = {Proceedings of the American Mathematical Society},
  volume = {144},
  number = {5},
  pages = {1867--1875},
  year = {2016},
  doi = {10.1090/proc/12843}
}

@inproceedings{hendrycks2021math,
  title = {Measuring Mathematical Problem Solving With the {MATH} Dataset},
  author = {Hendrycks, Dan and Burns, Collin and Kadavath, Saurav and Arora, Akul and Basart, Steven and Tang, Eric and Song, Dawn and Steinhardt, Jacob},
  booktitle = {Proceedings of the Neural Information Processing Systems Track on Datasets and Benchmarks},
  volume = {1},
  year = {2021},
  url = {https://arxiv.org/abs/2103.03874}
}

@inproceedings{lewkowycz2022minerva,
  title = {Solving Quantitative Reasoning Problems with Language Models},
  author = {Lewkowycz, Aitor and Andreassen, Anders and Dohan, David and Dyer, Ethan and Michalewski, Henryk and Ramasesh, Vinay and Slone, Ambrose and Anil, Cem and Schlag, Imanol and Gutman-Solo, Theo and Wu, Yuhuai and Neyshabur, Behnam and Gur-Ari, Guy and Misra, Vedant},
  booktitle = {Advances in Neural Information Processing Systems},
  volume = {35},
  pages = {3843--3857},
  year = {2022},
  doi = {10.52202/068431-0278},
  url = {https://arxiv.org/abs/2206.14858}
}

@article{qwen2025qwen3,
  title = {{Qwen3} Technical Report},
  author = {An Yang and Anfeng Li and Baosong Yang and Beichen Zhang and Binyuan Hui and Bo Zheng and Bowen Yu and Chang Gao and Chengen Huang and Chenxu Lv and Chujie Zheng and Dayiheng Liu and Fan Zhou and Fei Huang and Feng Hu and Hao Ge and Haoran Wei and Huan Lin and Jialong Tang and Jian Yang and Jianhong Tu and Jianwei Zhang and Jianxin Yang and Jiaxi Yang and Jing Zhou and Jingren Zhou and Junyang Lin and Kai Dang and Keqin Bao and Kexin Yang and Le Yu and Lianghao Deng and Mei Li and Mingfeng Xue and Mingze Li and Pei Zhang and Peng Wang and Qin Zhu and Rui Men and Ruize Gao and Shixuan Liu and Shuang Luo and Tianhao Li and Tianyi Tang and Wenbiao Yin and Xingzhang Ren and Xinyu Wang and Xinyu Zhang and Xuancheng Ren and Yang Fan and Yang Su and Yichang Zhang and Yinger Zhang and Yu Wan and Yuqiong Liu and Zekun Wang and Zeyu Cui and Zhenru Zhang and Zhipeng Zhou and Zihan Qiu},
  journal = {arXiv preprint arXiv:2505.09388},
  year = {2025},
  url = {https://arxiv.org/abs/2505.09388}
}

@article{wang2025nemotroncascade,
  title = {{Nemotron-Cascade}: Scaling Cascaded Reinforcement Learning for General-Purpose Reasoning Models},
  author = {Boxin Wang and Chankyu Lee and Nayeon Lee and Sheng-Chieh Lin and Wenliang Dai and Yang Chen and Yangyi Chen and Zhuolin Yang and Zihan Liu and Mohammad Shoeybi and Bryan Catanzaro and Wei Ping},
  journal = {arXiv preprint arXiv:2512.13607},
  year = {2025},
  url = {https://arxiv.org/abs/2512.13607}
}

@inproceedings{loshchilov2019adamw,
  title = {Decoupled Weight Decay Regularization},
  author = {Ilya Loshchilov and Frank Hutter},
  booktitle = {International Conference on Learning Representations (ICLR)},
  year = {2019},
  url = {https://arxiv.org/abs/1711.05101}
}

@article{zhou2023ifeval,
  title = {Instruction-Following Evaluation for Large Language Models},
  author = {Zhou, Jeffrey and Lu, Tianjian and Mishra, Swaroop and Brahma, Siddhartha and Basu, Sujoy and Luan, Yi and Zhou, Denny and Hou, Le},
  journal = {arXiv preprint arXiv:2311.07911},
  year = {2023},
  url = {https://arxiv.org/abs/2311.07911}
}

@inproceedings{liu2023evalplus,
  title = {Is Your Code Generated by {ChatGPT} Really Correct? {R}igorous Evaluation of Large Language Models for Code Generation},
  author = {Liu, Jiawei and Xia, Chunqiu Steven and Wang, Yuyao and Zhang, Lingming},
  booktitle = {Advances in Neural Information Processing Systems},
  year = {2023},
  url = {https://arxiv.org/abs/2305.01210}
}

@inproceedings{jain2024livecodebench,
  title = {{LiveCodeBench}: Holistic and Contamination Free Evaluation of Large Language Models for Code},
  author = {Jain, Naman and Han, King and Gu, Alex and Li, Wen-Ding and Yan, Fanjia and Zhang, Tianjun and Wang, Sida and Solar-Lezama, Armando and Sen, Koushik and Stoica, Ion},
  booktitle = {The Thirteenth International Conference on Learning Representations},
  pages = {58791--58831},
  year = {2025},
  url = {https://proceedings.iclr.cc/paper_files/paper/2025/hash/94074dd5a072d28ff75a76dabed43767-Abstract-Conference.html}
}

@misc{deepcoder2025dataset,
  title = {{DeepCoder-Preview-Dataset}},
  author = {{Agentica}},
  howpublished = {\url{https://huggingface.co/datasets/agentica-org/DeepCoder-Preview-Dataset}},
  year = {2025}
}

@misc{nvidia2025nemotronrlif,
  title = {{Nemotron-RL-instruction\_following}},
  author = {{NVIDIA}},
  howpublished = {\url{https://huggingface.co/datasets/nvidia/Nemotron-RL-instruction_following}},
  year = {2025}
}

@misc{xu2026tip,
      title={{TIP: Token Importance in On-Policy Distillation}},
      author={Yuanda Xu and Hejian Sang and Zhengze Zhou and Ran He and Zhipeng Wang and Alborz Geramifard},
      year={2026},
      eprint={2604.14084},
      archivePrefix={arXiv},
      primaryClass={cs.LG},
      url={https://arxiv.org/abs/2604.14084},
}

@misc{wang2026teachability,
      title={{Not All Disagreement Is Learnable: Token Teachability in On-Policy Distillation}},
      author={Yuanyi Wang and Su Lu and Yanggan Gu and Pengkai Wang and Yifan Yang and Zhaoyi Yan and Congkai Xie and Jianmin Wu and Hongxia Yang},
      year={2026},
      eprint={2605.26844},
      archivePrefix={arXiv},
      primaryClass={cs.LG},
      url={https://arxiv.org/abs/2605.26844},
}

@misc{ma2026mopd,
      title={{MOPD: Multi-Teacher On-Policy Distillation for Capability Integration in LLM Post-Training}},
      author={Wenhan Ma and Jianyu Wei and Liang Zhao and Hailin Zhang and Bangjun Xiao and Lei Li and Qibin Yang and Bofei Gao and Yudong Wang and Rang Li and Jinhao Dong and Zhifang Sui and Fuli Luo},
      year={2026},
      eprint={2606.30406},
      archivePrefix={arXiv},
      primaryClass={cs.CL},
      url={https://arxiv.org/abs/2606.30406},
}

@misc{zhao2026opsd,
      title={{Self-Distilled Reasoner: On-Policy Self-Distillation for Large Language Models}},
      author={Siyan Zhao and Zhihui Xie and Mengchen Liu and Jing Huang and Guan Pang and Feiyu Chen and Aditya Grover},
      year={2026},
      eprint={2601.18734},
      archivePrefix={arXiv},
      primaryClass={cs.LG},
      url={https://arxiv.org/abs/2601.18734},
}

@misc{wu2026lightningopd,
      title={{Lightning OPD: Efficient Post-Training for Large Reasoning Models with Offline On-Policy Distillation}},
      author={Yecheng Wu and Song Han and Han Cai},
      year={2026},
      eprint={2604.13010},
      archivePrefix={arXiv},
      primaryClass={cs.LG},
      url={https://arxiv.org/abs/2604.13010},
}

@inproceedings{kim-rush-2016-sequence,
    title = "Sequence-Level Knowledge Distillation",
    author = "Kim, Yoon  and
      Rush, Alexander M.",
    editor = "Su, Jian  and
      Duh, Kevin  and
      Carreras, Xavier",
    booktitle = "Proceedings of the 2016 Conference on Empirical Methods in Natural Language Processing",
    month = nov,
    year = "2016",
    address = "Austin, Texas",
    publisher = "Association for Computational Linguistics",
    url = "https://aclanthology.org/D16-1139/",
    doi = "10.18653/v1/D16-1139",
    pages = "1317--1327"
}

@misc{huang2025selectkd,
      title={{SelecTKD: Selective Token-Weighted Knowledge Distillation for LLMs}},
      author={Haiduo Huang and Jiangcheng Song and Yadong Zhang and Pengju Ren},
      year={2025},
      eprint={2510.24021},
      archivePrefix={arXiv},
      primaryClass={cs.CL},
      url={https://arxiv.org/abs/2510.24021},
}

@misc{yang2026pruneopdefficientreliableonpolicy,
      title={{Prune-OPD: Efficient and Reliable On-Policy Distillation for Long-Horizon Reasoning}},
      author={Zhicheng Yang and Zhijiang Guo and Yifan Song and Minrui Xu and Yongxin Wang and Yiwei Wang and Xiaodan Liang and Jing Tang},
      year={2026},
      eprint={2605.07804},
      archivePrefix={arXiv},
      primaryClass={cs.LG},
      url={https://arxiv.org/abs/2605.07804},
}

@misc{xin2026escapingklagreementtrap,
      title={{Escaping the KL Agreement Trap in On-Policy Distillation}},
      author={Haoran Xin and Anhao Zhao and Ying Sun and Jin Li and Xiaoyu Shen and Hui Xiong},
      year={2026},
      eprint={2606.09471},
      archivePrefix={arXiv},
      primaryClass={cs.LG},
      url={https://arxiv.org/abs/2606.09471},
}

@misc{zhang2026rolloutsnecessaryonpolicydistillation,
      title={{Are Full Rollouts Necessary for On-Policy Distillation?}},
      author={Yaocheng Zhang and Jiajun Chai and Yuqian Fu and Songjun Tu and Xiaohan Wang and Wei Lin and Guojun Yin and Qichao Zhang and Yuanheng Zhu and Dongbin Zhao},
      year={2026},
      eprint={2605.31490},
      archivePrefix={arXiv},
      primaryClass={cs.CL},
      url={https://arxiv.org/abs/2605.31490},
}

@inproceedings{zhang-etal-2026-fast,
    title = {{Fast and Effective On-Policy Distillation from Reasoning Prefixes}},
    author = "Zhang, Dongxu  and
      Yang, Zhichao  and
      Janghorbani, Sepehr  and
      Han, Jun  and
      Ressler, II, Andrew  and
      Qian, Qian  and
      Lyng, Gregory D  and
      Batra, Sanjit Singh  and
      Tillman, Robert E.",
    editor = "Liakata, Maria  and
      Moreira, Viviane P.  and
      Zhang, Jiajun  and
      Jurgens, David",
    booktitle = "Findings of the {A}ssociation for {C}omputational {L}inguistics: {ACL} 2026",
    month = jul,
    year = "2026",
    address = "San Diego, California, United States",
    publisher = "Association for Computational Linguistics",
    url = "https://aclanthology.org/2026.findings-acl.1276/",
    doi = "10.18653/v1/2026.findings-acl.1276",
    pages = "25553--25569",
    ISBN = "979-8-89176-395-1"
}

@misc{zhao2026prefixguidedonpolicydistillationmining,
      title={{Prefix-Guided On-Policy Distillation: Mining Golden Trajectories from Rollouts}},
      author={Qingfei Zhao and Huan Song and Shuyu Tian and Jiawei Shao and Xuelong Li},
      year={2026},
      eprint={2606.21994},
      archivePrefix={arXiv},
      primaryClass={cs.LG},
      url={https://arxiv.org/abs/2606.21994},
}
\bibliographystyle{icml2026}

\clearpage
\appendix
\newtheorem{sopdtheorem}{Theorem}[section]
\newtheorem{sopdproposition}[sopdtheorem]{Proposition}
\newtheorem{sopdlemma}[sopdtheorem]{Lemma}
\newtheorem{sopdcorollary}[sopdtheorem]{Corollary}

\section{Gradient Preservation and Normalization}
\label{app:gradient-preservation}

We prove that random truncation preserves the expected gradient sum and bound its gradient error relative to Vanilla OPD, including the effect of lagged normalization. We also analyze position weights, the relationship between token savings and variance, and the effect of batch size. The proofs specialize survival-weighted estimation \citep[Section~2]{McLeish_2011} and randomized telescoping \citep[Section~2.1 and supplementary Proposition~2.1]{pmlr-v97-beatson19a} to finite, naturally terminating rollouts.

\subsection{Conditional preservation and estimation error}
\label{app:preservation}

Fix a training step $k$, batch size $B$, max rollout length $L$, and
block width $K$. Its history $\mathcal H_k$ contains the current student parameters $\theta_k$ and the positive denominator $C_k$ chosen before the current cutoffs. For each prompt, consider the hypothetical complete rollout generated until its first natural end-of-sequence token or the max length $L$, with length $1\leq T_i\leq L$. The natural end-of-sequence token is included in this length and receives its ordinary local supervision. Let $\mathcal Y_k$ contain these complete rollouts and any randomness used for their teacher feedback. Throughout the fixed-step analysis, $\E_S$ and $\Prob_S$ denote expectation and probability conditional on $(\mathcal H_k,\mathcal Y_k)$; only the independent anchors remain random. We suppress $k$ until the lagged-denominator analysis.

The following conditions formalize the setting in Section~\ref{sec:method}.
\begin{enumerate}
\item Changing the remaining generation budget does not affect the next-token distribution when the prompt and previously generated tokens are fixed.
\item We use Vanilla OPD's differentiation rule for sampled tokens and teacher targets. The differentiable local loss $\ell_{i,t}(\theta)$ depends only on the prompt, tokens up to position $t$, and local feedback. Random truncation leaves this loss unchanged at retained positions. Full and truncated rollouts share any randomness used for local feedback.
\item Conditional on history, the anchors are mutually independent, uniform on $\{0,\ldots,L-1\}$, and independent of the complete token and feedback randomness. The method uses the exact weights $1/p_t$ and treats them and $C$ as constants in differentiation.
\end{enumerate}
We compare each observed rollout with its corresponding complete rollout. The local gradients $h_{i,t}=\nabla_\theta\ell_{i,t}(\theta_k)\in\mathbb R^d$ are finite. Norms below are Euclidean.

For the estimators in Eqs.~\eqref{eq:ipw}, \eqref{eq:loss}, and~\eqref{eq:lagged}, write
\[
G=\sum_{i=1}^B\sum_{t=1}^{T_i}h_{i,t},\qquad N=\sum_{i=1}^BT_i,
\qquad g^{\mathrm V}=G/N,\qquad \widehat g=\widehat G/C.
\]
Use the stopping rule in Eq.~\eqref{eq:stop} and let $I_{i,t}=\ind\{S_i\geq t\}$ indicate whether position $t$ is retained. Then $\E_S I_{i,t}=p_t$, with $p_t$ given in Eq.~\eqref{eq:survival}. To express the variance, define
\begin{equation}
\alpha_b=\frac{LK}{(L-bK)(L-(b-1)K)}\quad(1\leq bK<L),
\qquad
v(T)=\sum_{b=1}^{\lfloor(T-1)/K\rfloor}\alpha_b(T-bK)^2.
\label{eq:app-v}
\end{equation}
Empty sums equal zero, so $v(T)=0$ when $T\leq K$.

\begin{sopdtheorem}[Gradient preservation and error]
\label{thm:app-main}
Under the stated conditions, full and truncated generation can be coupled so that every observed local gradient equals its full-generation contribution. Consequently,
\begin{equation}
\E_S\widehat G=G,\qquad \E_S\widehat N=N,\qquad
\E_S\widehat g=\frac NCg^{\mathrm V}.
\label{eq:app-unbiased}
\end{equation}

For $V:=\E_S\|\widehat G-G\|^2$ and every $C>0$, the implemented estimator satisfies

\begin{equation}
\boxed{\E_S\|\widehat g-g^{\mathrm V}\|^2
=\frac{V}{C^2}
+\left(\frac{N-C}{C}\right)^2\|g^{\mathrm V}\|^2.}
\label{eq:app-mse}
\end{equation}

If $\|h_{i,t}\|\leq H$ for every token in the batch, then

\begin{equation}
\E_S\|\widehat g-g^{\mathrm V}\|^2
\leq\frac{H^2\sum_i v(T_i)}{C^2}
+\left(\frac{N-C}{C}\right)^2\|g^{\mathrm V}\|^2.
\label{eq:app-profile-mse}
\end{equation}

The variance bound $V\leq H^2\sum_i v(T_i)$ is attained when, within each rollout, all local gradients equal a fixed vector of norm $H$.
\end{sopdtheorem}

\begin{proof}
Couple full and truncated generation from the same prompt using the same
randomness for token sampling and local feedback. While the truncated run continues, identical tokens generated so far imply the same next-token distribution by budget independence, and the shared draw produces the same next token. By induction, the two runs generate identical tokens through position $M_i$, including a natural end-of-sequence token when reached. At each retained position, removing later tokens leaves the differentiable local loss function unchanged, so Vanilla OPD's differentiation rule gives the same gradient $h_{i,t}$ in both runs. Conditional on $(\mathcal H_k,\mathcal Y_k)$, these gradients are fixed; only their inclusion depends on the random cutoffs, as in the conditioning argument of \citet[Proposition~1 and Section~8]{tallec2017unbiasingtruncatedbackpropagationtime}.
For $t\leq T_i$, Eq.~\eqref{eq:survival} gives $\E_S I_{i,t}=p_t$, including when the final block is shorter than $K$. Therefore
\[
\widehat G=\sum_{i=1}^B\sum_{t=1}^{T_i}\frac{I_{i,t}}{p_t}h_{i,t},
\qquad
\E_S\widehat G=\sum_{i=1}^B\sum_{t=1}^{T_i}\frac{\E_S I_{i,t}}{p_t}h_{i,t}=G.
\]
Taking scalar contributions equal to $1$ gives $\E_S\widehat N=N$, and dividing by the fixed $C$ gives $\E_S\widehat g=G/C=(N/C)g^{\mathrm V}$, proving Eq.~\eqref{eq:app-unbiased}. Natural termination fixes which positions exist, while their inclusion probabilities depend only on $t,L,K$; the weights therefore do not require the complete lengths.

To quantify \method{}'s gradient error relative to Vanilla, center $\widehat g$ at its conditional mean $G/C$:
\begin{equation}
\widehat g-g^{\mathrm V}
=\frac{\widehat G-G}{C}+\left(\frac GC-\frac GN\right)
=\frac{\widehat G-G}{C}+\frac{N-C}{C}g^{\mathrm V}.
\label{eq:app-normalized-identity}
\end{equation}
Here $N,C,g^{\mathrm V}$ are fixed under $\E_S$ and $\E_S(\widehat G-G)=0$. Expanding the squared norm therefore gives
\begin{align*}
\E_S\|\widehat g-g^{\mathrm V}\|^2
&=\frac{\E_S\|\widehat G-G\|^2}{C^2}
+\left(\frac{N-C}{C}\right)^2\|g^{\mathrm V}\|^2\\
&\quad+\frac{2(N-C)}{C^2}\left\langle\E_S(\widehat G-G),g^{\mathrm V}\right\rangle\\
&=\frac{V}{C^2}
+\left(\frac{N-C}{C}\right)^2\|g^{\mathrm V}\|^2,
\end{align*}
proving Eq.~\eqref{eq:app-mse}. It remains to evaluate and bound $V$.

Write the centered error of rollout $i$ as
\[
e_i=\sum_{t=1}^{T_i}\left(\frac{I_{i,t}}{p_t}-1\right)h_{i,t},
\qquad \widehat G-G=\sum_{i=1}^B e_i,
\qquad \E_S e_i=0.
\]
Each $e_i$ depends only on its own anchor under the fixed conditioning. Thus, for $i\ne j$,
$\E_S\langle e_i,e_j\rangle=\langle\E_S e_i,\E_S e_j\rangle=0$ by anchor independence.
This reduces the batch variance to $V=\sum_i\E_S\|e_i\|^2$.

The indicators within a rollout share a cutoff: $I_{i,s}I_{i,t}=\ind\{S_i\geq\max(s,t)\}$. Hence
\begin{equation}
\E_S\left[\left(\frac{I_{i,s}}{p_s}-1\right)
\left(\frac{I_{i,t}}{p_t}-1\right)\right]
=\frac{p_{\max(s,t)}}{p_sp_t}-1
=\frac1{p_{\min(s,t)}}-1.
\label{eq:app-covariance}
\end{equation}
To express $\E_S\|e_i\|^2$ in terms of gradient tails, group token pairs by the block boundaries they both cross. The reciprocal survival probability changes at $bK+1$ by
\[
\frac1{p_{bK+1}}-\frac1{p_{bK}}
=\frac{L}{L-bK}-\frac{L}{L-(b-1)K}=\alpha_b,
\qquad 1\leq bK<L.
\]
Since $p_t^{-1}=1$ on the first block and is constant within each block, telescoping gives
\[
\frac1{p_{\min(s,t)}}-1
=\sum_{\substack{b\geq1\\bK<\min(s,t)}}\alpha_b.
\]
The coefficient $\alpha_b$ contributes exactly to pairs with both positions beyond $bK$. Expanding $\E_S\|e_i\|^2$ and regrouping these pairs yields
\begin{align}
\E_S\|e_i\|^2
&=\sum_{s,t=1}^{T_i}\left(\frac1{p_{\min(s,t)}}-1\right)
\langle h_{i,s},h_{i,t}\rangle\nonumber\\
&=\sum_{b=1}^{\lfloor(T_i-1)/K\rfloor}\alpha_b
\sum_{s,t=bK+1}^{T_i}\langle h_{i,s},h_{i,t}\rangle\nonumber\\
&=\sum_{b=1}^{\lfloor(T_i-1)/K\rfloor}\alpha_b
\left\|\sum_{t=bK+1}^{T_i}h_{i,t}\right\|^2.
\label{eq:app-tail-identity}
\end{align}
Summing over rollouts gives the exact batch variance. For the weighted count, each token contributes the scalar $1$, so its tail sum is $T_i-bK$ and the same identity gives
\begin{align}
V
&=\sum_{i=1}^B\E_S\|e_i\|^2
=\sum_{i=1}^B\sum_{b=1}^{\lfloor(T_i-1)/K\rfloor}
\alpha_b\left\|\sum_{t=bK+1}^{T_i}h_{i,t}\right\|^2,\nonumber\\
\operatorname{Var}_S(\widehat N)
&=\sum_{i=1}^B\sum_{b=1}^{\lfloor(T_i-1)/K\rfloor}
\alpha_b(T_i-bK)^2
=\sum_{i=1}^Bv(T_i).
\label{eq:app-exact-variance}
\end{align}

Under $\|h_{i,t}\|\leq H$, each tail satisfies
$\|\sum_{t=bK+1}^{T_i}h_{i,t}\|\leq\sum_{t=bK+1}^{T_i}\|h_{i,t}\|\leq H(T_i-bK)$.
Since $\alpha_b>0$, Eq.~\eqref{eq:app-exact-variance} gives
\begin{equation}
V\leq H^2\sum_{i=1}^B
\sum_{b=1}^{\lfloor(T_i-1)/K\rfloor}\alpha_b(T_i-bK)^2
=H^2\sum_{i=1}^Bv(T_i).
\label{eq:app-length-envelope}
\end{equation}
Substitution into the normalized-error decomposition yields
\[
\E_S\|\widehat g-g^{\mathrm V}\|^2
\leq\frac{H^2\sum_i v(T_i)}{C^2}
+\left(\frac{N-C}{C}\right)^2\|g^{\mathrm V}\|^2,
\]
which is Eq.~\eqref{eq:app-profile-mse}. When all gradients within each rollout equal a fixed vector of norm $H$, every tail has norm $H(T_i-bK)$, attaining the variance bound.
\end{proof}

\subsection{Position weights and token coverage}
\label{app:position-coverage}

Blockwise inverse-probability weighting compensates for how often each position is retained. A token at position $t$ is retained with probability $p_t$, so assigning it a deterministic weight $c_t$ gives an expected gradient contribution of $p_tc_th_{i,t}$. Matching Vanilla OPD for arbitrary token gradients requires $p_tc_t=1$. This condition is necessary by considering a single nonzero token gradient and sufficient by linearity. Thus $c_t=1/p_t$ is the only deterministic position weight that preserves every token's expected contribution. Since $p_t>0$ in Eq.~\eqref{eq:survival}, all these weights are finite.

We next calculate how often a position receives supervision in a batch. Among $n\geq1$ fixed complete rollouts that reach position $t$, independent random truncation retains that position with probability $p_t$ in each rollout. The number retained is therefore $Q_t\sim\operatorname{Binomial}(n,p_t)$, giving
\begin{equation}
\Prob_S(Q_t\geq1)=1-(1-p_t)^n.
\label{eq:app-position-coverage}
\end{equation}
To retain position $t$ at least once with probability at least $1-\delta$, solve $(1-p_t)^n\leq\delta$. For $0<p_t<1$ and $\delta\in(0,1)$, this requires
\begin{equation}
\Prob_S(Q_t\geq1)\geq1-\delta
\quad\Longleftrightarrow\quad
n\geq\left\lceil\frac{\log(1/\delta)}{-\log(1-p_t)}\right\rceil.
\label{eq:app-exposure-budget}
\end{equation}
Here $n$ counts rollouts that contain the position under full generation. Positions in the first block have $p_t=1$ and are retained in all such rollouts.

\subsection{Truncation variance and block width}
\label{app:variance-block-width}

Random truncation introduces variation into the weighted gradient sum. For a fixed rollout of length $T$, $v(T)$ from Eq.~\eqref{eq:app-v} is the variance of its weighted token-count estimate, and $D(T)$ is the expected token saving derived in Appendix~\ref{app:token-savings}. For every $1\leq T\leq L$, these quantities satisfy
\begin{equation}
v(T)\leq2TD(T),\qquad
V\leq2H^2\sum_iT_iD(T_i)
\quad\text{if }\|h_{i,t}\|\leq H.
\label{eq:app-variance-consequences}
\end{equation}
The bound connects expected token savings to gradient-sum variance through the rollout lengths.

To derive the bound, note that all tokens in a rollout share one stopping position, so their inclusion events are correlated. Applying the covariance formula in Eq.~\eqref{eq:app-covariance} to the weighted token count and grouping pairs by their smaller position gives
\[
v(T)=\sum_{t=1}^T\frac{1-p_t}{p_t}\{2(T-t)+1\}.
\]
Here, $2(T-t)+1$ counts the ordered pairs whose smaller position is $t$. Since $K\lfloor(t-1)/K\rfloor\leq t-1$ and $T\leq L$, the inclusion probability satisfies
\[
p_t\geq\frac{L-t+1}{L}\geq\frac{T-t+1}{T},
\qquad
\frac{2(T-t)+1}{p_t}
\leq T\frac{2(T-t)+1}{T-t+1}\leq2T.
\]
Multiplying by $1-p_t$ and summing proves $v(T)\leq2T\sum_t(1-p_t)=2TD(T)$. Substituting this result into the gradient-variance bound in Eq.~\eqref{eq:app-length-envelope} gives the second inequality in Eq.~\eqref{eq:app-variance-consequences}.

For rollouts reaching the max length, the effect of block width can be computed explicitly. Let $L=mK$, so there are $m$ equal blocks, and write $\mathsf H_m=\sum_{j=1}^m1/j$. The variance of the weighted count relative to $L^2$ is
\begin{equation}
\frac{v(L)}{L^2}=1-\frac{\mathsf H_m}{m}.
\label{eq:app-block-tradeoff}
\end{equation}
Setting $L=mK$ in Eq.~\eqref{eq:app-v} gives $\alpha_b(L-bK)^2/L^2=(m-b)/[m(m-b+1)]$, hence
\[
\frac{v(L)}{L^2}
=\frac1m\sum_{b=1}^{m-1}\frac{m-b}{m-b+1}
=\frac1m\sum_{j=2}^m\left(1-\frac1j\right)
=1-\frac{\mathsf H_m}{m}.
\]
For these max-length rollouts, larger blocks reduce the variance of the weighted count at the cost of more generation (Appendix~\ref{app:token-savings}). With one block, this variance is zero; as the number of blocks grows, $v(L)/L^2$ approaches $1$. The corresponding relative standard deviation is $\sqrt{1-\mathsf H_m/m}$, which is $0.99$, $0.98$, $0.93$, and $0.81$ for the block widths $K=32$, $128$, $512$, and $2048$ of Section~\ref{sec:block-width}.

\subsection{Error under lagged normalization}
\label{app:lagged-normalization}

The normalization contribution in Theorem~\ref{thm:app-main} is $\rho_k g_k^{\mathrm V}$, where $\rho_k=(N_k-C_k)/C_k$. For \method{}'s lagged denominator, $C_{W+1}=N_W$ after $W\geq1$ Vanilla OPD updates, and $C_k=\widehat N_{k-1}$ for $k>W+1$. We bound $\rho_k$ by comparing $N_k/B$ and $C_k/B$ with the conditional mean rollout lengths of their respective batches. This separates fluctuations due to rollout sampling and cutoffs from changes in the mean rollout length between updates.

For $j\in\{k-1,k\}$, let $\mathcal H_j$ be the history before sampling batch $j$, including its current student parameters. Assume additionally that, conditional on $\mathcal H_j$, the complete rollouts within batch $j$ are independent and identically distributed. The anchors have the independence specified in {condition~3 of Appendix~\ref{app:preservation}}. The current count $N_k$ sums complete lengths $T_{k,i}\in[1,L]$. To express the lagged count as a sum of rollout-level quantities, define the weighted rollout count and its upper bound by
\begin{equation}
\widehat T_{j,i}=\sum_{t=1}^{T_{j,i}}\frac{I_{j,i,t}}{p_t},
\qquad U_{L,K}=\sum_{t=1}^L\frac1{p_t}.
\label{eq:app-count-range}
\end{equation}
For a truncated batch, $\widehat N_j/B=B^{-1}\sum_i\widehat T_{j,i}$ is an average of independent and identically distributed counts conditional on $\mathcal H_j$. The first token is retained with weight one, giving $\widehat T_{j,i}\geq1$; the bounds $I_{j,i,t}\leq1$ and $T_{j,i}\leq L$ give $\widehat T_{j,i}\leq U_{L,K}$. Summing the lower bounds gives $C_k\geq B$, also for the complete lengths used at initialization.

The concentration of these averages depends on the means and variances of their rollout-level counts. For complete lengths, write
\begin{equation}
\mu_j=\E[T_{j,i}\mid\mathcal H_j],\qquad
\sigma_j^2=\operatorname{Var}(T_{j,i}\mid\mathcal H_j).
\label{eq:app-length-moments}
\end{equation}
Conditional on $(\mathcal H_j,T_{j,i})$, the weighted count depends only on its independent anchor. Equations~\eqref{eq:app-unbiased} and~\eqref{eq:app-exact-variance}, applied to one rollout, therefore give conditional mean $T_{j,i}$ and variance $v(T_{j,i})$. Averaging over the rollout length, the laws of total expectation and total variance give
\begin{align}
\E[\widehat T_{j,i}\mid\mathcal H_j]
&=\E[T_{j,i}\mid\mathcal H_j]=\mu_j,\nonumber\\
\operatorname{Var}(\widehat T_{j,i}\mid\mathcal H_j)
&=\sigma_j^2+\E[v(T_{j,i})\mid\mathcal H_j]=:\tau_j^2.
\label{eq:app-total-count-variance}
\end{align}
Thus weighting preserves the conditional mean count and adds the expected cutoff variance to the rollout-length variance. The corresponding conditional centers are $\mu_k$ for $N_k/B$ and $\mu_{k-1}$ for $C_k/B$. These centers may differ under adaptively chosen policies; write $\Delta_{\mu,k}=|\mu_k-\mu_{k-1}|$ for this mean-length drift.

The centered counts $N_k-B\mu_k$ and $C_k-B\mu_{k-1}$ are sums of independent centered contributions conditional on their respective histories. The following form of Bernstein's inequality controls their deviations using the variance and range width of each summand.

\begin{sopdlemma}[Bernstein concentration]
\label{lem:app-concentration}
Let $F_1,\ldots,F_B$ be independent real random variables, each supported on an interval of length at most $u\geq0$, and let $s^2=B^{-1}\sum_i\operatorname{Var}(F_i)$. For every $x>0$,
\begin{equation}
\Prob\!\left(\left|\frac1B\sum_i(F_i-\E F_i)\right|
>\sqrt{\frac{2s^2x}{B}}+\frac{ux}{3B}\right)\leq2e^{-x}.
\label{eq:app-scalar-concentration}
\end{equation}
The statement also holds conditional on a history under which its assumptions hold.
\end{sopdlemma}

\begin{proof}
Write $F_i^\circ=F_i-\E F_i$, so $\sum_i\E(F_i^\circ)^2=Bs^2$. If $s^2=0$, every centered summand is zero almost surely and the result follows. Otherwise $u>0$. The moment form of Bernstein's inequality \citep[Theorem~2.10]{boucheron2013concentration} gives the one-sided threshold $\sqrt{2Bs^2x}+cx$ with failure probability $e^{-x}$ whenever, for some scale $c>0$,
\[
\sum_{i=1}^B\E|F_i^\circ|^m
\leq\frac{m!}{2}Bs^2c^{m-2},
\qquad m=3,4,\ldots.
\]
This absolute-moment condition suffices for both signs, since it bounds the positive-part moments in the cited theorem. To verify it here, each mean lies in its supporting interval, so $|F_i^\circ|\leq u$. Consequently,
\[
\sum_{i=1}^B\E|F_i^\circ|^m
\leq u^{m-2}\sum_{i=1}^B\E(F_i^\circ)^2
=u^{m-2}Bs^2
\leq\frac{m!}{2}Bs^2\left(\frac u3\right)^{m-2},
\]
where $m!\geq2\cdot3^{m-2}$ gives the last inequality. Thus the scale $c=u/3$ satisfies the required condition. Applying the one-sided bound to both signs and taking a union bound gives
\[
\Prob\!\left(\left|\sum_iF_i^\circ\right|
>\sqrt{2Bs^2x}+\frac{ux}{3}\right)\leq2e^{-x}.
\]
Dividing the sum and its threshold by $B$ yields the average-deviation bound in Eq.~\eqref{eq:app-scalar-concentration}, since
\[
\frac{\sqrt{2Bs^2x}+ux/3}{B}
=\sqrt{\frac{2s^2x}{B}}+\frac{ux}{3B}.
\]
Conditional on a history satisfying the lemma's assumptions, the same argument applies with its conditional means and variances.
\end{proof}

To apply Lemma~\ref{lem:app-concentration} to the two counts, denote its deviation threshold by
\begin{equation}
\mathfrak r_B(s^2,u,x)
:=\sqrt{\frac{2s^2x}{B}}+\frac{ux}{3B},
\qquad s^2,u\geq0,\quad x>0.
\label{eq:app-radius}
\end{equation}
Choosing $x=\log(4/\delta)$ gives $2e^{-x}=\delta/2$ for any $\delta\in(0,1)$, assigning failure probability at most $\delta/2$ to each count bound. The next theorem substitutes the two counts' respective variances and range widths into this threshold and combines their deviations with the mean-length drift.

\begin{samepage}
\begin{sopdtheorem}[Finite-batch error with lagged normalization]
\label{thm:app-length-tracking}
Under these additional sampling assumptions, fix a step $k\geq W+1$ and $\delta\in(0,1)$, and set $x=\log(4/\delta)$. Define the deviation radii for $N_k/B$ and $C_k/B$ by
\begin{align}
r_{N,k}&=\mathfrak r_B(\sigma_k^2,L-1,x), && k\geq W+1,\nonumber\\
r_{C,k}&=\mathfrak r_B(\tau_{k-1}^2,U_{L,K}-1,x), && k>W+1.
\label{eq:app-count-radii}
\end{align}
At the first truncated step, $C_{W+1}=N_W$, so use $r_{C,W+1}=\mathfrak r_B(\sigma_W^2,L-1,x)$. With probability at least $1-\delta$, both
\begin{equation}
\left|\frac{N_k}{B}-\mu_k\right|\leq r_{N,k},\qquad
\left|\frac{C_k}{B}-\mu_{k-1}\right|\leq r_{C,k}
\label{eq:app-count-tracking}
\end{equation}
hold. These deviations and the mean-length drift bound the relative normalization error on the same event:
\begin{equation}
|\rho_k|=\frac{|N_k-C_k|}{C_k}
\leq\frac{\Delta_{\mu,k}+r_{N,k}+r_{C,k}}{C_k/B}.
\label{eq:app-relative-normalization}
\end{equation}
Combining this normalization control with the current cutoff error gives a bound on the gradient discrepancy. If $\|h_{k,i,t}\|\leq H$ almost surely for the current batch, then for any $\zeta>0$ satisfying $\delta+\zeta<1$, the following holds with probability at least $1-\delta-\zeta$:
\begin{equation}
\|\widehat g_k-g_k^{\mathrm V}\|
\leq
\frac{\displaystyle \frac{H}{B}\sqrt{\frac{\sum_i v(T_{k,i})}{\zeta}}
+(\Delta_{\mu,k}+r_{N,k}+r_{C,k})\|g_k^{\mathrm V}\|}
{C_k/B}.
\label{eq:app-joint-gradient}
\end{equation}
\end{sopdtheorem}
\end{samepage}

\begin{proof}[Proof of Theorem~\ref{thm:app-length-tracking}]
We first apply Lemma~\ref{lem:app-concentration} to the two batch averages using their rollout-level moments. For $N_k=\sum_iT_{k,i}$, condition on $\mathcal H_k$: its summands have mean $\mu_k$, variance $\sigma_k^2$, and range $[1,L]$. For $k>W+1$, the denominator is $C_k=\sum_i\widehat T_{k-1,i}$; conditional on $\mathcal H_{k-1}$, Eq.~\eqref{eq:app-total-count-variance} gives summand mean $\mu_{k-1}$ and variance $\tau_{k-1}^2$, with range $[1,U_{L,K}]$. At the first truncated step, $C_{W+1}=N_W$, whose summands are complete lengths with variance $\sigma_W^2$ and range $[1,L]$. These choices give $r_{N,k}$ and the two cases of $r_{C,k}$. Taking $x=\log(4/\delta)$ assigns failure probability $\delta/2$ to each count bound:
\begin{align*}
\Prob\!\left(\left|\frac{N_k}{B}-\mu_k\right|>r_{N,k}\,\middle|\,\mathcal H_k\right)
&\leq\frac\delta2,\\
\Prob\!\left(\left|\frac{C_k}{B}-\mu_{k-1}\right|>r_{C,k}\,\middle|\,\mathcal H_{k-1}\right)
&\leq\frac\delta2.
\end{align*}
Averaging over each respective history gives an unconditional failure probability at most $\delta/2$ for each event. A union bound then proves Eq.~\eqref{eq:app-count-tracking}.

On this joint count event, the difference between the two averages is bounded by their deviations from their own means and the gap between those means:
\begin{align*}
\frac{|N_k-C_k|}{B}
&\leq\left|\frac{N_k}{B}-\mu_k\right|
+|\mu_k-\mu_{k-1}|
+\left|\mu_{k-1}-\frac{C_k}{B}\right|\\
&\leq r_{N,k}+\Delta_{\mu,k}+r_{C,k}.
\end{align*}
Dividing by $C_k/B\geq1$ gives Eq.~\eqref{eq:app-relative-normalization}, which controls the normalization contribution to the gradient error.

\begin{samepage}
For the remaining cutoff contribution, condition on $(\mathcal H_k,\mathcal Y_k)$ so that the current complete rollouts and their local gradients are fixed. Equation~\eqref{eq:app-length-envelope} bounds the conditional second moment of $\widehat G_k-G_k$ by $H^2\sum_i v(T_{k,i})$. When this bound is positive, Markov's inequality at the squared-norm threshold $H^2\sum_i v(T_{k,i})/\zeta$ gives
\[
\Prob\!\left(\|\widehat G_k-G_k\|^2
>\frac{H^2\sum_i v(T_{k,i})}{\zeta}
\,\middle|\,\mathcal H_k,\mathcal Y_k\right)
\leq\frac{\E[\|\widehat G_k-G_k\|^2\mid\mathcal H_k,\mathcal Y_k]}
{H^2\sum_i v(T_{k,i})/\zeta}
\leq\zeta.
\]
Taking square roots on the complementary event yields, with conditional probability at least $1-\zeta$,
\[
\|\widehat G_k-G_k\|
\leq H\sqrt{\frac{\sum_i v(T_{k,i})}{\zeta}}.
\]
This also holds when the second-moment bound is zero, since the error is then zero almost surely.
\end{samepage}

\begin{samepage}
Averaging this conditional failure probability and intersecting with the count event gives probability at least $1-\delta-\zeta$. On this intersection, the normalized-error identity in Eq.~\eqref{eq:app-normalized-identity} combines the two contributions:
\begin{align*}
\|\widehat g_k-g_k^{\mathrm V}\|
&\leq\frac{\|\widehat G_k-G_k\|+|N_k-C_k|\,\|g_k^{\mathrm V}\|}{C_k}\\
&\leq\frac{H\sqrt{\sum_i v(T_{k,i})/\zeta}
+B(\Delta_{\mu,k}+r_{N,k}+r_{C,k})\|g_k^{\mathrm V}\|}{C_k}.
\end{align*}
Dividing the numerator and denominator by $B$ proves Eq.~\eqref{eq:app-joint-gradient}.
\end{samepage}
\end{proof}

\subsection{Batch size and gradient error}
\label{app:batch-size}

Larger batches reduce the variation from random truncation and the relative sampling error in the token-count estimate. Since $C_k$ comes from the preceding batch, changes in mean rollout length can still introduce a scale difference. The following corollary separates these effects and shows when our gradient approaches Vanilla OPD's gradient.

\begin{sopdcorollary}[Large-batch local agreement]
\label{cor:app-agreement}
Keep $L,K,H$ fixed and increase the batch size $B$ in both the current batch and its preceding normalization batch. Assume the sampling conditions of Theorem~\ref{thm:app-length-tracking} and $\|h_{k,i,t}\|\leq H$ almost surely for the current gradients. The histories and conditional means may depend on $B$. Then
\begin{equation}
\widehat g_k-g_k^{\mathrm V}
=\left(\frac{\mu_k}{\mu_{k-1}}-1\right)g_k^{\mathrm V}
+O_{\Prob}(B^{-1/2}).
\label{eq:app-local-agreement}
\end{equation}
Here the stochastic order concerns the Euclidean norm over the joint training and sampling randomness, with constants depending only on the fixed bounds. If $\mu_k/\mu_{k-1}\to1$ in probability, the local discrepancy converges to zero in probability. When the adjacent conditional mean lengths agree, its order is $B^{-1/2}$.
\end{sopdcorollary}

\begin{proof}
We first bound the variation in the current token count and the preceding batch's token-count estimate. For fixed $L,K$, each rollout's complete and weighted counts are bounded independently of $B$, so their conditional variances are uniformly bounded over training histories. At any fixed failure probability, Theorem~\ref{thm:app-length-tracking} then gives
\[
N_k/B=\mu_k+O_{\Prob}(B^{-1/2}),\qquad
C_k/B=\mu_{k-1}+O_{\Prob}(B^{-1/2}).
\]
Thus $N_k/B$ estimates the current mean length $\mu_k$, while $C_k/B$ estimates the preceding mean length $\mu_{k-1}$. To compare the normalization ratio with the ratio of these means, write
\[
\frac{N_k}{C_k}-\frac{\mu_k}{\mu_{k-1}}
=\frac{N_k/B-\mu_k}{C_k/B}
+\frac{\mu_k(\mu_{k-1}-C_k/B)}{(C_k/B)\mu_{k-1}}.
\]
Since $C_k/B\geq1$ and $1\leq\mu_j\leq L$, the two terms on the right are $O_{\Prob}(B^{-1/2})$.

We next bound the variation from random truncation. Equation~\eqref{eq:app-length-envelope} gives
\[
\E_{S_k}\left\|\frac{\widehat G_k-G_k}{C_k}\right\|^2
\leq\frac{H^2\sum_i v(T_{k,i})}{C_k^2}
\leq\frac{H^2v(L)}{B},
\]
where $C_k\geq B$ and $v(T_{k,i})\leq v(L)$; the latter follows because each tail-length term in Eq.~\eqref{eq:app-v} is nondecreasing in the complete length. Averaging this conditional bound and applying Markov's inequality gives, for every $a>0$,
\[
\Prob\!\left(\left\|\frac{\widehat G_k-G_k}{C_k}\right\|>
\frac{a}{\sqrt B}\right)\leq\frac{H^2v(L)}{a^2}.
\]
The truncation term is therefore $O_{\Prob}(B^{-1/2})$. Combining it with the normalization error using Eq.~\eqref{eq:app-normalized-identity} gives
\[
\widehat g_k-g_k^{\mathrm V}
-\left(\frac{\mu_k}{\mu_{k-1}}-1\right)g_k^{\mathrm V}
=\frac{\widehat G_k-G_k}{C_k}
+\left(\frac{N_k}{C_k}-\frac{\mu_k}{\mu_{k-1}}\right)g_k^{\mathrm V}.
\]
Both terms on the right have order $B^{-1/2}$ in probability, using $\|g_k^{\mathrm V}\|\leq H$. This proves Eq.~\eqref{eq:app-local-agreement} and its stated consequences.
\end{proof}

We can also bound the normalization error for a finite batch. The estimate $C_k$ first approximates the preceding count $N_{k-1}$, which can differ from the current count $N_k$. To bound the first difference, fix $k>W+1$ and $\delta\in(0,1)$, and condition on $(\mathcal H_{k-1},\mathcal Y_{k-1})$. The preceding rollout counts $\widehat T_{k-1,i}$ are independent over the sampled stopping positions, with means $T_{k-1,i}$, variances $v(T_{k-1,i})$, and range width at most $U_{L,K}-1$. Applying Lemma~\ref{lem:app-concentration} with average variance $B^{-1}\sum_i v(T_{k-1,i})$ and $x=\log(2/\delta)$, then multiplying by $B$, gives
\begin{equation}
|C_k-N_{k-1}|\leq q_{k-1,\delta},\qquad
q_{k-1,\delta}:=
\sqrt{2\log(2/\delta)\sum_i v(T_{k-1,i})}
+\frac{U_{L,K}-1}{3}\log\frac2\delta
\label{eq:app-cutoff-count}
\end{equation}
with conditional probability at least $1-\delta$. On this event,
\begin{equation}
|C_k-N_k|\leq q_{k-1,\delta}+|N_k-N_{k-1}|,
\label{eq:app-realized-drift}
\end{equation}
by the triangle inequality. The first term bounds the error in estimating the preceding count; the second is the difference between the two complete counts. Since the probability bound holds for each fixed training history, it also holds after averaging over histories. At the first step with random truncation, $C_{W+1}=N_W$, so the preceding count is known exactly and its estimation error is zero.

\begin{samepage}
At a fixed student policy, we can also average over the rollouts and teacher feedback. Let $T$ be the complete rollout length and $h_t$ its token gradients, with $\|h_t\|\leq H$. Here, $\E$ averages over the rollout, teacher feedback, and independent stopping position. The unbiased gradient-sum and token-count estimates in Eq.~\eqref{eq:app-unbiased} give
\begin{equation}
g_\star:=\frac{\E[\sum_{t=1}^T h_t]}{\E T}
=\frac{\E[\sum_{t=1}^T I_t h_t/p_t]}
{\E[\sum_{t=1}^T I_t/p_t]}.
\label{eq:app-population-objects}
\end{equation}
Both methods therefore have the same population token-mean gradient, defined as the expected gradient sum divided by the expected token count.
\end{samepage}

\section{Optimizer Update Error}
\label{app:optimizers}

At a fixed training step, we compare the parameter updates of \method{}
and Vanilla under the same optimizer. Their gradients are
$\widehat g=\widehat G/C$ and $g^{\mathrm V}=G/N$, respectively. Both updates
start from the same parameters and previous optimizer state. We use the
gradient-error decomposition from Theorem~\ref{thm:app-main} to bound the
update error of \method{} relative to Vanilla for SGD, Adam, and Muon.

Condition on the training history $\mathcal H_k$ and hypothetical complete
batch $\mathcal Y_k$, so only the current cutoffs remain random. For a
fixed step size $\eta>0$, let $\Phi(x)$ be the parameter increment produced
from gradient input $x$ in real arithmetic. A common decoupled weight-decay
increment cancels because the two updates start from the same parameters. With
$\rho=(N-C)/C$, Theorem~\ref{thm:app-main} gives
\[
\E_S\|\widehat g-g^{\mathrm V}\|^2
=\frac{V}{C^2}+\rho^2\|g^{\mathrm V}\|^2.
\]
The two terms measure cutoff noise and the mismatch between the lagged
denominator and the current complete token count.

The global norm clipping map $\mathcal Q$ fits the same comparison. For a threshold
$a_{\rm clip}>0$ and denominator stabilizer $\delta_{\rm clip}\geq0$,
define the map on the full gradient vector by
\begin{equation}
\mathcal Q(x)=\frac{x}{\max\{1,(\|x\|+\delta_{\rm clip})/a_{\rm clip}\}}.
\label{eq:app-global-clip}
\end{equation}
This map is nonexpansive.

For an optimizer below, evaluate its coefficient at the clipped Vanilla
reference $\mathcal Q(g^{\mathrm V})$ and denote the result by
$\Lambda_{\rm opt}$. For disjoint parameter blocks, use the largest block
coefficient; the full squared norm is the sum of the blockwise squared
norms. The resulting update satisfies
\begin{equation}
\begin{split}
\|\Phi(\mathcal Q(\widehat g))-\Phi(\mathcal Q(g^{\mathrm V}))\|
&\leq\Lambda_{\rm opt}\|\mathcal Q(\widehat g)-\mathcal Q(g^{\mathrm V})\|\\
&\leq\Lambda_{\rm opt}\|\widehat g-g^{\mathrm V}\|.
\end{split}
\label{eq:app-clipped-update}
\end{equation}
The reference and coefficients remain fixed under the conditioning on
$(\mathcal H_k,\mathcal Y_k)$.
Squaring Eq.~\eqref{eq:app-clipped-update}, taking conditional expectations,
and applying Eq.~\eqref{eq:app-mse} bounds the mean-square update error by
$\Lambda_{\rm opt}^2\{V/C^2+\rho^2\|g^{\mathrm V}\|^2\}$.

To verify nonexpansiveness, consider the Jacobian of $\mathcal Q$.
It is the identity in the interior of the unscaled region. In the interior of the scaled region, with
$s=\|x\|>0$, its tangential and radial
eigenvalues are $a_{\rm clip}/(s+\delta_{\rm clip})$ and
$a_{\rm clip}\delta_{\rm clip}/(s+\delta_{\rm clip})^2$, respectively.
Both lie in $[0,1]$ because $s+\delta_{\rm clip}\geq a_{\rm clip}$.
The map is continuous at zero and at the branch boundary. Integrating
its derivative along the smooth pieces of a line segment therefore gives
$\|\mathcal Q(x)-\mathcal Q(y)\|\leq\|x-y\|$.

\subsection{Stochastic gradient descent}
\label{app:sgd}

For SGD, $\Phi(x)=-\eta x$, so the update error of \method{} relative to
Vanilla is the gradient error multiplied by $\Lambda_{\mathrm{SGD}}=\eta$:
\begin{equation}
\|\Phi(\widehat g)-\Phi(g^{\mathrm V})\|
=\eta\|\widehat g-g^{\mathrm V}\|.
\label{eq:app-sgd-pointwise}
\end{equation}
The same equality holds for any input $x$ in place of $\widehat g$.
Taking the expected squared difference over the current cutoffs gives
\begin{equation}
\E_S\|\Phi(\widehat g)-\Phi(g^{\mathrm V})\|^2
=\eta^2\left\{\frac{V}{C^2}
+\rho^2\|g^{\mathrm V}\|^2\right\}.
\label{eq:app-sgd-error}
\end{equation}
Our SGD runs (Section~\ref{sec:optimizers}) use neither momentum nor clipping, so their coefficient is
$\Lambda_{\mathrm{SGD}}=\eta=5\times10^{-4}$ at every step.

The momentum update is also affine in the current gradient. With fixed
previous momentum $\boldsymbol m$, its update is
$\boldsymbol m^+(x)=\beta\boldsymbol m+(1-\beta)x$ and $\Phi(x)=-\eta\boldsymbol m^+(x)$.
The previous-momentum term cancels between the two inputs, leaving
$\Phi(x)-\Phi(g^{\mathrm V})=-\eta(1-\beta)(x-g^{\mathrm V})$.
Its coefficient is therefore $\eta|1-\beta|$, and the squared-error
equality above has multiplier $\eta^2(1-\beta)^2$.
For any affine update,
$\E_S\Phi(\widehat g)=\Phi(G/C)$ because
$\E_S\widehat g=G/C$. The Adam and Muon maps below are nonlinear in the
current input; their bounds account for the full input dependence.

\subsection{Adam}
\label{app:adam}

Adam incorporates the current gradient into both its first moment and the
second moment in its denominator. We account for this coupled dependence
when bounding the update error of \method{} relative to Vanilla.

At step $k\geq1$, let $\boldsymbol m$ and $\boldsymbol q\geq0$ be Adam's
previous first and second moments. Fix $\eta>0$,
$\beta_1,\beta_2\in[0,1)$, and a stabilization parameter $\epsilon_{\mathrm A}>0$. For a gradient
input $x\in\mathbb R^d$, Adam's parameter increment
\citep[Algorithm~1]{kingma2015adam} is
\begin{equation}
\Phi_{\mathrm A}(x)=-\frac{\eta}{1-\beta_1^k}
\frac{\beta_1\boldsymbol m+(1-\beta_1)x}
{\sqrt{\{\beta_2\boldsymbol q+(1-\beta_2)x^{\odot2}\}/(1-\beta_2^k)}+\epsilon_{\mathrm A}},
\label{eq:app-adam-map}
\end{equation}
where the square, square root, and division are coordinatewise. The factors
$1-\beta_1^k$ and $1-\beta_2^k$ provide Adam's moment bias correction. Each
output coordinate depends on only one input coordinate, but that input
changes both the numerator and denominator.

We evaluate this dependence through the scalar function
\begin{equation}
F(z)=\frac{a+bz}{\sqrt{c+\nu z^2}+\epsilon},
\qquad
\mathfrak L(a,b,c,\nu,\epsilon;y)
=\sup_{z\ne y}\left|\frac{F(z)-F(y)}{z-y}\right|,
\label{eq:app-adam-scalar-coefficient}
\end{equation}
with $a,y\in\mathbb R$, $b,\nu,\epsilon>0$, and $c\geq0$. The reference
$y$ is fixed. The coefficient $\mathfrak L$ is the largest absolute
change in $F$ per unit change from $y$ to any other input $z$.
For Adam's coordinate $j$, use
\[
a=\beta_1m_j,\quad b=1-\beta_1,\quad
c=c_j:=\frac{\beta_2q_j}{1-\beta_2^k},\quad
\nu=\frac{1-\beta_2}{1-\beta_2^k},\quad
\epsilon=\epsilon_{\mathrm A},\quad y=g_j^{\mathrm V}.
\]
Here $a$ and $c$ contain the previous moments; $b$ and $\nu$ weight the
current gradient and its square. These values satisfy the scalar domain
conditions and give
\[
\Phi_{\mathrm A}(x)_j=-\frac{\eta}{1-\beta_1^k}F(x_j).
\]
The scalar coefficient can be computed exactly from a closed formula or
a finite set of polynomial roots, as established in the proof.

\begin{sopdproposition}[Update error with Adam]
\label{prop:app-adam}
For the current batch, let $\widehat g$ and $g^{\mathrm V}$ denote the
gradients of \method{} and Vanilla, respectively. The Adam updates
in Eq.~\eqref{eq:app-adam-map}, from the same parameter vector and previous
moments $\boldsymbol m,\boldsymbol q$, satisfy
\begin{equation}
\boxed{\|\Phi_{\mathrm A}(\widehat g)-\Phi_{\mathrm A}(g^{\mathrm V})\|
\leq\Lambda_{\mathrm A}\|\widehat g-g^{\mathrm V}\|}.
\label{eq:app-adam-error}
\end{equation}
Here the coordinate coefficients and their maximum are
\begin{equation}
\lambda_{\mathrm A,j}=\frac{\eta}{1-\beta_1^k}
\mathfrak L(\beta_1m_j,1-\beta_1,c_j,\nu,\epsilon_{\mathrm A};g_j^{\mathrm V}),
\qquad \Lambda_{\mathrm A}=\max_{1\leq j\leq d}\lambda_{\mathrm A,j}.
\label{eq:app-adam-coefficient}
\end{equation}
These coefficients are finite and fixed conditional on the training
history $\mathcal H_k$ and complete batch $\mathcal Y_k$.
The same pointwise bound holds with $\widehat g$ replaced by any gradient
input $x\in\mathbb R^d$.

To state the conditional mean-square bound, let
$V_j=\operatorname{Var}_S(\widehat G_j)$ and
$\rho=(N-C)/C$. The coordinate form of Eq.~\eqref{eq:app-mse} is
\begin{equation}
\mathcal E_j:=\E_S(\widehat g_j-g_j^{\mathrm V})^2
=\frac{V_j}{C^2}+\rho^2(g_j^{\mathrm V})^2.
\label{eq:app-coordinate-error}
\end{equation}
The variance $V_j$ is obtained from
Eq.~\eqref{eq:app-exact-variance} by replacing each token gradient with
its $j$th component, and $\sum_jV_j=V$.
Averaging over the current cutoffs then gives
\begin{equation}
\E_S\|\Phi_{\mathrm A}(\widehat g)-\Phi_{\mathrm A}(g^{\mathrm V})\|^2
\leq\sum_j\lambda_{\mathrm A,j}^2\mathcal E_j
\leq\Lambda_{\mathrm A}^2
\left\{\frac{V}{C^2}+\rho^2\|g^{\mathrm V}\|^2\right\}.
\label{eq:app-adam-variance}
\end{equation}
\end{sopdproposition}

For the globally clipped update in Eq.~\eqref{eq:app-clipped-update},
evaluate the coefficient in Eq.~\eqref{eq:app-adam-coefficient} using
$y=\mathcal Q(g^{\mathrm V})_j$ in each coordinate.

\begin{proof}
We first compute the largest absolute secant from the fixed reference $y$
of $F(z)=(a+bz)/(\sqrt{c+\nu z^2}+\epsilon)$, namely
$\mathfrak L(a,b,c,\nu,\epsilon;y)$. Only $z$ varies in this calculation;
the scalar parameters and reference are fixed.
The value of $c$ determines the form of the denominator: for $c>0$ the
square root stays positive, whereas for $c=0$ it becomes
$\sqrt\nu\,|z|$. We compute the secant maximum in these two cases
before inserting the Adam state and combining coordinates.

Suppose $c>0$. To maximize the secant over the whole real line, introduce
a coordinate that both bounds the search interval and makes the square
root rational:
\[
t=\frac{\sqrt\nu\,z}{\sqrt c+\sqrt{c+\nu z^2}},
\qquad
t_y=\frac{\sqrt\nu\,y}{\sqrt c+\sqrt{c+\nu y^2}}.
\]
The first transformation is a bijection from $\mathbb R$ to $(-1,1)$;
$t_y$ is the image of the fixed reference. Its inverse and the resulting
square root are
\[
z(t)=\frac{2\sqrt c\,t}{\sqrt\nu(1-t^2)},
\qquad
\sqrt{c+\nu z(t)^2}=\frac{\sqrt c(1+t^2)}{1-t^2}.
\]
Substitution gives rational expressions for both quantities needed in
the secant:
\begin{align*}
F(z(t))&=\frac{a(1-t^2)+2b\sqrt c\,t/\sqrt\nu}
{\sqrt c(1+t^2)+\epsilon(1-t^2)},\\
z(t)-y&=\frac{2\sqrt c(t-t_y)(1+tt_y)}
{\sqrt\nu(1-t^2)(1-t_y^2)}.
\end{align*}
Subtract $F(y)=F(z(t_y))$ in the first expression and divide by the
second. Cancelling the common factor $t-t_y$ gives the exact secant
\begin{equation}
\frac{F(z)-F(y)}{z-y}=\frac{\mathsf n(t)}{\mathsf d(t)},
\label{eq:app-adam-exact-secant}
\end{equation}
where the numerator and denominator polynomials are
\begin{align*}
\mathsf n(t)
&=(1-t^2)(1-t_y^2)
\left\{b\bigl[\sqrt c(1-tt_y)+\epsilon(1+tt_y)\bigr]
-a\sqrt\nu(t+t_y)\right\},\\
\mathsf d(t)
&=(1+tt_y)
\bigl[\sqrt c(1+t^2)+\epsilon(1-t^2)\bigr]
\bigl[\sqrt c(1+t_y^2)+\epsilon(1-t_y^2)\bigr].
\end{align*}

The denominator has no zero on $[-1,1]$:
$1+tt_y\geq1-|t_y|>0$, and each bracket in $\mathsf d$ is positive.
Consequently $\mathsf n/\mathsf d$ is continuous on the closed interval.
At the omitted reference point $t=t_y$, its value is the limiting secant
$F'(y)$. At $t=\pm1$, which correspond to $z\to\pm\infty$, its value is
zero. Including the reference and endpoint values therefore leaves the
supremum unchanged and makes it a maximum on $[-1,1]$.

The limits of $F$ at $\pm\infty$ are $\pm b/\sqrt\nu$, so the
secant quotient is not identically zero. Its positive absolute maximum
therefore occurs inside $(-1,1)$, where the quotient is nonzero and its
derivative vanishes. The stationary equation is
$\mathsf n'\mathsf d-\mathsf n\mathsf d'=0$.
Since both polynomials have degree at most three, their degree-five terms
cancel in this equation, leaving degree at most four. The stationary
polynomial is nonzero because the quotient is nonconstant. Its roots in
$(-1,1)$ therefore form a finite nonempty set that contains a maximizer:
\begin{equation}
\begin{split}
\mathcal Z_c&=\{t\in(-1,1):
\mathsf n'(t)\mathsf d(t)-\mathsf n(t)\mathsf d'(t)=0\},\\
\mathfrak L(a,b,c,\nu,\epsilon;y)
&=\max_{t\in\mathcal Z_c}\frac{|\mathsf n(t)|}{\mathsf d(t)}.
\end{split}
\label{eq:app-adam-quartic}
\end{equation}
To evaluate this coefficient, solve the displayed stationary polynomial,
retain its real roots in $(-1,1)$, and compare $|\mathsf n|/\mathsf d$ at
those roots. A root equal to $t_y$ is retained: cancellation has already
removed the apparent division by zero in the original secant.

It remains to compute the scalar coefficient when $c=0$. The absolute
value in $F(z)=(a+bz)/(\sqrt\nu\,|z|+\epsilon)$ makes it useful to
consider the two half-lines separately. If the reference is $y=0$, then
\[
\frac{F(z)-F(0)}{z}
=\frac{b\epsilon-a\sqrt\nu\,\operatorname{sgn}(z)}
{\epsilon(\sqrt\nu\,|z|+\epsilon)},\qquad z\ne0.
\]
On each half-line the numerator is constant and the positive denominator
increases with $|z|$. The supremum on that half-line is therefore the
limiting absolute slope at zero. Taking the larger of the two limits
gives
\begin{equation}
\mathfrak L(a,b,0,\nu,\epsilon;0)
=\frac{b\epsilon+|a|\sqrt\nu}{\epsilon^2}.
\label{eq:app-adam-zero-reference}
\end{equation}

Now suppose $c=0$ and $y\ne0$. On the same side of zero as the reference,
the secant reduces to
\[
\frac{F(z)-F(y)}{z-y}
=\frac{b\epsilon-a\sqrt\nu\,\operatorname{sgn}(y)}
{(\sqrt\nu\,|z|+\epsilon)(\sqrt\nu\,|y|+\epsilon)},
\qquad zy\geq0,\quad z\ne y.
\]
The absolute value is largest at $z=0$, since the numerator is constant
and the denominator increases with $|z|$.
For inputs on the opposite side, write
$z=-\operatorname{sgn}(y)u$, where $u\geq0$ is the input's distance from
zero. Subtracting $F(y)$ and dividing by
$z-y=-\operatorname{sgn}(y)(u+|y|)$ gives the rational secant
$\mathsf n_0(u)/\mathsf d_0(u)$ with
\begin{align*}
\mathsf n_0(u)
&=b\bigl\{\epsilon(u+|y|)+2\sqrt\nu\,|y|u\bigr\}
+a\sqrt\nu\,\operatorname{sgn}(y)(u-|y|),\\
\mathsf d_0(u)
&=(u+|y|)(\sqrt\nu\,u+\epsilon)(\sqrt\nu\,|y|+\epsilon).
\end{align*}
Its value at $u=0$ is the secant to $z=0$, which already captures the
maximum on the reference's half-line. It therefore suffices to maximize
$|\mathsf n_0|/\mathsf d_0$ over $u\geq0$.

The denominator $\mathsf d_0$ is positive on this domain, so the quotient
is continuous. The quotient tends to zero as $u\to\infty$ and is not
identically zero: if $\mathsf n_0(0)=0$, the coefficient of $u$ in
$\mathsf n_0$ is $2b(\epsilon+\sqrt\nu\,|y|)>0$.
Thus its absolute value has a positive maximum on a finite interval and
hence on $[0,\infty)$. The maximum occurs at $u=0$ or at a positive
stationary point. Since $\mathsf n_0$ and $\mathsf d_0$ have degrees at
most one and two, respectively, the stationary polynomial has degree at
most two. It is nonzero because the quotient is nonconstant. Consequently
\begin{equation}
\begin{split}
\mathcal Z_0&=\{0\}\cup\{u>0:
\mathsf n_0'(u)\mathsf d_0(u)-\mathsf n_0(u)\mathsf d_0'(u)=0\},\\
\mathfrak L(a,b,0,\nu,\epsilon;y)
&=\max_{u\in\mathcal Z_0}\frac{|\mathsf n_0(u)|}{\mathsf d_0(u)}.
\end{split}
\label{eq:app-adam-quadratic}
\end{equation}
For this case, solve the displayed polynomial, retain its positive real
roots, and compare the absolute quotient at those roots and at $u=0$.
Together with Eqs.~\eqref{eq:app-adam-quartic}
and~\eqref{eq:app-adam-zero-reference}, this gives a finite value of
$\mathfrak L$ for every admissible scalar parameter choice.

Return to Adam's coordinate $j$, whose scalar reference is
$y=g_j^{\mathrm V}$ and previous second-moment contribution is
$c=c_j=\beta_2q_j/(1-\beta_2^k)$. Insert these and the remaining
state parameters from the substitution following
Eq.~\eqref{eq:app-adam-scalar-coefficient} into
the applicable case of Eqs.~\eqref{eq:app-adam-quartic}--\eqref{eq:app-adam-quadratic}.
Scaling its finite value by $\eta/(1-\beta_1^k)$ gives
$\lambda_{\mathrm A,j}$; the largest coordinate value gives
$\Lambda_{\mathrm A}$. Both are determined by the fixed state,
hyperparameters, and reference gradient.

For any input $x$, the scalar secant bound and the identity
$\Phi_{\mathrm A}(x)_j=-\eta F(x_j)/(1-\beta_1^k)$ now give
\[
|\Phi_{\mathrm A}(x)_j-\Phi_{\mathrm A}(g^{\mathrm V})_j|
\leq\lambda_{\mathrm A,j}|x_j-g_j^{\mathrm V}|.
\]
We examine the size of Adam's coefficient bounds during training
(Figure~\ref{fig:adam-lambda}) and find that they remain moderate,
so small gradient errors also give small update-error bounds.

This also holds when $x_j=g_j^{\mathrm V}$, since both sides are zero.
Squaring each coordinate inequality and summing yields
\[
\|\Phi_{\mathrm A}(x)-\Phi_{\mathrm A}(g^{\mathrm V})\|^2
\leq\sum_j\lambda_{\mathrm A,j}^2(x_j-g_j^{\mathrm V})^2
\leq\Lambda_{\mathrm A}^2\|x-g^{\mathrm V}\|^2.
\]
Taking square roots gives the bound for every input $x$; setting
$x=\widehat g$, the gradient of \method{}, proves
Eq.~\eqref{eq:app-adam-error}. For the mean-square bound, take conditional
expectations in the coordinate-weighted inequality at this input.
The coefficients are fixed under the
conditioning on $(\mathcal H_k,\mathcal Y_k)$, so its right-hand side
becomes $\sum_j\lambda_{\mathrm A,j}^2\mathcal E_j$.
Equation~\eqref{eq:app-coordinate-error} and
$\sum_jV_j=V$ give
\[
\sum_j\mathcal E_j
=\frac{V}{C^2}+\rho^2\|g^{\mathrm V}\|^2.
\]
Bounding each $\lambda_{\mathrm A,j}^2$ by $\Lambda_{\mathrm A}^2$
proves Eq.~\eqref{eq:app-adam-variance}.
\end{proof}

\begin{figure}[H]
\centering
\includegraphics[width=\textwidth]{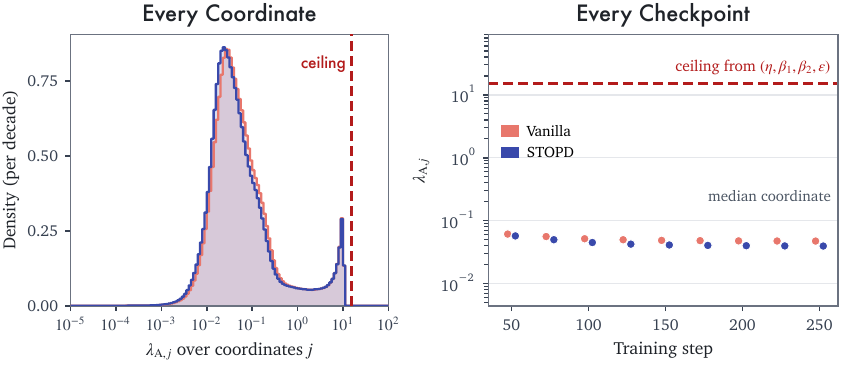}
\caption{Upper bounds on Adam's coefficients $\lambda_{\mathrm A,j}$: distribution across coordinates and checkpoints (left), the median coordinate at each step (right), and the theoretical ceiling (dashed).}
\label{fig:adam-lambda}
\end{figure}

\subsection{Muon}
\label{app:muon}

{For a matrix parameter block, let} $x\in\mathbb R^{d_{\rm out}\times d_{\rm in}}$ be its gradient
input. Fix the previous momentum $\boldsymbol M$, the momentum parameter
$\beta\in[0,1)$, and the stabilization parameter $\epsilon_{\mathrm M}>0$.
We use Frobenius norms except for the operator norm $\|\cdot\|_{\rm op}$,
and restrict $g^{\mathrm V}$, $\widehat g$, and $V$ to this block.
Muon forms the new momentum
$\boldsymbol M^+(x)=\beta\boldsymbol M+(1-\beta)x$ and then the Nesterov input
\[
\boldsymbol R(x)=(1-\beta)x+\beta\boldsymbol M^+(x)
=\beta^2\boldsymbol M+(1-\beta^2)x.
\]
We normalize the Nesterov input by clamping its Frobenius norm below at
$\epsilon_{\mathrm M}$:
\[
\boldsymbol X_0(x)=
\frac{\boldsymbol R(x)}{\max\{\|\boldsymbol R(x)\|_F,\epsilon_{\mathrm M}\}}.
\]
Muon then applies $J\geq0$ polynomial iterations, with $J$ an integer
\citep{jordan2024muon}:
\begin{equation}
\boldsymbol X_{j+1}=\omega_1\boldsymbol X_j+\omega_3\boldsymbol X_j\boldsymbol X_j^\top\boldsymbol X_j
+\omega_5(\boldsymbol X_j\boldsymbol X_j^\top)^2\boldsymbol X_j.
\label{eq:app-muon-iteration}
\end{equation}
Here $\omega_1,\omega_3,\omega_5\in\mathbb R$ are the optimizer's fixed
polynomial coefficients.
Let $\mathcal P_J$ denote the composition of these matrix iterations,
with $\mathcal P_0(X)=X$. Let $\gamma_{\rm shape}>0$ be the optimizer's
fixed scaling factor for this block, determined by its original dimensions
and the chosen scaling convention. Its parameter increment is
\[
\Phi_{\mathrm M}(x)=-\eta\gamma_{\rm shape}\mathcal P_J(\boldsymbol X_0(x)).
\]

The update comparison uses the reference Nesterov matrix and its normalization,
\[
\boldsymbol B=\boldsymbol R(g^{\mathrm V}),\qquad
b=\|\boldsymbol B\|_F,\qquad
Y=\boldsymbol X_0(g^{\mathrm V})=\frac{\boldsymbol B}{\max\{b,\epsilon_{\mathrm M}\}}.
\]
The effect of normalization on the distance from this reference is
controlled by
\begin{equation}
\kappa_b=\frac{2}{\max\{b,\epsilon_{\mathrm M}\}+\epsilon_{\mathrm M}}.
\label{eq:app-muon-radial-coefficient}
\end{equation}

For the polynomial iterations, write the scalar one-step polynomial as
$\varphi(s)=\omega_1s+\omega_3s^3+\omega_5s^5$ and its compositions as
$f_0(s)=s$, $f_{j+1}=\varphi\circ f_j$.
Their scalar sensitivity is the divided difference
\begin{equation}
\mathsf D_J(s,t)=
\begin{cases}
\dfrac{f_J(s)-f_J(t)}{s-t},&s\ne t,\\[1mm]
f_J'(t),&s=t.
\end{cases}
\label{eq:app-muon-response}
\end{equation}
Since $f_J$ is a polynomial, $\mathsf D_J$ is a polynomial in $(s,t)$
including on the diagonal, and $\mathsf D_0=1$.
To relate these scalar polynomials to the rectangular matrix iterations
\citep{andersson2016operator}, use the symmetric dilation
\[
\mathcal D(Z)=\begin{pmatrix}0&Z\\Z^\top&0\end{pmatrix},
\qquad \mathcal T_Y=\operatorname{spec}(\mathcal D(Y)).
\]
The dilation eigenvalues are the signed nonzero singular values,
together with the remaining zeros. Since normalized inputs have
Frobenius norm at most one, their dilation spectra lie in $[-1,1]$.
We combine $\kappa_b$ with divided differences over this interval and
the fixed reference spectrum $\mathcal T_Y$.

\begin{sopdproposition}[Update error with Muon]
\label{prop:app-muon}
For the current batch, let $\widehat g$ and $g^{\mathrm V}$ denote the
gradients of \method{} and Vanilla, respectively, restricted to a matrix parameter block. The
Muon updates defined by $\Phi_{\mathrm M}$ above, from the same parameter
values and previous momentum $\boldsymbol M$, satisfy
\begin{equation}
\boxed{\|\Phi_{\mathrm M}(\widehat g)-\Phi_{\mathrm M}(g^{\mathrm V})\|_F
\leq\Lambda_{\mathrm M}\|\widehat g-g^{\mathrm V}\|_F}.
\label{eq:app-muon-error}
\end{equation}
Here the coefficient is
\begin{equation}
\Lambda_{\mathrm M}=\eta\gamma_{\rm shape}(1-\beta^2)\kappa_b
\max_{\substack{|s|\leq1\\t\in\mathcal T_Y}}|\mathsf D_J(s,t)|.
\label{eq:app-muon-coefficient}
\end{equation}
This coefficient is finite and fixed conditional on the training history
and complete batch. The same pointwise bound holds with $\widehat g$
replaced by any matrix gradient input $x$ of the same shape.
Averaging over the current cutoffs of \method{} bounds the conditional
mean-square update error:
\begin{equation}
\E_S\|\Phi_{\mathrm M}(\widehat g)-\Phi_{\mathrm M}(g^{\mathrm V})\|_F^2
\leq\Lambda_{\mathrm M}^2
\left\{\frac{V}{C^2}+\rho^2\|g^{\mathrm V}\|_F^2\right\}.
\label{eq:app-muon-variance}
\end{equation}
\end{sopdproposition}

For a block of the globally clipped update in
Eq.~\eqref{eq:app-clipped-update}, evaluate
Eq.~\eqref{eq:app-muon-coefficient} with $\boldsymbol B=\boldsymbol R(y)$,
where $y$ is the corresponding block of the globally clipped full
Vanilla gradient.

\begin{proof}
We track the gradient difference through the Nesterov input,
normalization, and polynomial iterations, then combine their factors
to obtain the update-error coefficient.
Fix a matrix gradient input $x$. The shared previous momentum cancels
between $\boldsymbol R(x)$ and $\boldsymbol B=\boldsymbol R(g^{\mathrm V})$:
\[
\boldsymbol R(x)-\boldsymbol B=(1-\beta^2)(x-g^{\mathrm V}).
\]
Write $r=\|\boldsymbol R(x)\|_F$ and $b=\|\boldsymbol B\|_F$ for their
Frobenius norms. Normalization maps them to $\boldsymbol X_0(x)$ and
$Y=\boldsymbol X_0(g^{\mathrm V})$, with respective Frobenius norms
$a=r/\max\{r,\epsilon_{\mathrm M}\}$ and
$c=b/\max\{b,\epsilon_{\mathrm M}\}$.
Suppose first that $r,b>0$ and $r\ne b$, and let
$\cos\vartheta=\langle\boldsymbol R(x),\boldsymbol B\rangle_F/(rb)$.
Expanding the squared normalization-distance ratio
$\|\boldsymbol X_0(x)-Y\|_F^2/\|\boldsymbol R(x)-\boldsymbol B\|_F^2$
gives
\[
\frac{a^2+c^2-2ac\cos\vartheta}
{r^2+b^2-2rb\cos\vartheta}.
\]
Its denominator is at least $(r-b)^2>0$. Both numerator and denominator
are affine in $\cos\vartheta\in[-1,1]$, so differentiation gives a
constant sign and the maximum occurs at an endpoint. For
$\cos\vartheta=-1$ and $\cos\vartheta=1$, the square roots of the ratio
are, respectively,
\[
\frac{a+c}{r+b}
\quad\text{and}\quad
\frac{|a-c|}{|r-b|}.
\]
The first is at least the second. Indeed, when $r>b$, we have
$a\geq c$ and $a/r\leq c/b$, hence
$(a-c)(r+b)\leq(a+c)(r-b)$. Interchanging the radii proves the same
comparison when $r<b$.

The cases excluded from this ratio calculation give the same factor.
When $r=b>0$, both matrices are scaled by
$1/\max\{r,\epsilon_{\mathrm M}\}=(a+c)/(r+b)$, including when the
matrices coincide. If exactly one radius is zero, the distance ratio is
again $(a+c)/(r+b)$. If both vanish, both distances vanish. Define
$q_b(r)=(a+c)/(r+b)$ when $r+b>0$, and $q_0(0)=1/\epsilon_{\mathrm M}$.
Then every input satisfies
\begin{equation}
\|\boldsymbol X_0(x)-Y\|_F
\leq q_b(r)\|\boldsymbol R(x)-\boldsymbol B\|_F
\leq\kappa_b\|\boldsymbol R(x)-\boldsymbol B\|_F.
\label{eq:app-muon-radial-exact}
\end{equation}
To obtain the second inequality, first take $b\leq\epsilon_{\mathrm M}$.
Each normalized radius is at most its original radius divided by
$\epsilon_{\mathrm M}$, so $q_b(r)\leq1/\epsilon_{\mathrm M}=\kappa_b$.
For $b>\epsilon_{\mathrm M}$,
\[
q_b(r)=
\begin{cases}
\dfrac{r/\epsilon_{\mathrm M}+1}{r+b},&0\leq r\leq\epsilon_{\mathrm M},\\[3mm]
\dfrac2{r+b},&r\geq\epsilon_{\mathrm M}.
\end{cases}
\]
The first expression has derivative
$(b/\epsilon_{\mathrm M}-1)/(r+b)^2>0$; the second decreases with $r$.
Their common maximum at $r=\epsilon_{\mathrm M}$ is
$2/(b+\epsilon_{\mathrm M})=\kappa_b$.
For every $b>0$, choosing
$\boldsymbol R(x)=-\epsilon_{\mathrm M}\boldsymbol B/b$ attains the
factor $\kappa_b$. When $b=0$, any input with
$\|\boldsymbol R(x)\|_F=\epsilon_{\mathrm M}$ attains it.
Thus $\kappa_b$ is the exact distance-amplification factor of normalization
about $\boldsymbol B$.

For a matrix $X$ of the same shape as the normalized reference
$Y=\boldsymbol X_0(g^{\mathrm V})$, the symmetric dilation $\mathcal D$
relates the scalar polynomial $\varphi$ and its compositions $f_J$
to the matrix iterations. Block multiplication and iteration give
\[
\varphi(\mathcal D(X))=\mathcal D(\mathcal P_1(X)),
\qquad f_J(\mathcal D(X))=\mathcal D(\mathcal P_J(X)).
\]
The second identity includes $J=0$, since both compositions are then
the identity map.

For $\|X\|_{\rm op}\leq1$, choose eigenvalue--eigenvector pairs
$(s_u,\xi_u)$ of $\mathcal D(X)$ and $(t_v,\chi_v)$ of $\mathcal D(Y)$,
with both eigenvector bases orthonormal. The dilation eigenvalues are
signed singular values (and zeros), so the input eigenvalues satisfy
$|s_u|\leq1$, and the reference eigenvalues lie in
$\mathcal T_Y=\operatorname{spec}(\mathcal D(Y))$.
The divided difference $\mathsf D_J$ of $f_J$ relates the corresponding
entries in these two bases:
\begin{align*}
\xi_u^\top\{f_J(\mathcal D(X))-f_J(\mathcal D(Y))\}\chi_v
&=\{f_J(s_u)-f_J(t_v)\}\xi_u^\top \chi_v\\
&=\mathsf D_J(s_u,t_v)\,\xi_u^\top\{\mathcal D(X)-\mathcal D(Y)\}\chi_v.
\end{align*}
When $s_u=t_v$, both sides of the second equality are zero.
The sum of squared entries in these orthonormal bases is the squared
Frobenius norm. Bound every divided difference by
$\max_{|s|\leq1,\,t\in\mathcal T_Y}|\mathsf D_J(s,t)|$, sum the
squares, and take square roots. The dilation identity above and
$\|\mathcal D(Z)\|_F=\sqrt2\|Z\|_F$ then give
\[
\|\mathcal P_J(X)-\mathcal P_J(Y)\|_F
\leq\max_{\substack{|s|\leq1\\t\in\mathcal T_Y}}|\mathsf D_J(s,t)|\,\|X-Y\|_F.
\]

For the actual normalized input $X=\boldsymbol X_0(x)$,
$\|X\|_{\rm op}\leq\|X\|_F\leq1$, so this polynomial bound applies.
Its coefficient is finite: $\mathcal T_Y$ is finite, and each
$\mathsf D_J(\cdot,t)$ is a polynomial on the compact interval $[-1,1]$.
To compute $\Lambda_{\mathrm M}$ from the fixed state, form $\boldsymbol B$,
$b$, $Y$, and $f_J$. For each reference eigenvalue $t\in\mathcal T_Y$,
maximize $|\mathsf D_J(s,t)|$ over $s\in[-1,1]$ by evaluating the
endpoints and the stationary points of $\mathsf D_J(\cdot,t)$; a constant
polynomial needs only an endpoint. Take the largest value over $t$
and multiply by the displayed factors.

The parameter increments are $-\eta\gamma_{\rm shape}$ times the
polynomial outputs, so the two bounds give
\begin{align*}
\|\Phi_{\mathrm M}(x)-\Phi_{\mathrm M}(g^{\mathrm V})\|_F
&\leq\eta\gamma_{\rm shape}\kappa_b
\max_{\substack{|s|\leq1\\t\in\mathcal T_Y}}|\mathsf D_J(s,t)|
\,\|\boldsymbol R(x)-\boldsymbol B\|_F\\
&=\Lambda_{\mathrm M}\|x-g^{\mathrm V}\|_F.
\end{align*}
The last step uses the Nesterov-input difference established at the
start of the proof and the maximum defining $\Lambda_{\mathrm M}$.
This gives the bound for every matrix input $x$; setting $x=\widehat g$,
the gradient of \method{}, proves Eq.~\eqref{eq:app-muon-error}.

The coefficient is fixed under conditioning on history and the complete
batch. Squaring at $x=\widehat g$ and averaging over the current cutoffs gives
\[
\E_S\|\Phi_{\mathrm M}(\widehat g)-\Phi_{\mathrm M}(g^{\mathrm V})\|_F^2
\leq\Lambda_{\mathrm M}^2\E_S\|\widehat g-g^{\mathrm V}\|_F^2.
\]
Substituting the conditional gradient-error decomposition
Eq.~\eqref{eq:app-mse} for this block proves
Eq.~\eqref{eq:app-muon-variance}.
\end{proof}

We examine the size of Muon's coefficient bounds during training
(Figure~\ref{fig:muon-lambda}) and find that they remain moderate,
so small gradient errors also give small update-error bounds.
\begin{figure}[H]
\centering
\includegraphics[width=\textwidth]{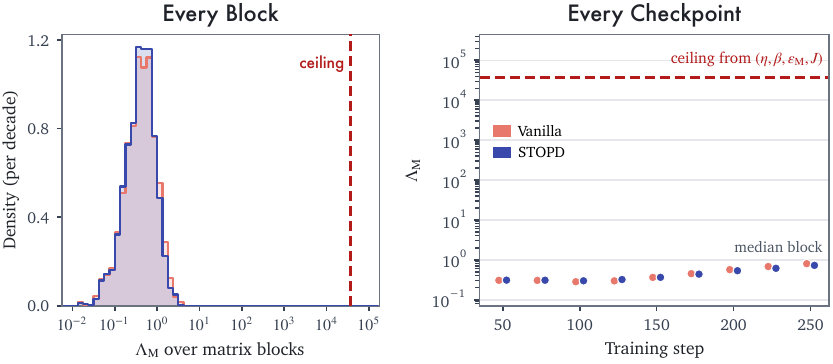}
\caption{Muon's coefficients $\Lambda_{\mathrm M}$: distribution across blocks and checkpoints (left), the median block (right), and the theoretical ceiling (dashed).}
\label{fig:muon-lambda}
\end{figure}

\section{Token and Time Savings}
\label{app:savings}

We first derive the expected token savings, then analyze how these savings translate into reduced generation time.

\subsection{Token savings}
\label{app:token-savings}

For a complete rollout of length $T$, random truncation skips token $t$ with probability $1-p_t$. The expected number of tokens saved is therefore
\[
D(T)=\sum_{t=1}^{T}(1-p_t).
\]
Tokens within the same block have the same inclusion probability. Position $t$ belongs to block $b=\lfloor(t-1)/K\rfloor$, where $b=0$ denotes the first block. Under the stopping rule in Eq.~\eqref{eq:stop}, its inclusion probability is
\begin{equation}
p_t=\frac{L-K\lfloor(t-1)/K\rfloor}{L}=1-\frac{bK}{L}>0.
\label{eq:app-survival}
\end{equation}
Every token in block $b$ is therefore skipped with probability $bK/L$. Write $T=mK+r$, where $m$ is the number of complete blocks and $0\leq r<K$ is the number of tokens in the next block. Each complete block contributes $K\cdot bK/L$ expected saved tokens, while the remaining $r$ tokens contribute $r\cdot mK/L$. Adding these contributions gives
\begin{equation}
D(T)=\frac KL\left(K\sum_{b=0}^{m-1}b+mr\right)
=\frac{K}{L}\left\{\frac{Km(m-1)}2+mr\right\}.
\label{eq:app-savings}
\end{equation}
Substituting $m=(T-r)/K$ makes the dependence on rollout length explicit:
\[
D(T)=\frac{T^2-KT+r(K-r)}{2L}.
\]
For $K\ll T$, the terms involving $K$ are small compared with $T^2$, giving
\[
D(T)\approx\frac{T^2}{2L},\qquad
\frac{D(T)}{T}\approx\frac{T}{2L}.
\]
Thus longer rollouts save a larger fraction of tokens: about $25\%$ for $T=L/2$ and nearly $50\%$ for $T=L$. At the max length, when $K$ divides $L$, the exact saving is $D(L)=(L-K)/2$. Rollouts ending within the first block ($T\leq K$) save no tokens.

Subtracting the expected saving from the complete length gives the expected generated length of rollout $i$:
\[
\E_S M_i=T_i-D(T_i).
\]
Here, $\E_S$ averages over the sampled stopping positions for a given complete rollout length $T_i$. In particular, a rollout reaching the max length retains $(L+K)/2$ tokens on average when $K$ divides $L$.

For a batch of $B$ rollouts, Vanilla OPD generates $N=\sum_{i=1}^BT_i$ tokens. Random truncation never generates more tokens than this. Summing the per-rollout expectations gives
\begin{equation}
\sum_i M_i\leq N\quad\text{for every set of stopping positions},
\qquad
\E_S\sum_i M_i=N-\sum_iD(T_i).
\label{eq:app-batch-cost}
\end{equation}

\subsection{Time savings}
\label{app:time-savings}
We model generation time using the number of decoding rounds and the
context lengths processed in those rounds. During autoregressive decoding,
the key--value (KV) cache stores attention keys and values for the prompt
and previously generated tokens. Generating each new token requires a
forward pass, whose attention reads this growing cached context. Random
truncation avoids forward-pass computation for skipped tokens, reduces
KV-cache reads, and limits the growth of each rollout's cache. This model explains when
random truncation can save more time than token savings alone would predict.
Let $\mathcal R_r$ be the set of active rollouts in decoding round $r$,
each generating one token. Over $R$ rounds, we approximate the time per
round and the total decoding time as
\begin{equation}
\tau_r
\approx
\phi
+
\psi\sum_{j\in\mathcal R_r}c_{j,r},
\qquad
t_{\mathrm{decode}}
\approx
\phi R+\psi\mathcal A,
\qquad
\mathcal A
=
\sum_{r=1}^{R}
\sum_{j\in\mathcal R_r}c_{j,r},
\label{eq:app-time-model}
\end{equation}
where $c_{j,r}$ is the context length of rollout $j$ in round $r$ and
$\phi\geq0$ is the fixed time cost of each round. The coefficient $\psi\geq0$
models the additional time per context token processed.
We call the cumulative sum $\mathcal A$ the KV area: it measures the total
cached context processed across decoding rounds. For a fixed model, it
serves as a proxy for attention work and KV-cache reads.

\noindent\textbf{KV area.}
For rollout $i$ with prompt length $T_i^{\mathrm p}$, the context before generating
token $t\geq1$ contains $T_i^{\mathrm p}$ prompt tokens and $t-1$ earlier rollout
tokens. Adding these context lengths over the generated tokens gives
\begin{equation}
\mathcal A_i
=
\sum_{t=1}^{M_i}(T_i^{\mathrm p}+t-1)
=
\sum_{t=1}^{T_i}\mathbf 1\{S_i\geq t\}(T_i^{\mathrm p}+t-1),
\qquad
\mathcal A=\sum_{i=1}^{B}\mathcal A_i.
\label{eq:app-area}
\end{equation}
Using $\E_S[\mathbf 1\{S_i\geq t\}]=p_t$ and
$\sum_{t=1}^{T_i}p_t=T_i-D(T_i)$ gives
\begin{align}
\E_S[\mathcal A_i]
&=
\sum_{t=1}^{T_i}p_t(T_i^{\mathrm p}+t-1)
=
T_i^{\mathrm p}\{T_i-D(T_i)\}
+
\sum_{t=1}^{T_i}(t-1)p_t,
\label{eq:app-area-expectation}\\
\mathcal A_i^{\mathrm V}
&=
\sum_{t=1}^{T_i}(T_i^{\mathrm p}+t-1)
=
T_i^{\mathrm p}T_i+\frac{T_i(T_i-1)}{2},
\qquad
\mathcal A_i^{\mathrm V}-\E_S[\mathcal A_i]
=
\sum_{t=1}^{T_i}(1-p_t)(T_i^{\mathrm p}+t-1).
\label{eq:app-area-saving}
\end{align}
Truncation is more likely to skip later tokens, each of which would
require processing a longer context.

To account for block boundaries, let $o_t$ count the positions before
token $t$ within its block. Eq.~\eqref{eq:survival} then gives
\begin{equation}
p_t
=
1-\frac{t-1-o_t}{L},
\qquad
o_t
=
(t-1)-K\left\lfloor\frac{t-1}{K}\right\rfloor,
\qquad
0\leq o_t<K.
\label{eq:app-survival-remainder}
\end{equation}
For $K\ll T_i$, summing over positions yields
\begin{align}
\E_S[M_i]
&=
\sum_{t=1}^{T_i}p_t
=
T_i-\frac{T_i(T_i-1)}{2L}
+\frac1L\sum_{t=1}^{T_i}o_t
\nonumber\\
&=
T_i-\frac{T_i^2}{2L}
+
O\!\left(\frac{KT_i}{L}\right),
\nonumber\\
\sum_{t=1}^{T_i}(t-1)p_t
&=
\frac{T_i(T_i-1)}{2}
-
\frac{T_i(T_i-1)(2T_i-1)}{6L}
+
\frac1L\sum_{t=1}^{T_i}(t-1)o_t
\nonumber\\
&=
\frac{T_i^2}{2}
-
\frac{T_i^3}{3L}
+
O\!\left(
T_i+\frac{KT_i^2}{L}
\right).
\label{eq:app-length-area-expansion}
\end{align}


Random truncation therefore saves the following amounts of tokens and
KV area from the generated rollout context:
\begin{equation}
T_i-\E_S[M_i]
\approx
\frac{T_i^2}{2L},
\qquad
\frac{T_i(T_i-1)}{2}
-
\sum_{t=1}^{T_i}(t-1)p_t
\approx
\frac{T_i^3}{3L}.
\label{eq:app-token-saving}
\end{equation}

The token saving counts positions whose forward-pass computation is avoided.
The KV-area saving measures the cached rollout context that attention no longer
needs to read at those positions.

Each generated token also uses the $T_i^{\mathrm p}$ prompt tokens. Adding this
contribution gives
\begin{equation}
\E_S[\mathcal A_i]
\approx
T_i^{\mathrm p}\left(T_i-\frac{T_i^2}{2L}\right)
+
\frac{T_i^2}{2}
-
\frac{T_i^3}{3L},
\qquad
\mathcal A_i^{\mathrm V}-\E_S[\mathcal A_i]
\approx
\frac{T_i^{\mathrm p} T_i^2}{2L}
+
\frac{T_i^3}{3L}.
\label{eq:app-area-with-prompt}
\end{equation}


When prompts are much shorter than rollouts, we can neglect their
contribution to the KV area. Writing $u_i=T_i/L$ for the rollout length
as a fraction of the max length gives
\begin{equation}
\E_S[M_i]
\approx
T_i\left(1-\frac{u_i}{2}\right),
\qquad
\E_S[\mathcal A_i]
\approx
\frac{T_i^2}{2}
\left(1-\frac{2u_i}{3}\right).
\end{equation}
Dividing Vanilla OPD's token count and KV area by the corresponding
expectations under random truncation gives
\begin{equation}
\frac{T_i}{\E_S[M_i]}
\approx
\frac{1}{1-u_i/2},
\qquad
\frac{\mathcal A_i^{\mathrm V}}{\E_S[\mathcal A_i]}
\approx
\frac{1}{1-2u_i/3}.
\label{eq:app-area-ratios}
\end{equation}
Dividing the area ratio by the token-count ratio shows how much the
average context per generated token decreases:
\begin{equation}
\frac{
\mathcal A_i^{\mathrm V}/T_i
}{
\E_S[\mathcal A_i]/\E_S[M_i]
}
\approx
\frac{1-u_i/2}{1-2u_i/3}.
\label{eq:app-context-ratio}
\end{equation}

For $T_i=L$, the expected generated length is
\begin{equation}
\E_S[M_i]
=
\frac{L+K}{2}
\end{equation}
when $K$ divides $L$. Under the same small-block and short-prompt
approximations, a rollout of length $T_i=L/2$ saves about $25\%$ of its
tokens and $33\%$ of its KV area in expectation. At $T_i=L$, these savings
increase to about $50\%$ and $67\%$, respectively. Random truncation thus
reduces the context processed by a larger fraction than the token count.

\noindent\textbf{Effect of scheduling.}
When processing long contexts dominates the cost in Eq.~\eqref{eq:app-time-model},
decoding time is approximately proportional to KV area, so time savings
can exceed the fraction of tokens saved. Scheduling also affects the
speedup by determining the number of decoding rounds $R$. We examine two
cases, with batch areas $\mathcal A=\sum_i\mathcal A_i$ and
$\mathcal A^{\mathrm V}=\sum_i\mathcal A_i^{\mathrm V}$.

First, suppose new rollouts can start as others finish, keeping about
$Q$ context tokens active in the KV cache of both methods during most
rounds. We ignore the brief periods at the beginning and end when fewer
rollouts are active:
\begin{equation}
\sum_{j\in\mathcal R_r}c_{j,r}
\approx
Q.
\end{equation}
Each round contributes about $Q$ to the KV area, so
\begin{equation}
\mathcal A
\approx
QR,
\qquad
R
\approx
\frac{\mathcal A}{Q},
\qquad
t_{\mathrm{decode}}
\approx
\left(
\frac{\phi}{Q}
+\psi
\right)\mathcal A.
\label{eq:app-full-cache}
\end{equation}
Dividing the generated-token count $N^{\mathrm S}=\sum_iM_i$ by decoding
time gives the throughput:
\begin{equation}
\frac{N^{\mathrm S}}{t_{\mathrm{decode}}}
\approx
\frac{1}{\phi/Q+\psi}
\frac{N^{\mathrm S}}{\mathcal A}.
\label{eq:app-throughput}
\end{equation}
Let $t_{\mathrm V}$ denote Vanilla OPD's decoding time. The predicted
speedup under this scheduling regime is
\begin{equation}
\frac{
t_{\mathrm V}
}{
\E_S[t_{\mathrm{decode}}]
}
\approx
\frac{
\mathcal A^{\mathrm V}
}{
\E_S[\mathcal A]
}
=
\frac{
N
}{
\E_S[N^{\mathrm S}]
}
\frac{
\mathcal A^{\mathrm V}/N
}{
\E_S[\mathcal A]/\E_S[N^{\mathrm S}]
}.
\label{eq:app-speedup}
\end{equation}
Equation~\eqref{eq:app-speedup} separates the predicted speedup into
the reduction in token count and the reduction in mean context per token.
The first factor reflects fewer tokens and hence less forward-pass computation.
The second captures the throughput gain from shorter cached contexts:
at the same active cache occupancy, more rollouts can generate concurrently.
When the prompt contribution is neglected, both factors are at least one.

Second, suppose all rollouts start together and generate one token per
round without pauses. No new rollouts enter the batch, so generation
ends when the longest rollout finishes:
\begin{equation}
R
=
\max_iM_i,
\qquad
t_{\mathrm{decode}}
\approx
\phi\max_iM_i+\psi\mathcal A.
\label{eq:app-fixed-batch}
\end{equation}
Consider $n$ rollouts whose complete lengths equal $L$. Ignoring block
rounding, let
\begin{equation}
U_1,\ldots,U_n
\stackrel{\mathrm{iid}}{\sim}
\mathrm{Uniform}(0,L),
\qquad
Z
=
\max_iU_i.
\end{equation}
For $0\leq z\leq L$, the largest stopping position is at most $z$ only
when all $n$ stopping positions are at most $z$. Each has probability
$z/L$ of doing so.
Independence lets us multiply these probabilities; integrating the
resulting tail probability gives the mean:
\begin{equation}
\Prob(Z\leq z)
=
\left(\frac zL\right)^n,
\qquad
\E[Z]
=
\int_0^L
\left[
1-\left(\frac zL\right)^n
\right]\mathrm dz
=
\frac{nL}{n+1}.
\label{eq:app-max}
\end{equation}
The expected gap between the max length and the largest stopping position is therefore
\begin{equation}
L-\E[Z]
=
\frac{L}{n+1}.
\end{equation}

Rounding $U_i$ up to a block boundary without exceeding the max length $L$ gives
\begin{equation}
U_i
\leq
S_i
\leq
\min\{U_i+K,L\}.
\end{equation}
Taking the maximum over rollouts and then the expectation gives
\begin{equation}
\frac{nL}{n+1}
\leq
\E_S\!\left[\max_iS_i\right]
\leq
\min\left\{
L,\,
\frac{nL}{n+1}+K
\right\}.
\label{eq:app-max-block}
\end{equation}
Random truncation can accelerate generation even when the longest retained
rollout remains close to $L$. It still removes many tokens across the batch,
avoiding their forward-pass computation and KV-cache reads. In the time
model, this reduces $\psi\mathcal A$ even when $\phi R$ changes little.
The greater the share of decoding time spent processing cached contexts,
the larger the resulting speedup.

\noindent\textbf{Practical implications.}
A larger batch can improve generation throughput while reducing truncation
noise. When the scheduler admits waiting rollouts as others finish, it can
reuse the KV-cache space freed by random truncation. A larger batch supplies
more waiting rollouts, helping sustain concurrency and cache occupancy.
The student update then combines gradient contributions from these same
rollouts, averaging over more independent stopping positions and reducing
the gradient variance caused by random truncation
(Appendix~\ref{app:batch-size}).

\section{Accuracy by Step}
\label{app:accuracy}

Table~\ref{tab:results} lists the accuracy behind Figures~\ref{fig:main-time} and~\ref{fig:scaling} at every evaluated checkpoint, and Tables~\ref{tab:results-optim}--\ref{tab:results-esr} the accuracy behind Figures~\ref{fig:optim}, \ref{fig:blockwidth}, \ref{fig:batchsize-math}, and~\ref{fig:esr}.

\begin{table}[h]
\centering
\caption{Accuracy (\%) every 25 steps, mean $\pm$ standard deviation, from the base model and the shared step-25 Vanilla warmup to 250 steps.}
\label{tab:results}
\footnotesize
\setlength{\tabcolsep}{2pt}
\begin{tabular}{llccccccccccc}
\toprule
 & & & \multicolumn{10}{c}{\textbf{Accuracy (\%) at step}}\\
\cmidrule(lr){4-13}
\textbf{Setting} & \textbf{Method} & \textbf{Base} & \textbf{25} & \textbf{50} & \textbf{75} & \textbf{100} & \textbf{125} & \textbf{150} & \textbf{175} & \textbf{200} & \textbf{225} & \textbf{250}\\
\midrule
\multicolumn{13}{l}{\textit{Main results: Qwen3-1.7B-Base $\leftarrow$ Qwen3-4B-Instruct-2507}}\\
\multirow{2}{*}{Math} & Vanilla & \multirow{2}{*}{11.7} & \multirow{2}{*}{23.4$_{\pm0.8}$} & 24.1$_{\pm0.0}$ & 27.9$_{\pm0.7}$ & 26.9$_{\pm0.2}$ & 28.4$_{\pm0.6}$ & 28.5$_{\pm0.7}$ & 29.3$_{\pm0.4}$ & 28.1$_{\pm1.8}$ & 30.0$_{\pm0.4}$ & 30.1$_{\pm0.5}$\\
 & \method{} & & & 24.4$_{\pm0.9}$ & 27.3$_{\pm0.7}$ & 27.5$_{\pm0.5}$ & 27.7$_{\pm0.6}$ & 29.0$_{\pm0.5}$ & 29.1$_{\pm0.1}$ & 27.9$_{\pm1.4}$ & 28.5$_{\pm0.6}$ & 29.4$_{\pm0.6}$\\
\multirow{2}{*}{Code} & Vanilla & \multirow{2}{*}{21.8} & \multirow{2}{*}{27.6$_{\pm0.2}$} & 26.4$_{\pm0.2}$ & 25.5$_{\pm0.3}$ & 26.5$_{\pm0.2}$ & 27.2$_{\pm0.3}$ & 27.6$_{\pm0.2}$ & 27.6$_{\pm0.3}$ & 28.1$_{\pm0.3}$ & 28.5$_{\pm0.1}$ & 28.5$_{\pm0.1}$\\
 & \method{} & & & 26.3$_{\pm0.3}$ & 25.7$_{\pm0.4}$ & 25.8$_{\pm0.3}$ & 26.9$_{\pm0.3}$ & 27.4$_{\pm0.3}$ & 27.2$_{\pm0.1}$ & 27.5$_{\pm0.6}$ & 27.8$_{\pm0.4}$ & 28.1$_{\pm0.4}$\\
\multirow{2}{*}{\shortstack[l]{Instruction\\following}} & Vanilla & \multirow{2}{*}{15.6} & \multirow{2}{*}{24.6$_{\pm0.1}$} & 34.8$_{\pm0.5}$ & 44.4$_{\pm0.7}$ & 49.9$_{\pm1.0}$ & 52.2$_{\pm0.2}$ & 54.3$_{\pm0.2}$ & 57.0$_{\pm0.5}$ & 58.6$_{\pm1.1}$ & 59.9$_{\pm1.4}$ & 59.8$_{\pm1.3}$\\
 & \method{} & & & 35.0$_{\pm0.7}$ & 44.3$_{\pm0.5}$ & 50.6$_{\pm0.2}$ & 51.3$_{\pm0.5}$ & 53.6$_{\pm0.3}$ & 56.2$_{\pm0.3}$ & 59.0$_{\pm0.5}$ & 59.1$_{\pm1.2}$ & 60.1$_{\pm0.9}$\\
\midrule
\multicolumn{13}{l}{\textit{Scaling: mathematics, Qwen3-30B-A3B-Instruct-2507 teacher}}\\
\multirow{2}{*}{4B} & Vanilla & \multirow{2}{*}{15.2} & \multirow{2}{*}{34.9$_{\pm0.8}$} & 34.7$_{\pm0.7}$ & 41.1$_{\pm0.7}$ & 44.0$_{\pm0.6}$ & 44.2$_{\pm0.2}$ & 45.0$_{\pm0.5}$ & 45.0$_{\pm0.4}$ & 46.2$_{\pm0.9}$ & 46.3$_{\pm0.3}$ & 46.2$_{\pm0.2}$\\
 & \method{} & & & 34.7$_{\pm0.6}$ & 40.1$_{\pm0.6}$ & 43.5$_{\pm0.2}$ & 43.7$_{\pm0.3}$ & 44.9$_{\pm0.7}$ & 44.9$_{\pm0.5}$ & 45.5$_{\pm0.6}$ & 46.3$_{\pm0.7}$ & 45.3$_{\pm0.4}$\\
\multirow{2}{*}{8B} & Vanilla & \multirow{2}{*}{23.5} & \multirow{2}{*}{38.0$_{\pm0.3}$} & 39.3$_{\pm1.3}$ & 45.8$_{\pm0.8}$ & 47.1$_{\pm0.2}$ & 48.3$_{\pm0.4}$ & 48.3$_{\pm0.2}$ & 49.3$_{\pm0.4}$ & 50.1$_{\pm0.1}$ & 50.2$_{\pm0.3}$ & 50.1$_{\pm0.1}$\\
 & \method{} & & & 38.7$_{\pm0.9}$ & 44.6$_{\pm0.1}$ & 47.3$_{\pm0.8}$ & 47.2$_{\pm0.9}$ & 48.8$_{\pm0.7}$ & 48.6$_{\pm0.4}$ & 49.9$_{\pm0.8}$ & 49.1$_{\pm0.6}$ & 49.7$_{\pm0.8}$\\
\midrule
\multicolumn{13}{l}{\textit{Scaling: mathematics, Qwen3-235B-A22B-Instruct-2507 teacher}}\\
\multirow{2}{*}{30B} & Vanilla & \multirow{2}{*}{13.3} & \multirow{2}{*}{36.4$_{\pm0.6}$} & 40.9$_{\pm0.9}$ & 48.8$_{\pm0.7}$ & 53.4$_{\pm0.3}$ & 54.7$_{\pm0.2}$ & 56.0$_{\pm0.6}$ & 56.2$_{\pm0.7}$ & 56.5$_{\pm0.4}$ & 56.4$_{\pm0.6}$ & 56.9$_{\pm0.8}$\\
 & \method{} & & & 40.6$_{\pm0.7}$ & 48.8$_{\pm0.5}$ & 52.4$_{\pm0.1}$ & 54.7$_{\pm0.4}$ & 55.3$_{\pm0.8}$ & 56.4$_{\pm0.1}$ & 56.7$_{\pm0.2}$ & 57.2$_{\pm0.6}$ & 57.2$_{\pm0.3}$\\
\bottomrule
\end{tabular}

\end{table}

\begin{table}[h]
\centering
\caption{Accuracy (\%) every 25 steps under each optimizer in mathematics, mean $\pm$ standard deviation.}
\label{tab:results-optim}
\footnotesize
\setlength{\tabcolsep}{2pt}
\begin{tabular}{llccccccccccc}
\toprule
 & & & \multicolumn{10}{c}{\textbf{Accuracy (\%) at step}}\\
\cmidrule(lr){4-13}
\textbf{Optimizer} & \textbf{Method} & \textbf{Base} & \textbf{25} & \textbf{50} & \textbf{75} & \textbf{100} & \textbf{125} & \textbf{150} & \textbf{175} & \textbf{200} & \textbf{225} & \textbf{250}\\
\midrule
\multicolumn{13}{l}{\textit{Mathematics: Qwen3-1.7B-Base $\leftarrow$ Qwen3-4B-Instruct-2507}}\\
\multirow{2}{*}{AdamW} & Vanilla & \multirow{2}{*}{11.7} & \multirow{2}{*}{23.4$_{\pm0.8}$} & 24.1$_{\pm0.0}$ & 27.9$_{\pm0.7}$ & 26.9$_{\pm0.2}$ & 28.4$_{\pm0.6}$ & 28.5$_{\pm0.7}$ & 29.3$_{\pm0.4}$ & 28.1$_{\pm1.8}$ & 30.0$_{\pm0.4}$ & 30.1$_{\pm0.5}$\\
 & \method{} & & & 24.4$_{\pm0.9}$ & 27.3$_{\pm0.7}$ & 27.5$_{\pm0.5}$ & 27.7$_{\pm0.6}$ & 29.0$_{\pm0.5}$ & 29.1$_{\pm0.1}$ & 27.9$_{\pm1.4}$ & 28.5$_{\pm0.6}$ & 29.4$_{\pm0.6}$\\
\multirow{2}{*}{SGD} & Vanilla & \multirow{2}{*}{11.7} & \multirow{2}{*}{22.6$_{\pm0.5}$} & 24.4$_{\pm0.3}$ & 24.3$_{\pm0.3}$ & 24.5$_{\pm1.6}$ & 24.9$_{\pm0.1}$ & 24.4$_{\pm0.6}$ & 24.9$_{\pm0.2}$ & 25.2$_{\pm0.8}$ & 25.3$_{\pm0.3}$ & 25.2$_{\pm0.5}$\\
 & \method{} & & & 24.2$_{\pm0.5}$ & 24.6$_{\pm0.2}$ & 24.7$_{\pm0.3}$ & 24.2$_{\pm0.2}$ & 25.0$_{\pm0.3}$ & 24.1$_{\pm0.4}$ & 25.2$_{\pm0.2}$ & 25.2$_{\pm0.6}$ & 25.1$_{\pm0.6}$\\
\multirow{2}{*}{Muon} & Vanilla & \multirow{2}{*}{11.7} & \multirow{2}{*}{21.9$_{\pm0.3}$} & 24.6$_{\pm0.9}$ & 20.1$_{\pm0.5}$ & 25.0$_{\pm0.4}$ & 25.9$_{\pm0.9}$ & 26.5$_{\pm0.4}$ & 28.2$_{\pm0.1}$ & 28.8$_{\pm0.5}$ & 27.8$_{\pm0.1}$ & 29.5$_{\pm0.2}$\\
 & \method{} & & & 24.5$_{\pm0.2}$ & 20.7$_{\pm0.2}$ & 24.1$_{\pm0.9}$ & 25.8$_{\pm0.5}$ & 26.0$_{\pm0.4}$ & 27.6$_{\pm0.8}$ & 28.2$_{\pm0.2}$ & 28.4$_{\pm0.6}$ & 29.8$_{\pm0.2}$\\
\bottomrule
\end{tabular}

\end{table}

\begin{table}[h]
\centering
\caption{Accuracy (\%) every 25 steps for each block width $K$ in mathematics, mean $\pm$ standard deviation.}
\label{tab:results-blockwidth}
\footnotesize
\setlength{\tabcolsep}{2pt}
\begin{tabular}{llccccccccccc}
\toprule
 & & & \multicolumn{10}{c}{\textbf{Accuracy (\%) at step}}\\
\cmidrule(lr){4-13}
\textbf{\boldmath$K$} & \textbf{Method} & \textbf{Base} & \textbf{25} & \textbf{50} & \textbf{75} & \textbf{100} & \textbf{125} & \textbf{150} & \textbf{175} & \textbf{200} & \textbf{225} & \textbf{250}\\
\midrule
\multicolumn{13}{l}{\textit{Mathematics: Qwen3-1.7B-Base $\leftarrow$ Qwen3-4B-Instruct-2507}}\\
 & Vanilla & \multirow{5}{*}{11.7} & \multirow{5}{*}{23.4$_{\pm0.8}$} & 24.1$_{\pm0.0}$ & 27.9$_{\pm0.7}$ & 26.9$_{\pm0.2}$ & 28.4$_{\pm0.6}$ & 28.5$_{\pm0.7}$ & 29.3$_{\pm0.4}$ & 28.1$_{\pm1.8}$ & 30.0$_{\pm0.4}$ & 30.1$_{\pm0.5}$\\
32 & \method{} &  &  & 24.1$_{\pm0.2}$ & 27.3$_{\pm0.4}$ & 27.1$_{\pm0.5}$ & 27.6$_{\pm0.1}$ & 28.2$_{\pm0.6}$ & 29.2$_{\pm0.4}$ & 29.2$_{\pm0.5}$ & 28.8$_{\pm0.3}$ & 29.8$_{\pm0.7}$\\
128 & \method{} &  &  & 24.4$_{\pm0.9}$ & 27.3$_{\pm0.7}$ & 27.5$_{\pm0.5}$ & 27.7$_{\pm0.6}$ & 29.0$_{\pm0.5}$ & 29.1$_{\pm0.1}$ & 27.9$_{\pm1.4}$ & 28.5$_{\pm0.6}$ & 29.4$_{\pm0.6}$\\
512 & \method{} &  &  & 24.1$_{\pm0.7}$ & 27.7$_{\pm0.5}$ & 27.7$_{\pm0.5}$ & 27.5$_{\pm0.8}$ & 28.8$_{\pm0.2}$ & 28.9$_{\pm0.7}$ & 28.7$_{\pm0.8}$ & 28.8$_{\pm1.1}$ & 29.7$_{\pm0.2}$\\
2048 & \method{} &  &  & 23.7$_{\pm0.2}$ & 27.5$_{\pm0.2}$ & 26.9$_{\pm0.3}$ & 27.7$_{\pm0.5}$ & 28.4$_{\pm0.6}$ & 29.1$_{\pm0.6}$ & 28.5$_{\pm1.0}$ & 29.4$_{\pm0.6}$ & 30.0$_{\pm0.2}$\\
\bottomrule
\end{tabular}

\end{table}

\begin{table}[h]
\centering
\caption{Accuracy (\%) every 25 steps for each batch size $B$ in mathematics, mean $\pm$ standard deviation.}
\label{tab:results-batch-math}
\footnotesize
\setlength{\tabcolsep}{8pt}
\begin{tabular}{llccccc}
\toprule
 & & & \multicolumn{4}{c}{\textbf{Accuracy (\%) at step}}\\
\cmidrule(lr){4-7}
\textbf{\boldmath$B$} & \textbf{Method} & \textbf{Base} & \textbf{25} & \textbf{50} & \textbf{75} & \textbf{100}\\
\midrule
\multicolumn{7}{l}{\textit{Mathematics: Qwen3-1.7B-Base $\leftarrow$ Qwen3-4B-Instruct-2507}}\\
\multirow{2}{*}{128} & Vanilla & \multirow{2}{*}{11.7} & \multirow{2}{*}{23.1$_{\pm0.3}$} & 23.5$_{\pm0.6}$ & 27.6$_{\pm0.1}$ & 27.8$_{\pm0.3}$\\
 & \method{} & & & 24.3$_{\pm0.4}$ & 27.9$_{\pm0.6}$ & 28.4$_{\pm0.8}$\\
\multirow{2}{*}{256} & Vanilla & \multirow{2}{*}{11.7} & \multirow{2}{*}{23.4$_{\pm0.8}$} & 24.1$_{\pm0.0}$ & 27.9$_{\pm0.7}$ & 26.9$_{\pm0.2}$\\
 & \method{} & & & 24.4$_{\pm0.9}$ & 27.3$_{\pm0.7}$ & 27.5$_{\pm0.5}$\\
\multirow{2}{*}{512} & Vanilla & \multirow{2}{*}{11.7} & \multirow{2}{*}{24.4$_{\pm0.4}$} & 24.6$_{\pm0.1}$ & 27.3$_{\pm0.2}$ & 27.5$_{\pm0.2}$\\
 & \method{} & & & 24.9$_{\pm0.7}$ & 27.1$_{\pm0.2}$ & 27.0$_{\pm0.3}$\\
\bottomrule
\end{tabular}

\end{table}

\begin{table}[h]
\centering
\caption{Accuracy (\%) every 25 steps for ESR in mathematics, mean $\pm$ standard deviation.}
\label{tab:results-esr}
\footnotesize
\setlength{\tabcolsep}{2pt}
\begin{tabular}{lccccccccccc}
\toprule
 & & \multicolumn{10}{c}{\textbf{Accuracy (\%) at step}}\\
\cmidrule(lr){3-12}
\textbf{Method} & \textbf{Base} & \textbf{25} & \textbf{50} & \textbf{75} & \textbf{100} & \textbf{125} & \textbf{150} & \textbf{175} & \textbf{200} & \textbf{225} & \textbf{250}\\
\midrule
\multicolumn{12}{l}{\textit{Mathematics: Qwen3-1.7B-Base $\leftarrow$ Qwen3-4B-Instruct-2507}}\\
Vanilla & \multirow{3}{*}{11.7} & \multirow{3}{*}{23.4$_{\pm0.8}$} & 24.1$_{\pm0.0}$ & 27.9$_{\pm0.7}$ & 26.9$_{\pm0.2}$ & 28.4$_{\pm0.6}$ & 28.5$_{\pm0.7}$ & 29.3$_{\pm0.4}$ & 28.1$_{\pm1.8}$ & 30.0$_{\pm0.4}$ & 30.1$_{\pm0.5}$\\
\method{} &  &  & 24.4$_{\pm0.9}$ & 27.3$_{\pm0.7}$ & 27.5$_{\pm0.5}$ & 27.7$_{\pm0.6}$ & 29.0$_{\pm0.5}$ & 29.1$_{\pm0.1}$ & 27.9$_{\pm1.4}$ & 28.5$_{\pm0.6}$ & 29.4$_{\pm0.6}$\\
ESR &  &  & 24.5$_{\pm0.1}$ & 23.7$_{\pm1.2}$ & 23.4$_{\pm0.4}$ & 24.0$_{\pm0.4}$ & 23.4$_{\pm0.1}$ & 23.2$_{\pm0.4}$ & 23.0$_{\pm0.4}$ & 23.2$_{\pm0.5}$ & 23.5$_{\pm0.6}$\\
\bottomrule
\end{tabular}

\end{table}



\FloatBarrier   
\section{Additional Time Decomposition}
\label{app:scaling-cost}

Tables~\ref{tab:speed-scaling}, \ref{tab:time-optim}, \ref{tab:time-blockwidth}, and~\ref{tab:time-batch-math} break down the training step for the 4B{, 8B, and 30B} students (Section~\ref{sec:scaling}), the optimizers (Section~\ref{sec:optimizers}), the block widths (Section~\ref{sec:block-width}), and the batch sizes (Section~\ref{sec:batch-size-math}), respectively.

\begin{table}[h]
\centering
\caption{Cost of a training step for the 4B{, 8B, and 30B} students in mathematics.
}
\label{tab:speed-scaling}
\small
\setlength{\tabcolsep}{6pt}
\begin{tabular}{llrrrrrr}
\toprule
 & & & & \multicolumn{4}{c}{\textbf{Time per step (s)}}\\
\cmidrule(lr){5-8}
\textbf{Student} & \textbf{Method} & \textbf{Tokens} & \textbf{k tok/s} & \textbf{Generation} & \textbf{Update} & \textbf{Other} & \textbf{Total}\\
\midrule
\multirow{3}{*}{4B} & Vanilla & 10{,}005 & 3.4 & 753 & 360 & 95 & 1209\\
 & \method{} & 6{,}511 & 4.3 & 392 & 219 & 61 & 672\\
 & \textcolor{cornellred}{\textit{Speedup}} & \textcolor{cornellred}{1.5$\times$} & \textcolor{cornellred}{1.3$\times$} & \textcolor{cornellred}{1.9$\times$} & \textcolor{cornellred}{1.6$\times$} & \textcolor{cornellred}{1.6$\times$} & \textcolor{cornellred}{\textbf{1.8$\times$}}\\
\midrule
\multirow{3}{*}{8B} & Vanilla & 9{,}432 & 5.6 & 434 & 303 & 81 & 818\\
 & \method{} & 6{,}276 & 6.5 & 246 & 190 & 54 & 490\\
 & \textcolor{cornellred}{\textit{Speedup}} & \textcolor{cornellred}{1.5$\times$} & \textcolor{cornellred}{1.2$\times$} & \textcolor{cornellred}{1.8$\times$} & \textcolor{cornellred}{1.6$\times$} & \textcolor{cornellred}{1.5$\times$} & \textcolor{cornellred}{\textbf{1.7$\times$}}\\
\midrule
\multirow{3}{*}{30B} & Vanilla & 8{,}342 & 6.0 & 355 & 192 & 78 & 625\\
 & \method{} & 5{,}755 & 6.5 & 228 & 134 & 63 & 425\\
 & \textcolor{cornellred}{\textit{Speedup}} & \textcolor{cornellred}{1.4$\times$} & \textcolor{cornellred}{1.1$\times$} & \textcolor{cornellred}{1.6$\times$} & \textcolor{cornellred}{1.4$\times$} & \textcolor{cornellred}{1.2$\times$} & \textcolor{cornellred}{\textbf{1.5$\times$}}\\
\bottomrule
\end{tabular}

\end{table}

\begin{table}[h]
\centering
\caption{Cost of a training step under each optimizer in mathematics.}
\label{tab:time-optim}
\small
\setlength{\tabcolsep}{6pt}
\begin{tabular}{llrrrrrr}
\toprule
\multicolumn{4}{l}{\textit{Qwen3-1.7B-Base $\leftarrow$ Qwen3-4B-Instruct-2507}} & \multicolumn{4}{c}{\textbf{Time per step (s)}}\\
\cmidrule(lr){5-8}
\textbf{Optimizer} & \textbf{Method} & \textbf{Tokens} & \textbf{k tok/s} & \textbf{Generation} & \textbf{Update} & \textbf{Other} & \textbf{Total}\\
\midrule
\multirow{3}{*}{AdamW} & Vanilla & 9{,}516 & 3.5 & 697 & 189 & 56 & 941\\
 & \method{} & 6{,}255 & 4.5 & 358 & 119 & 36 & 513\\
 & \textcolor{cornellred}{\textit{Speedup}} & \textcolor{cornellred}{1.5$\times$} & \textcolor{cornellred}{1.3$\times$} & \textcolor{cornellred}{2.0$\times$} & \textcolor{cornellred}{1.6$\times$} & \textcolor{cornellred}{1.5$\times$} & \textcolor{cornellred}{\textbf{1.8$\times$}}\\
\midrule
\multirow{3}{*}{SGD} & Vanilla & 9{,}278 & 2.9 & 810 & 185 & 50 & 1045\\
 & \method{} & 5{,}474 & 4.1 & 342 & 109 & 34 & 485\\
 & \textcolor{cornellred}{\textit{Speedup}} & \textcolor{cornellred}{1.7$\times$} & \textcolor{cornellred}{1.4$\times$} & \textcolor{cornellred}{2.4$\times$} & \textcolor{cornellred}{1.7$\times$} & \textcolor{cornellred}{1.5$\times$} & \textcolor{cornellred}{\textbf{2.2$\times$}}\\
\midrule
\multirow{3}{*}{Muon} & Vanilla & 9{,}991 & 3.3 & 784 & 192 & 52 & 1029\\
 & \method{} & 6{,}247 & 4.2 & 379 & 124 & 39 & 542\\
 & \textcolor{cornellred}{\textit{Speedup}} & \textcolor{cornellred}{1.6$\times$} & \textcolor{cornellred}{1.3$\times$} & \textcolor{cornellred}{2.1$\times$} & \textcolor{cornellred}{1.6$\times$} & \textcolor{cornellred}{1.3$\times$} & \textcolor{cornellred}{\textbf{1.9$\times$}}\\
\bottomrule
\end{tabular}

\end{table}

\begin{table}[h]
\centering
\caption{Cost of a training step for each block width $K$ in mathematics.}
\label{tab:time-blockwidth}
\small
\setlength{\tabcolsep}{6pt}
\begin{tabular}{llrrrrrr}
\toprule
\multicolumn{4}{l}{\textit{Qwen3-1.7B-Base $\leftarrow$ Qwen3-4B-Instruct-2507}} & \multicolumn{4}{c}{\textbf{Time per step (s)}}\\
\cmidrule(lr){5-8}
\textbf{\boldmath$K$} & \textbf{Method} & \textbf{Tokens} & \textbf{k tok/s} & \textbf{Generation} & \textbf{Update} & \textbf{Other} & \textbf{Total}\\
\midrule
 & Vanilla & 9{,}516 & 3.5 & 697 & 189 & 56 & 941\\
\midrule
\multirow{2}{*}{32} & \method{} & 6{,}227 & 4.4 & 366 & 124 & 41 & 531\\
 & \textcolor{cornellred}{\textit{Speedup}} & \textcolor{cornellred}{1.5$\times$} & \textcolor{cornellred}{1.2$\times$} & \textcolor{cornellred}{1.9$\times$} & \textcolor{cornellred}{1.5$\times$} & \textcolor{cornellred}{1.4$\times$} & \textcolor{cornellred}{\textbf{1.8$\times$}}\\
\midrule
\multirow{2}{*}{128} & \method{} & 6{,}255 & 4.5 & 358 & 119 & 36 & 513\\
 & \textcolor{cornellred}{\textit{Speedup}} & \textcolor{cornellred}{1.5$\times$} & \textcolor{cornellred}{1.3$\times$} & \textcolor{cornellred}{2.0$\times$} & \textcolor{cornellred}{1.6$\times$} & \textcolor{cornellred}{1.5$\times$} & \textcolor{cornellred}{\textbf{1.8$\times$}}\\
\midrule
\multirow{2}{*}{512} & \method{} & 6{,}365 & 4.4 & 372 & 124 & 39 & 535\\
 & \textcolor{cornellred}{\textit{Speedup}} & \textcolor{cornellred}{1.5$\times$} & \textcolor{cornellred}{1.3$\times$} & \textcolor{cornellred}{1.9$\times$} & \textcolor{cornellred}{1.5$\times$} & \textcolor{cornellred}{1.4$\times$} & \textcolor{cornellred}{\textbf{1.8$\times$}}\\
\midrule
\multirow{2}{*}{2048} & \method{} & 6{,}792 & 4.4 & 397 & 131 & 41 & 569\\
 & \textcolor{cornellred}{\textit{Speedup}} & \textcolor{cornellred}{1.4$\times$} & \textcolor{cornellred}{1.3$\times$} & \textcolor{cornellred}{1.8$\times$} & \textcolor{cornellred}{1.4$\times$} & \textcolor{cornellred}{1.4$\times$} & \textcolor{cornellred}{\textbf{1.7$\times$}}\\
\bottomrule
\end{tabular}

\end{table}


\begin{table}[h]
\centering
\caption{Cost of a training step for each batch size in mathematics over the first 100 steps.}
\label{tab:time-batch-math}
\small
\setlength{\tabcolsep}{6pt}
\begin{tabular}{llrrrrrr}
\toprule
\multicolumn{4}{l}{\textit{Qwen3-1.7B-Base $\leftarrow$ Qwen3-4B-Instruct-2507}} & \multicolumn{4}{c}{\textbf{Time per step (s)}}\\
\cmidrule(lr){5-8}
\textbf{Batch size} & \textbf{Method} & \textbf{Tokens} & \textbf{k tok/s} & \textbf{Generation} & \textbf{Update} & \textbf{Other} & \textbf{Total}\\
\midrule
\multirow{3}{*}{128} & Vanilla & 9{,}403 & 3.0 & 399 & 91 & 26 & 516\\
 & \method{} & 5{,}781 & 4.0 & 185 & 53 & 17 & 256\\
 & \textcolor{cornellred}{\textit{Speedup}} & \textcolor{cornellred}{1.6$\times$} & \textcolor{cornellred}{1.3$\times$} & \textcolor{cornellred}{2.2$\times$} & \textcolor{cornellred}{1.7$\times$} & \textcolor{cornellred}{1.5$\times$} & \textcolor{cornellred}{\textbf{2.0$\times$}}\\
\midrule
\multirow{3}{*}{256} & Vanilla & 9{,}650 & 3.1 & 793 & 195 & 57 & 1045\\
 & \method{} & 5{,}940 & 4.3 & 356 & 112 & 33 & 501\\
 & \textcolor{cornellred}{\textit{Speedup}} & \textcolor{cornellred}{1.6$\times$} & \textcolor{cornellred}{1.4$\times$} & \textcolor{cornellred}{2.2$\times$} & \textcolor{cornellred}{1.7$\times$} & \textcolor{cornellred}{1.7$\times$} & \textcolor{cornellred}{\textbf{2.1$\times$}}\\
\midrule
\multirow{3}{*}{512} & Vanilla & 9{,}779 & 3.3 & 1531 & 385 & 107 & 2024\\
 & \method{} & 6{,}041 & 4.5 & 682 & 229 & 67 & 978\\
 & \textcolor{cornellred}{\textit{Speedup}} & \textcolor{cornellred}{1.6$\times$} & \textcolor{cornellred}{1.4$\times$} & \textcolor{cornellred}{2.2$\times$} & \textcolor{cornellred}{1.7$\times$} & \textcolor{cornellred}{1.6$\times$} & \textcolor{cornellred}{\textbf{2.1$\times$}}\\
\bottomrule
\end{tabular}

\end{table}

\FloatBarrier   


%
%

\FloatBarrier
\section{Comparison with Early Stopping Rollout}
\label{app:esr}


\begin{figure}[h]
\centering
\includegraphics[width=\textwidth]{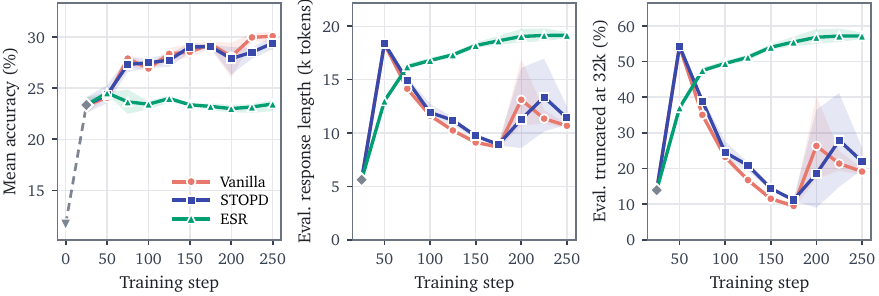}
\caption{ESR in mathematics: accuracy (left), response length at evaluation (middle), and truncated share (right).}
\label{fig:esr}
\end{figure}

Early Stopping Rollout (ESR) \citep{ziheng2026moreearlystoppingrollout} truncates each rollout after its first 100 tokens, or at natural termination, and distills only that prefix. We evaluate ESR in the same mathematics setting as our main experiments (Figure~\ref{fig:esr}). Accuracy every 25 steps is shown in Table~\ref{tab:results-esr} in Appendix~\ref{app:accuracy}.
The performance of ESR saturates after approximately 50 training steps, while \method{} and Vanilla continue increasing and outperform ESR by 5.9 and 6.6 points at step 250.
ESR's response length and truncation rate continue to increase throughout training, reaching approximately 20k tokens and 57\% at step 250, while Vanilla and \method{} peak around step 50 before declining to approximately 10k tokens and 20\%. This divergence arises from fixed-prefix distillation, which excludes end-of-sequence supervision whenever a response extends beyond the truncation boundary, weakening the student's ability to terminate generation. 
In contrast, \method{} preserves termination supervision by including every token position with positive probability, closely matching Vanilla's accuracy.

\FloatBarrier

\FloatBarrier

\section{Per-Benchmark Accuracy}
\label{app:benchmarks}

Figures~\ref{fig:bench-math}--{\ref{fig:bench-esr}} break the average accuracy of Sections~\ref{sec:main-results}--\ref{sec:batch-size-math} and {Appendix}~\ref{app:esr} into their benchmarks, against training step: the base model sits at step 0 and the shared warmup checkpoint at step 25, and each point averages the three seeds with the range across seeds shaded.

\begin{figure}[h]
\centering
\includegraphics[width=\textwidth]{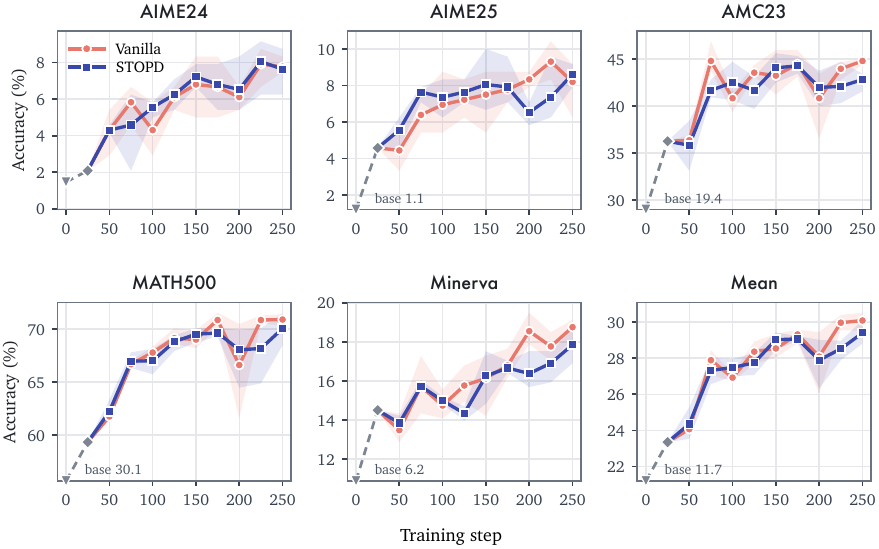}
\caption{Mathematics, Qwen3-1.7B-Base $\leftarrow$ Qwen3-4B-Instruct-2507: accuracy per benchmark against training step.}
\label{fig:bench-math}
\end{figure}

\begin{figure}[h]
\centering
\includegraphics[width=\textwidth]{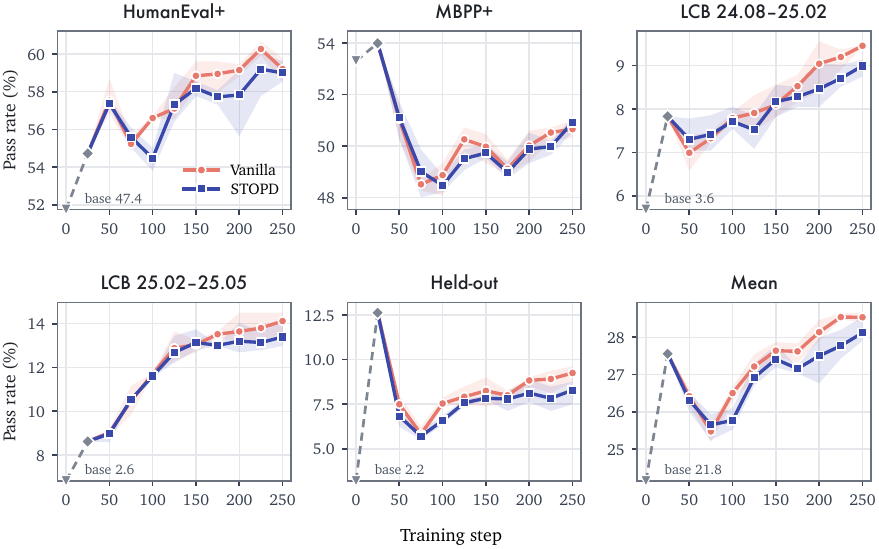}
\caption{Code, Qwen3-1.7B-Base $\leftarrow$ Qwen3-4B-Instruct-2507: pass rate per benchmark against training step.}
\label{fig:bench-code}
\end{figure}

\begin{figure}[h]
\centering
\includegraphics[width=\textwidth]{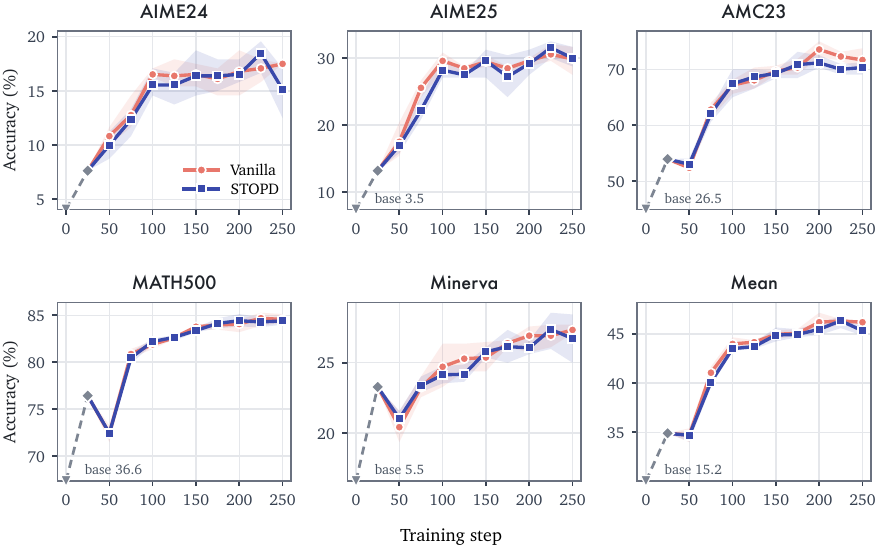}
\caption{Mathematics, Qwen3-4B-Base $\leftarrow$ Qwen3-30B-A3B-Instruct-2507: accuracy per benchmark against training step.}
\label{fig:bench-math-4b}
\end{figure}

\begin{figure}[h]
\centering
\includegraphics[width=\textwidth]{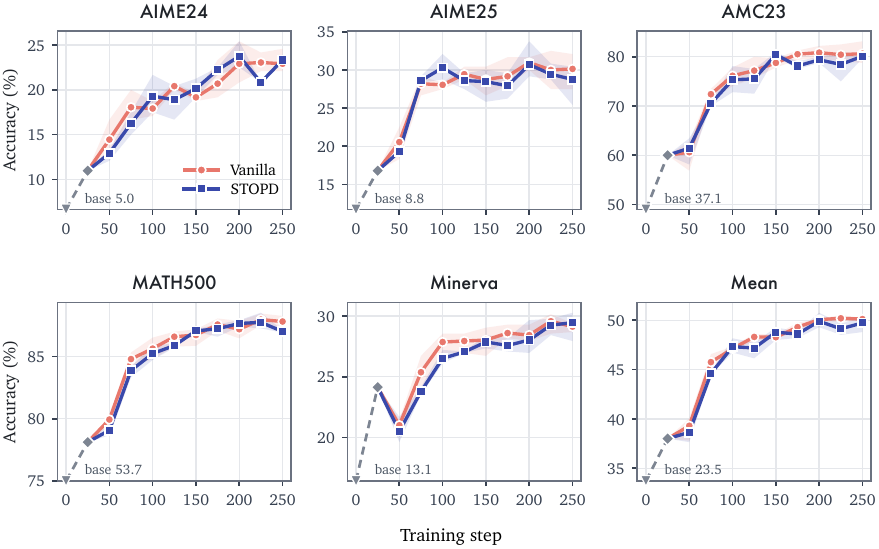}
\caption{Mathematics, Qwen3-8B-Base $\leftarrow$ Qwen3-30B-A3B-Instruct-2507: accuracy per benchmark against training step.}
\label{fig:bench-math-8b}
\end{figure}

\begin{figure}[h]
\centering
\includegraphics[width=\textwidth]{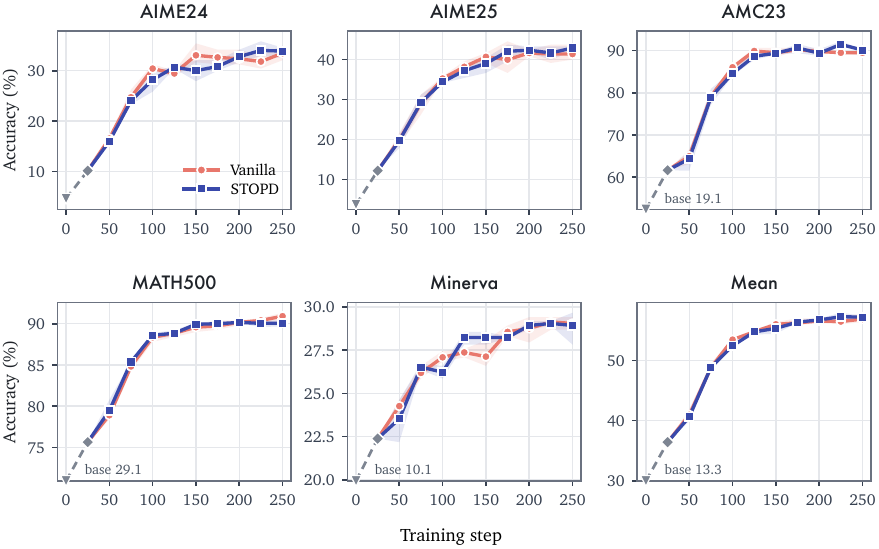}
\caption{{Mathematics, Qwen3-30B-A3B-Base $\leftarrow$ Qwen3-235B-A22B-Instruct-2507: accuracy per benchmark against training step.}}
\label{fig:bench-math-30b}
\end{figure}

\begin{figure}[h]
\centering
\includegraphics[width=\textwidth]{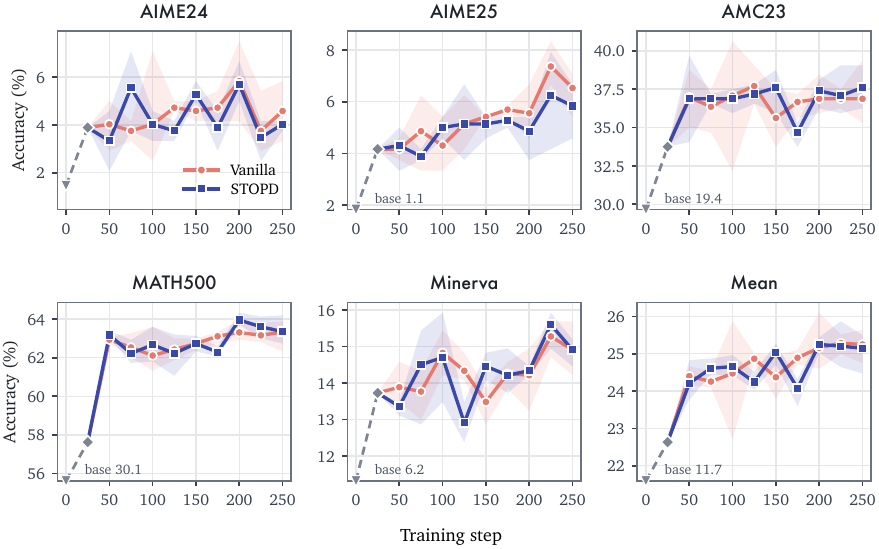}
\caption{Mathematics under SGD, Qwen3-1.7B-Base $\leftarrow$ Qwen3-4B-Instruct-2507: accuracy per benchmark against training step.}
\label{fig:bench-sgd}
\end{figure}

\begin{figure}[h]
\centering
\includegraphics[width=\textwidth]{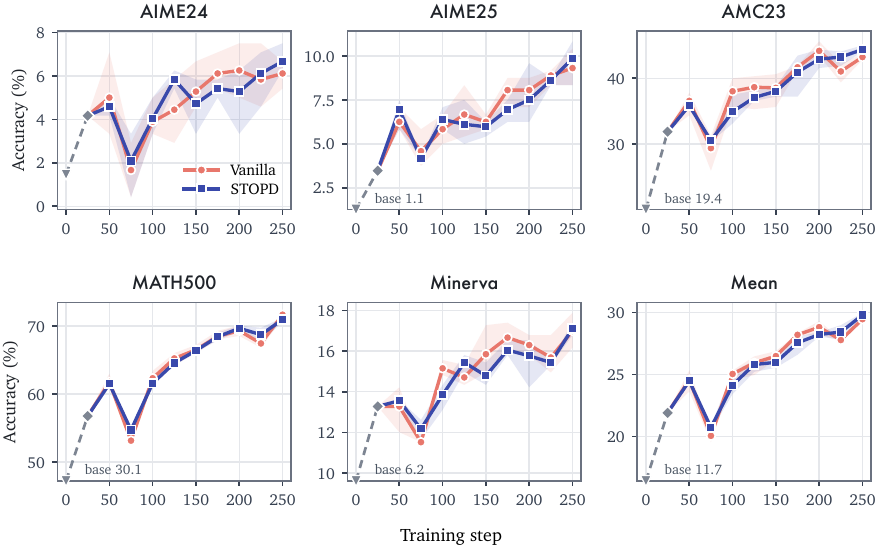}
\caption{Mathematics under Muon, Qwen3-1.7B-Base $\leftarrow$ Qwen3-4B-Instruct-2507: accuracy per benchmark against training step.}
\label{fig:bench-muon}
\end{figure}

\begin{figure}[h]
\centering
\includegraphics[width=\textwidth]{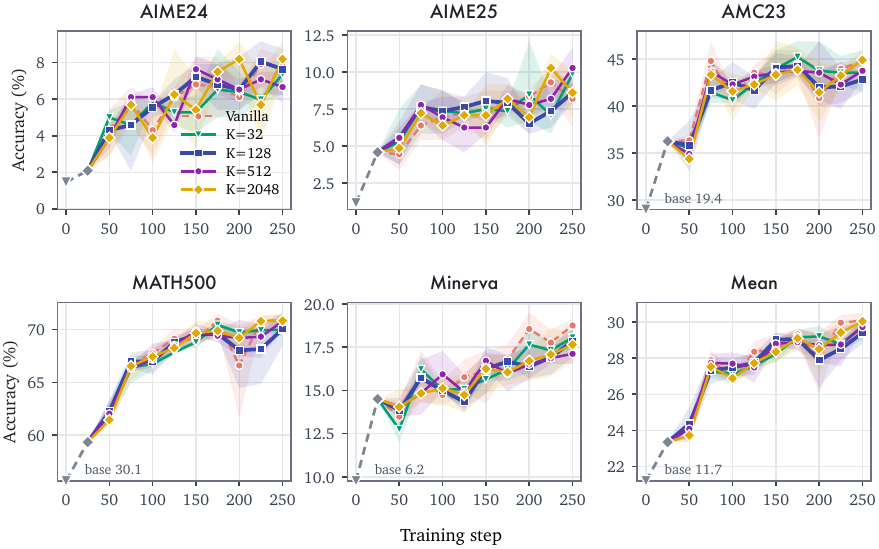}
\caption{Block width in mathematics, Qwen3-1.7B-Base $\leftarrow$ Qwen3-4B-Instruct-2507: accuracy per benchmark against training step.}
\label{fig:bench-blockwidth}
\end{figure}

\begin{figure}[h]
\centering
\includegraphics[width=\textwidth]{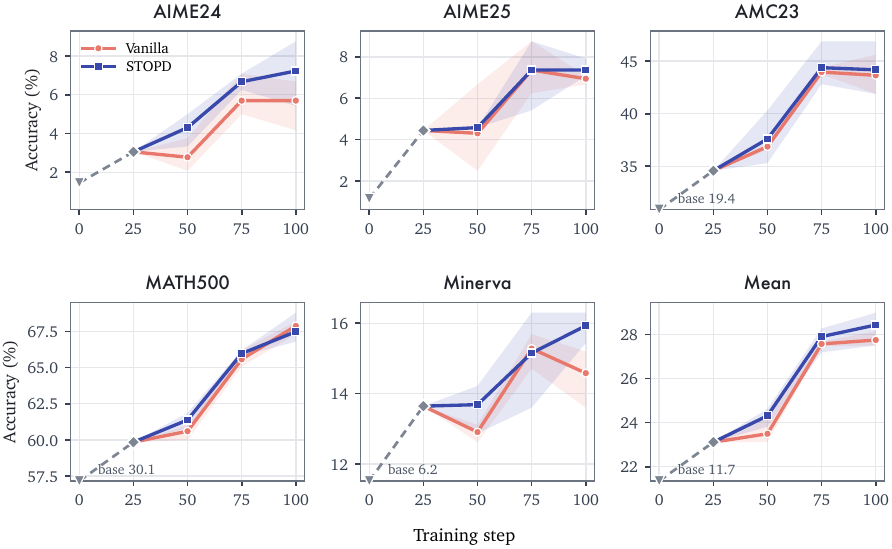}
\caption{Mathematics at $B=128$, Qwen3-1.7B-Base $\leftarrow$ Qwen3-4B-Instruct-2507: accuracy per benchmark against training step.}
\label{fig:bench-batch-math-128}
\end{figure}

\begin{figure}[h]
\centering
\includegraphics[width=\textwidth]{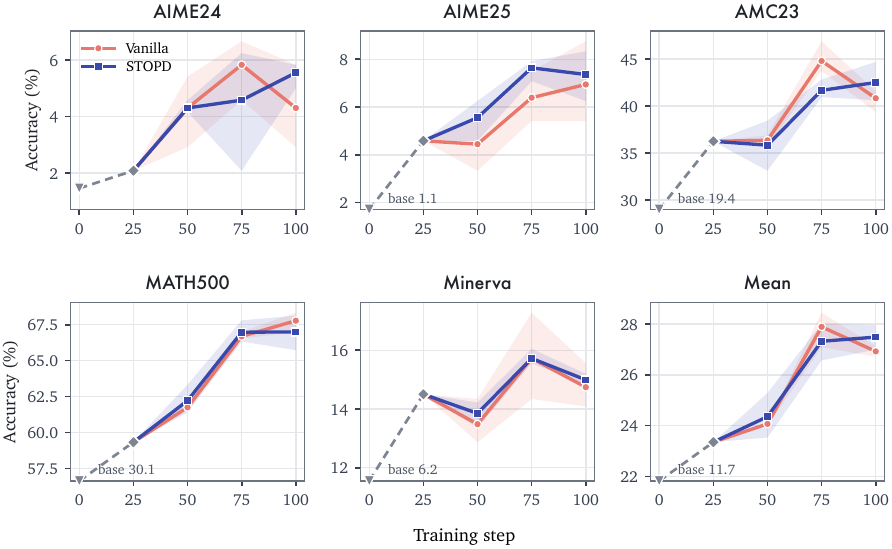}
\caption{Mathematics at $B=256$, Qwen3-1.7B-Base $\leftarrow$ Qwen3-4B-Instruct-2507: accuracy per benchmark against training step.}
\label{fig:bench-batch-math-256}
\end{figure}

\begin{figure}[h]
\centering
\includegraphics[width=\textwidth]{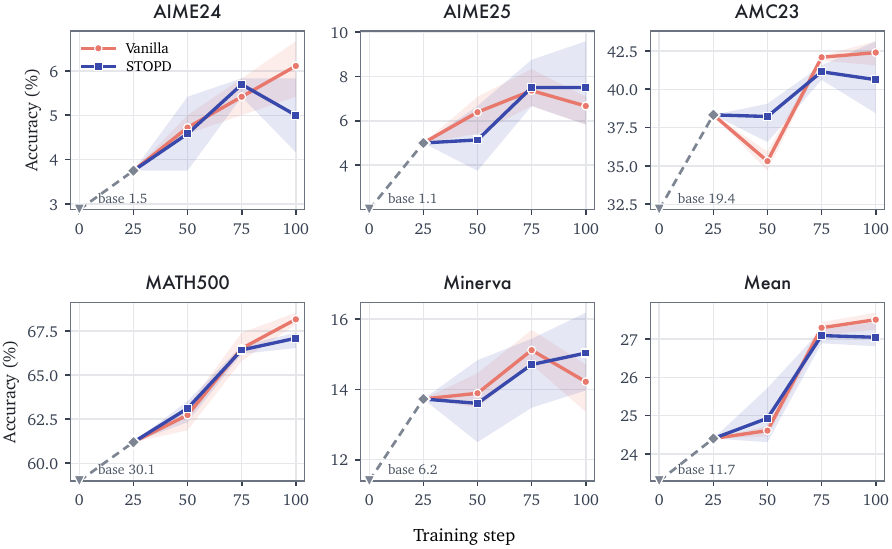}
\caption{Mathematics at $B=512$, Qwen3-1.7B-Base $\leftarrow$ Qwen3-4B-Instruct-2507: accuracy per benchmark against training step.}
\label{fig:bench-batch-math-512}
\end{figure}

\begin{figure}[h]
\centering
\includegraphics[width=\textwidth]{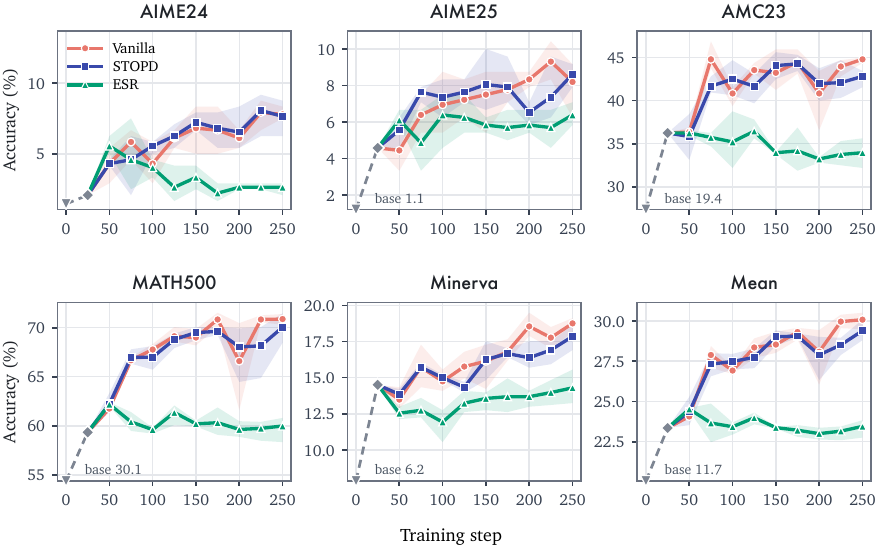}
\caption{ESR in mathematics, Qwen3-1.7B-Base $\leftarrow$ Qwen3-4B-Instruct-2507: accuracy per benchmark against training step.}
\label{fig:bench-esr}
\end{figure}





\end{document}